\documentclass{jair}

\usepackage{colortbl} 

\usepackage{amsfonts}
\usepackage{amsmath}
\usepackage[linesnumbered,ruled,vlined]{algorithm2e}
\usepackage[bibliography=common,appendix=append]{apxproof}
\usepackage{soul}
\usepackage{pgf,tikz}
\usetikzlibrary{positioning}
\usetikzlibrary{arrows}
\tikzset{block/.style={minimum size=0.8cm,outer sep=0pt,draw,rectangle,node distance=0pt}}

\usepackage{enumitem}
\usepackage{natbib}
\bibpunct{(}{)}{,}{a}{}{,}

\setcopyright{cc}
\copyrightyear{2026}
\acmYear{2026}
\acmDOI{10.1613/jair.1.19922}

\JAIRAE{Roberta Calegari}
\JAIRTrack{}
\acmVolume{86}
\acmArticle{41}
\acmMonth{7}
\acmYear{2026}

 \newtheoremrep{theorem}{Theorem}
 \newtheoremrep{lemma}[theorem]{Lemma}
 \newtheoremrep{proposition}[theorem]{Proposition}
 \newtheoremrep{example}[theorem]{Example}
 \newtheoremrep{definition}[theorem]{Definition}
 \newtheoremrep{corollary}[theorem]{Corollary}

 \newcommand{\ie}{{i.e.}}
 \newcommand{\fun}[1]{\ensuremath{\mathsf{#1}}}

 \newcommand{\pathAtom}{\ensuremath{p}}
 
 \newcommand{\poly}{\ensuremath{P}}
 \newcommand{\canonRenaming}[1]{\ensuremath{\varPhi_{#1}}}

\newcommand{\vocabulary}{\ensuremath{\mathcal{V}}}
\newcommand{\predicateSet}{\ensuremath{\mathcal{P}}}
\newcommand{\constantSet}{\ensuremath{\mathcal{C}}}
\newcommand{\atomSet}{\ensuremath{A}}
\newcommand{\termSet}{\ensuremath{T}}
 \newcommand{\erule}{\ensuremath{\rho}}
 \newcommand{\body}[1]{\ensuremath{\mathsf{body}(#1)}}
 \newcommand{\head}[1]{\ensuremath{\mathsf{head}(#1)}}
  \newcommand{\trigger}{\fun{t}}	

 \renewcommand{\path}{\ensuremath{p}}

 \newcommand{\ruleset}{\ensuremath{\mathcal{R}}}
 \newcommand{\instance}{\ensuremath{I}}
 \newcommand{\derivation}{\ensuremath{D}}
 \newcommand{\kb}{\ensuremath{\mathcal{K}}}
 \newcommand{\kblong}{\ensuremath{(\instance,\ruleset)}}
 \newcommand{\predicate}{\ensuremath{r}}
 \newcommand{\query}{\ensuremath{q}}
 \newcommand{\maxArity}{\ensuremath{w}}
 \newcommand{\twoWayBinaryPred}{\ensuremath{\mathcal{P}_2^\pm}}
 \newcommand{\interpretation}{\ensuremath{\mathcal{I}}}
 \newcommand{\vars}[1]{\fun{vars}{(#1)}}
 \renewcommand{\terms}[1]{\fun{terms}{(#1)}}
 \newcommand{\atoms}[1]{\fun{atoms}{(#1)}}
 \newcommand{\glb}[1]{\fun{glb}{(#1)}}
 \newcommand{\chasegraph}{\ensuremath{\mathcal{CG}}}
 \newcommand{\term}{\ensuremath{t}}
 \newcommand{\const}{a}
 \newcommand{\chasenull}{\ensuremath{\bot}}
 \newcommand{\cglabel}{\ensuremath{\ell}}
 \newcommand{\aux}{\fun{aux}}
 \newcommand{\type}[1]{\fun{type}{(#1)}}
 \newcommand{\chase}[2]{\fun{chase}(#1,#2)}
 \newcommand{\chasekb}[1]{\fun{chase}(#1)}
 \newcommand{\pathLabel}[1]{\lambda(#1)}

 \newcommand{\atom}{\alpha}
 \newcommand{\vect}[1]{\mathbf{#1}}
 \newcommand{\match}{\pi}
 \newcommand{\transition}[5]{\transitionAtomSet{#1}{#2}{#3}{\chase{#4}{#5}}}
 \newcommand{\transitionAtomSet}[4]{\ensuremath{\mathcal{T}_{#3,#4}(#1,#2)}}

 \newcommand{\proofScheme}{\ensuremath{\mathbb{P}}}
 \newcommand{\psforest}{\ensuremath{\mathcal{F}}}
 \newcommand{\pstransitions}{\ensuremath{\mathcal{T}}}
 \newcommand{\witness}{\ensuremath{\psi}}
 \newcommand{\expansion}{\ensuremath{\mathrm{expansion}}}
 \newcommand{\normVar}{\ensuremath{X}}

 \newcommand{\lang}{\ensuremath{\Lambda}}
 \newcommand{\automaton}{\ensuremath{\mathbb{A}}}
 \newcommand{\alphabet}{\Sigma}
 \newcommand{\letter}{a}
 \renewcommand{\state}{\ensuremath{s}}
 \newcommand{\emptyWord}{\varepsilon}
 \newcommand{\subLanguage}[3]{\ensuremath{\mathcal{L}_{#1}(#2,#3)}}
 \newcommand{\rpqAnswering}{\textsc{RPQ} Answering}

 \newcommand{\turingMachine}{\ensuremath{\mathcal{M}}}
 \newcommand{\tapeLetter}{\gamma}
 \newcommand{\transitionFunction}{\delta}
 \newcommand{\confOf}[1]{c(#1)}
 \newcommand{\inputOf}[1]{i(#1)}
 \newcommand{\outputOf}[1]{o(#1)}
 \newcommand{\pspace}{\textsc{PSpace}}
 \newcommand{\exptime}{\textsc{ExpTime}}
 \newcommand{\ptime}{\textsc{PTime}}
 
 \newcommand{\OMQAshort}{OMQA}
 \newcommand{\node}{\ensuremath \nu}
 \newcommand{\domel}{\ensuremath d}
 
 \newcommand{\ttype}{T}
  \newcommand{\attype}{\atom_{\ttype}}
 \newcommand{\word}{u} 
 \newcommand{\samepart}{\sim}

\begin{document}

\title[Answering Path Queries under Linear and Guarded Existential Rules]{Answering Path Queries under Linear and Guarded Existential Rules}


\author[Baget]{Jean-Fran\c cois Baget}
\orcid{0000-0001-7221-5770}
\email{baget@lirmm.fr}
\affiliation{
 \institution{LIRMM, Inria, University of Montpellier, CNRS}
 \city{Montpellier}
 \country{France}}
 
\author[Bienvenu]{Meghyn Bienvenu}
\orcid{0000-0001-6229-8103}
\email{meghyn.bienvenu@labri.fr}
\affiliation{%
  \institution{Univ. Bordeaux, CNRS, Bordeaux INP, LaBRI}
  \city{Talence}
  \country{France}}

\author[Mugnier]{Marie-Laure Mugnier}
\orcid{0000-0002-0574-3693}
\email{mugnier@lirmm.fr}
\affiliation{%
  \institution{LIRMM, Inria, University of Montpellier, CNRS}
  \city{Montpellier}
  \country{France}}

\author[Thomazo]{Michaël Thomazo}
\orcid{0000-0002-1437-6389}
\email{michael.thomazo@inria.fr}
\affiliation{%
  \institution{Inria, DIENS, ENS, PSL University, CNRS}
  \city{Paris}
    \country{France}}

\renewcommand{\shortauthors}{Baget, Bienvenu, Mugnier \& Thomazo}


  \begin{abstract}
Ontology-mediated query answering is concerned
with the problem of answering queries over knowledge
bases consisting of a database instance and
an ontology. While most work in the area focuses
on conjunctive queries (CQs), navigational
queries have gained increasing attention. In this
paper, we investigate the complexity of answering
two-way (conjunctive) regular path queries ((C)RPQs)
over knowledge bases whose ontology is given by
a set of guarded existential rules. {We first consider
the subclass of linear existential rules and show that
(C)RPQ answering is NL-complete in data complexity,
which matches the data complexity
of answering RPQs over plain graph databases (i.e.,\ without an ontology). 
In combined complexity, both 
tasks are \textsc{ExpTime}-complete in the general case, but RPQ and CRPQ
answering drop to \textsc{PTime}-complete and \textsc{PSpace}-complete respectively 
if there is a bound on predicate arity. }
For guarded rules, we provide a non-trivial reduction to the linear
case, which allows us to show that the complexity
of (C)RPQ answering is the same as for CQs,
namely 2\textsc{ExpTime}-complete in combined complexity {(\textsc{ExpTime}-complete in the bounded-arity case)}
and \textsc{PTime}-complete in data complexity. 
\end{abstract}

\maketitle

\section{Introduction}
{Over the past two decades, significant research efforts have been devoted to} \emph{ontology-mediated query answering} (\OMQAshort),  
in which
data is enriched with an ontology
that expresses domain knowledge, and queries are answered by taking both the data and the ontology into account (see e.g. \citep{DBLP:journals/jods/PoggiLCGLR08} for a seminal early work and \citep{DBLP:conf/ijcai/XiaoCKLPRZ18} for a short survey).  
Adding an ontological layer on top of data 
allows a user to formulate queries in a more familiar way, closer to their conceptualization of the application domain, thereby abstracting from the actual structure of databases, which may have been designed for independent purposes. On the other hand, it can also provide more complete answers to queries, as answers are based not only on facts explicitly stored in the data but also on those that are consequences of the ontology and the data. Finally, in the context of data integration, it provides a uniform mediating layer that facilitates the integration of heterogeneous data sources. {The \OMQAshort\ approach thus offers several important advantages, which are of relevance in diverse application areas (we refer readers to \citep{DBLP:journals/dint/XiaoDCC19} for an overview of \OMQAshort\ applications)}. However, {the need to take ontological information into account  when computing query answers can make the query answering task} significantly more complex, depending on the expressivity of the ontology {and query languages. }

Two families of ontology languages are
mainly considered in OMQA, namely description logics (DLs) and existential rules (see e.g.\ the survey chapters \citep{DBLP:conf/rweb/OrtizS12,DBLP:conf/rweb/KontchakovRZ13,DBLP:conf/rweb/BienvenuO15} on \OMQAshort~ with DLs and \citep{swim-09-cgl,DBLP:conf/rweb/MugnierT14} for existential rules). These ontology languages allow one to reason in open domains, which means that existing entities are not supposed to be restricted to those explicitly encoded in the data. %
Existential rules, also known as tuple generating dependencies  \citep{alice}, can be seen as an {extension of Datalog (equivalently, function-free Horn rules)} with existentially quantified variables in the rule heads. For instance, the following existential rule expresses that `every person has a parent who is a person': $\forall x.~human(x) \rightarrow \exists y.~hasParent(x,y) \land human(y)$. Applied to a fact like $human(a)$, this rule leads to infer the existence of an infinite number of entites. Existential rules generalize most DLs that are used in the context of data access, often referred to as Horn description logics. More generally, they overcome some well-known limitations of DLs by their unrestricted predicate arity and their ability to express non-tree-shaped relationships between atoms. This added expressivity makes ground fact entailment undecidable (e.g. \citep{beeri-vardi:81}), however many decidable classes of existential rules have been exhibited, which offer different tradeoffs between expressivity and reasoning complexity. We focus here on two central classes {\citep{pods-09-cgl,DBLP:journals/jair/CaliGK13}}: \emph{guarded} existential rules, in which the body of a rule has a guard, i.e., an atom that contains all the variables occurring in the rule body, and \emph{linear} existential rules, a subclass of guarded rules in which the body of a rule is made of a single atom, {as in the above example rule}.

{
As for the query language, most work on OMQA has focused on \emph{conjunctive queries} (CQs), a fundamental query class in relational databases and in the Semantic Web, where this class is also known as {basic graph pattern queries}. However, the increasing popularity of graph databases
and the rise of knowledge graphs
have shifted attention towards navigational queries \citep{DBLP:journals/csur/AnglesABHRV17,DBLP:conf/pods/LibkinMMPV25}, which are able to capture structural patterns of unbounded size.
}
The simplest such queries are \emph{regular path queries} (RPQs) \citep{DBLP:conf/sigmod/CruzMW87,DBLP:journals/siamcomp/MendelzonW95,DBLP:journals/jcss/CalvaneseGLV02}, which ask for paths of potentially arbitrary length {whose label conforms to a given regular language}, hence allowing for a controlled form of recursion over binary predicates.
Many extensions of RPQs have been investigated, including \emph{conjunctive regular path queries} (CRPQs) \citep{DBLP:conf/pods/FlorescuLS98,DBLP:conf/kr/CalvaneseGLV00,DBLP:conf/icdt/CucumidesRV23} which generalize both RPQs and CQs. 
In particular, the core queries of SPARQL~1.1 %
(the W3C standard\footnote{SPARQL 1.1 Query Language
W3C Recommendation:  \url{http://www.w3.org/TR/sparql11-query}} for querying RDF data) roughly correspond to CRPQs.
{The next example illustrates these different types of queries in the context of ontology-mediated query answering with existential rules.
 It will be used as a running example throughout the paper to provide a high-level overview of the key notions. Additional dedicated examples will illustrate the details and subtilities underlying the technical constructions. }

\medskip

\begin{example}[Running Example]
\label{ex-running-general}
Consider data describing relationships in a social network, which uses the 
predicates 
 \fun{follows}($x$,$y$), \fun{isFriendOf}($x$,$y$)
  and \fun{message}($m,x,y$), meaning that message $m$ was sent from $x$ to $y$. 
Asking for all pairs  $(x,y)$ such that $x$ and $y$ both follow some person can be expressed as a CQ:
	$$q_1(x,y) = \exists z. ~\fun{follows}(x,z) \land \fun{follows}(y,z)$$
It can also be expressed as an RPQ:
	$$q_2(x,y) = \fun{follows \cdot follows^{-}}(x,y)$$
where the operator $^-$ denotes the inverse relation. 
Asking for pairs $(x,y)$ such that $y$ is a direct or indirect follower of $x$ can be expressed by an RPQ but not by a CQ:
	$$q_3(x,y) = \fun{follows}^{-} \cdot ( \fun{follows}^{-})^{*}(x,y)$$
In turn, asking for pairs $(x,y)$ such that $x$ and $y$ follow each other can be expressed by a CQ but not by an RPQ:
	$$q_4(x,y) = \fun{follows}(x,y) \land  \fun{follows}(y,x)$$
Now, to retrieve  pairs $(x,y)$ such that $x$ follows $y$ and $y$ directly or indirectly follows $x$ requires the expressivity of a CRPQ, which generalizes $q_4$ by replacing a standard predicate with a path expression:
	$$q_5(x,y) = \fun{follows}(x,y) \land  \fun{follows \cdot follows}^{*}(y,x)$$
Asking whether user $\fun{Alice}$ received a message from someone she follows directly or indirectly can be expressed by a (Boolean) CRPQ:
	$$q_6() = \exists y,m.~\fun{follows \cdot follows}^{*}(Alice,y) \land message(m,y,Alice)$$
Now, let us add ontological knowledge stating that (1) the relation \fun{isFriendOf} is symmetric, (2) it is a specific case of the relation \fun{follows}, and (3) following someone implies sending her a message. This can be expressed using the following three linear rules: 

\begin{quote}
$(\erule_1)$\quad $\forall x,y.~\fun{isFriendOf}(x,y) \rightarrow  \fun{isFriendOf}(y,x)$
\\
$(\erule_2)$\quad $\forall x,y.~\fun{isFriendOf}(x,y) \rightarrow  \fun{follows}(x,y)$ 
\\
$(\erule_3)$\quad $\forall x,y.~\fun{follows}(x,y) \rightarrow  \exists m. ~\fun{message}(m,x,y)$ 
\end{quote}
Assume the data indicates that $\fun{Bob}$ follows $\fun{Alice}$ and they have a common friend $\fun{Carmen}$, which is expressed by the facts $\fun{follows(Bob,Alice)}, \fun{isFriendOf(Carmen,Alice)}$ and $\fun{isFriendOf(Carmen,Bob)}$. None of the preceding queries have an answer on this data alone, while all of them have answers when the rules are taken into account. 
 Indeed, the following facts are inferred from the data and the rules:
\begin{quote}
$ \fun{isFriendOf(Alice,Carmen)}, \fun{isFriendOf(Bob,Carmen)}$	\\
$\fun{follows(Carmen,Alice)}, \fun{follows(Carmen,Bob)}, \fun{follows(Alice,Carmen)}, \fun{follows(Bob,Carmen)}$ \\
$\fun{message(m_0,Bob,Alice)}, \fun{message(m_1,Carmen,Alice)}, \fun{message(m_2,Carmen,Bob)},$\\$\fun{message(m_3,Alice,Carmen)}, \fun{message(m_4,Bob,Carmen)}$
\end{quote}
For instance, the facts $\fun{follows(Alice,Carmen)}$
and $\fun{follows(Carmen,Bob)}$ yield the answer $\fun{(Bob,Alice)}$ to $q_3$; 
 from these two facts and $\fun{message(m_0,Bob,Alice)}$, we conclude that $q_6$ is answered positively.

As an example of a guarded rule, consider the following rule defining the predicate \fun{isPaired}:
\begin{quote}
$(\erule_4)$ \quad  $\forall x,y.~\fun{follows}(x,y) \land  \fun{follows}(y,x) \rightarrow  \fun{isPaired}(x,y)$  
\end{quote}

Finally, the rules below define the predicate \fun{extFollows}, which captures the regular expression $\fun{follows.follows}^{*}$: 
\begin{quote}
$\forall x,y.~\fun{follows}(x,y) \rightarrow  \fun{extFollows}(x,y)$\\
$\forall x,y,z.~\fun{follows}(x,y) \land  \fun{extFollows}(y,z) \rightarrow  \fun{extFollows}(x,z)$
\end{quote}
Using the predicate \fun{extFollows}, the above (C)RPQs can be reformulated as CQs. Note however that the last rule is not guarded, hence does not belong to the existential rule fragments studied in this paper. This shows that navigational features in queries can compensate for certain limitations of the selected ontological fragment. 
\end{example}
\color{black}

{As description logics are well suited to expressing ontological knowledge about graph-like data, such as ABoxes and RDF graphs, it is not surprising that navigational queries have long been investigated for DL ontologies. Such studies cover a variety of DLs,}  
ranging from highly expressive DLs of the $\mathcal{Z}$ family
\citep{DBLP:conf/aaai/CalvaneseEO07,DBLP:conf/ijcai/CalvaneseEO09,DBLP:journals/iandc/CalvaneseEO14}, to Horn DLs like Horn-$\mathcal{SROIQ}$ \citep{DBLP:conf/ijcai/OrtizRS11,DBLP:conf/kr/BienvenuCOS14} and
lightweight DLs of the DL-Lite and $\mathcal{EL}$ families \citep{DBLP:journals/jair/StefanoniMKR14,DBLP:conf/aaai/KostylevRV15,DBLP:journals/jair/BienvenuOS15}. While the main focus of the preceding work was on establishing the complexity of answering navigational queries,
 there has also been some very recent work exploring how to obtain practical algorithms \citep{DBLP:conf/dlog/DragovicO023,DBLP:conf/esws/LohnertAOO25}. 

{By allowing predicates of any arity, in addition to the expression of cyclic relations between entities, existential rules naturally apply not only to graph data but also to data with a more complex hypergraph structure, such as relational databases or graph data enriched with contextual information. 
This flexibility is particularly useful in applications involving heterogeneous data, since it makes it possible to handle each kind of data as it is, without breaking it down into binary facts.  This does not contradict the observation that binary relationships are still 
central in any context, whether they are found directly in certain data or built as part of the ontological modeling. Hence, navigational queries remain highly relevant even when the data and/or ontology includes higher-arity relations.}
However, while there is now an extensive literature on navigational queries in the presence of DL ontologies, 
very little is known about the complexity (or even decidability) of answering such queries under different kinds of existential rules. 

Our paper makes an important contribution towards clarifying the complexity landscape for navigational queries under existential rules by establishing the complexity of answering (C)RPQs in the presence of linear and guarded existential rulesets. To the best of our knowledge, these results (which were first reported in the conference papers \citep{rr-16-bt,DBLP:conf/ijcai/BagetBMT17}) constitute the first and the only complexity results targeting navigational queries with existential rules. 
However, two recent works have established the decidability of (i) RPQ answering for sticky existential rulesets \citep{DBLP:conf/kr/Ostropolski-Nalewaja24}, and (ii) answering an extension of CRPQs for finite clique-width rulesets \citep{DBLP:conf/icdt/0001LOR23} (in both cases without any upper complexity bounds), demonstrating the continued interest in the topic.

\paragraph{Contributions.} 
We obtain tight complexity results for both combined and data complexities over guarded rules and their linear subclass for both RPQ and CRPQ answering. With respect to combined complexity, we furthermore distinguish between bounded and unbounded predicate arity. These results are synthesized in Table~\ref{table-linear-complexity} for linear rules and in Table~\ref{table-guarded-complexity} for guarded rules, alongside existing results for CQs  (the grey background indicates new results). 
In the linear case, we see that CRPQ answering and RPQ answering have the same data and unbounded-arity combined complexities, and these complexities are higher than those of CQ answering. However, for bounded-arity combined complexity, RPQ answering is easier than CQ answering, while CRPQ is more difficult. In the guarded case, we show that RPQ answering and CRPQ answering both have the same complexity as CQ answering, regardless of the complexity measure. 
All these results are proven for existential rules with a single head atom; since arbitrary linear rules (respectively guarded rules) can be polynomially translated into atomic-head linear rules (respectively guarded rules), the results also hold for combined complexity with unbounded predicate arity as well as data complexity. Even though the translation does not preserve bounded predicate arity, we show that for arbitrary linear rules the same upper bounds apply (and we conjecture this is true also for guarded rules). 
It is also worth noting that, while our complexity results for (C)RPQs are in most cases higher than the analogous results for plain graph databases, the \textsc{NL} data complexity of (C)RPQ answering with linear rules coincides with the data complexity of (C)RPQs in the ontology-free setting. 

  \begin{table}
   \caption{Landscape of (C)(RP)Q Answering Complexity under Linear Rules}
  \centering
 \begin{tabular}{|c|c|c|c|}
  \hline
  &\cellcolor[gray]{.8} {\textbf{RPQ}}&\textbf{CQ}&\cellcolor[gray]{.8}   \textbf{CRPQ}\\
  \hline
  \textbf{Data} &  \cellcolor[gray]{.8} \textsc{NL}-c & \textsc{AC$_0$} &  \cellcolor[gray]{.8}   \textsc{NL}-c \\
  \hline
  \textbf{Combined, bounded arity} &  \cellcolor[gray]{.8}  \textsc{PTime}-c & \textsc{NP}-c & \cellcolor[gray]{.8}   \textsc{PSpace}-c\\
  \hline
  \textbf{Combined, unbounded arity} & \cellcolor[gray]{.8} \textsc{ExpTime}-c & \textsc{PSpace}-c  & \cellcolor[gray]{.8}    \textsc{ExpTime}-c \\
  \hline
 \end{tabular}
  \label{table-linear-complexity}
\end{table}

 \begin{table}
    \caption{Landscape of (C)(RP)Q Answering Complexity under Guarded Rules}
 \centering
 \scalebox{0.9}{
 \begin{tabular}{|c|c|c|c|}
  \hline
  &{ \cellcolor[gray]{.8}  \textbf{RPQ}}&\textbf{CQ}&\cellcolor[gray]{.8}  \textbf{CRPQ}\\
  \hline
  \textbf{Data} &  \cellcolor[gray]{.8}  \textsc{PTime}-c&\textsc{PTime}-c&  \cellcolor[gray]{.8}  \textsc{PTime}-c \\
  \hline
  \textbf{Combined, bounded arity} & \cellcolor[gray]{.8}   \textsc{ExpTime}-c  &  \textsc{ExpTime}-c & \cellcolor[gray]{.8}   \textsc{ExpTime}-c\\
  \hline
  \textbf{Combined, unbounded arity} & \cellcolor[gray]{.8}   \textsc{2ExpTime}-c & \textsc{2ExpTime}-c & \cellcolor[gray]{.8}  \textsc{2ExpTime}-c \\
  \hline
 \end{tabular}
}
    \label{table-guarded-complexity}
\end{table}

\paragraph{Paper organization.} After a preliminary section, we first investigate the complexity of query answering for RPQs and linear existential rules. We rely on a forward chaining scheme, known as the \emph{chase} \citep{maier-al:79,beeri-vardi:84} which, starting from a (finite) set of facts, iteratively applies rules until a fixpoint is reached, if any. To get complexity upper bounds, we provide an algorithm that exploits the structure of paths of terms in the chase to guess a path that complies with the language defined by the RPQ (Section \ref{sec-rpq-linear-upper}). We then prove that this algorithm is worst-case optimal: in the bounded arity case by using known lower bounds for DL-Lite, a language in which relevant assertions are specific linear rules \citep{DBLP:journals/jair/BienvenuOS15}, and in the general case by providing a novel reduction from an alternating \pspace\ Turing machine to RPQ answering under linear rules (Section \ref{sec-linear-lower}). We then consider CRPQ answering, still under linear rules (Section \ref{sec-crpq-linear-upper}). We devise an algorithm that exploits additional structural properties of the chase and uses the previous RPQ answering algorithm as an oracle, which allows us to upper-bound the problem complexity.  As the obtained upper bounds match lower bounds coming from previous results, this algorithm is worst-case optimal. Finally, to study the complexity of CRPQ answering under guarded rules, we provide a non-trivial reduction of the guarded case to the linear case 
(Section \ref{sec-crpq-guarded-upper}). This translation involves a double exponential blow-up of the set of rules (while the instance only grows exponentially in the predicate arity). However, a careful analysis of the algorithm provided for CRPQ answering under linear rules shows that it actually runs in 2\textsc{ExpTime} with respect to the input guarded knowledge base (and in \textsc{ExpTime} in the case of bounded-arity rules). {We end with a presentation of closely related work (Section \ref{sec-related-work}) and remaining questions (Section \ref{sec-conclusion}).}

 This article is an extended version of two previously published conference papers \citep{rr-16-bt,DBLP:conf/ijcai/BagetBMT17}. 
 {With respect to these conference papers, we provide a unified presentation, full proofs of the results, detailed examples and a review of related work with the latest results. }

\section{Preliminaries}
We consider logical vocabularies of the form $\vocabulary = (\predicateSet, \constantSet)$, where $\predicateSet$ is a finite set of predicates and $\constantSet$ is an infinite set of constants.
A (standard) 
\emph{atom} $\atom$ has the form $\predicate(\vect{t})$ where $\predicate$ is a predicate of arity $n$ and $\vect{t}$ is a tuple of terms (i.e., variables or constants) with $|\vect{t}|=n$.  
For $1 \leq i \leq |\vect t|$,  we denote by $\atom[i]$  the term at position $i$ in $\atom$.  We denote by $\terms{\atom}$ 
(resp. $\vars{\atom})$ the set of terms 
(resp. variables) in $\atom$ and extend the notations to a set of atoms.  A \emph{ground} atom contains only constants. 
A formula is \emph{atomic} if it is a single atom. It is \emph{closed} if it has no free variable. 
Given a (possibly infinite) set of atoms $\atomSet$ and a {(possibly infinite)} set of terms $\termSet$,   
 $\atomSet_{\mid \termSet}$ denotes the restriction of $\atomSet$  to atoms $\atom$ with $\terms{\atom} \subseteq \termSet$.

An \emph{interpretation} of a vocabulary $\vocabulary = (\predicateSet, \constantSet)$ is denoted by $\interpretation = (\Delta_\interpretation, .^\interpretation)$, where $\Delta_\interpretation$ is the possibly infinite (non-empty) domain of $\interpretation$ and $.^\interpretation$ is the interpretation function, such that $a^\interpretation \in \Delta$ for each $a \in \constantSet$ and $\predicate^\interpretation \subseteq \Delta^n$ for each predicate $\predicate \in \predicateSet$ with arity $n$. An interpretation is a \emph{model} of a closed formula $F$ (resp. a set of closed formulas $F$) if it makes $F$ (resp. every formula in $F$) true. For $F$ a conjunction of atoms, a \emph{match} of $F$ in $\interpretation$ is a mapping $\match$ from $\terms{F}$ to elements of $\Delta_\interpretation$ such that {\it (i)} $\match(\const) = \const^\interpretation$ for each constant $\const$; {\it (ii)} $\match(\vect t) \in r^\interpretation$ for each atom $r(\vect t)$ in $F$. Then, it holds that $\interpretation$ is a model of a closed formula $F$ iff there is a \emph{match} of $F$ in $\interpretation$. 
We consider classical logical entailment: given a set of closed formulas $F$ and a closed formula $f$,  $F \models f$ means that every model of $F$ is a model of $f$. 
In the following, we identify (the existential closure of) a conjunction of atoms with the set of these atoms.  Given two sets of atoms $\atomSet_1$ and $\atomSet_2$,  a \emph{homomorphism} from $\atomSet_2$ to $\atomSet_1$ is a substitution $\match$ of $\vars{\atomSet_2}$ by $\terms{\atomSet_1}$ such that $\match(\atomSet_2) \subseteq \atomSet_1$.  
An \emph{isomorphism} from $\atomSet_2$ to $\atomSet_1$ is a bijective substitution $b$ of $\vars{\atomSet_2}$ by $\vars{\atomSet_1}$ such that $b(\atomSet_2) = \atomSet_1$.

Given two existentially closed conjunctions of atoms $\atomSet_1$ and $\atomSet_2$, it holds that $\atomSet_1 \models \atomSet_2$ iff there is a homomorphism from (the set of atoms in) $\atomSet_2$ to (the set of atoms in) $\atomSet_1$.
Hence, $\atomSet_1$ and $\atomSet_2$ are logically equivalent if and only if they are homomorphically equivalent. It is worth noticing that equivalent sets of atoms are not necessarily isomorphic. 

\subsection{Existential Rule Knowledge Bases}
\label{sec-ER}

 An \emph{instance} is a finite set of ground atoms. 
 An \emph{extended instance} is a finite set of {arbitrary} atoms, logically translated into an existentially-closed conjunction of atoms. 
  An \emph{existential rule} $\erule$ (or simply \emph{rule}) is of  the form 
 $\forall\vect{x}\forall\vect{y} ~[~B(\vect{x},\vect{y}) \rightarrow \exists \vect{z} ~H(\vect{x},\vect{z})~]$,  
 where $B$ and $H$ are non-empty conjunctions of atoms on variables, respectively called the \emph{body} and  the \emph{head} of $\erule$, and $\vect{x}, \vect{y}$ and $\vect{z}$ are pairwise disjoint. We also denote the body and the head of a rule $\erule$ by $\body{\erule}$ and $\head{\erule}$, respectively. 
 We make the common assumption that rules do not contain constants, which simplifies technical tools. The variables of $\vect x$ (resp. $\vect z$)  are called \emph{frontier variables}  (resp. \emph{existential variables}).  For brevity, we denote by $B \rightarrow H$ a rule with body $B$ and head $H$ and in our examples universal quantifiers {are implicit}.

A \emph{knowledge base} (KB) is of the form $\kb = \kblong$, where $\instance$ is an instance and $\ruleset$ a set of existential rules.  In the following, we assume that distinct rules in $\ruleset$ have disjoint sets of variables, even if we reuse variables in examples for the sake of simplicity.
 
 A rule $\erule = B \rightarrow H$  is \emph{applicable} to a set of atoms $\atomSet$ if there is a homomorphism $\match$ from $B$ to $\atomSet$. The pair $(\erule,\match)$ is called a \emph{trigger} on $\atomSet$. The application of $\erule$ according to $\match$ (or: the application of the trigger $(\erule,\match)$) produces a set of atoms obtained from $\head{\erule}$ by replacing each frontier variable $x$ with $\match(x)$ and each existential variable with a fresh variable, usually called a \emph{null}. We denote by $\match^{\mathrm{safe}}$ this extension of $\match$ that ``safely'' renames existential variables, so that distinct applications of the same rule 
 produce disjoint sets of nulls. The resulting set of atoms  is $\atomSet\cup \match^{\mathrm{safe}}(H)$.
 An atom $\atom \in \match^{\mathrm{safe}}(H)$ is said to be \emph{(directly) generated} by the trigger $(\erule,\match)$ if  $\atom \not \in \atomSet$. 
 {Observe that an atom $\atom$ that is produced by a trigger  $(\erule,\match)$ without being ``generated'' by that trigger necessarily comes from an atom in $\head{\erule}$ that does not contain any existential variable. }

The fundamental tool for reasoning on existential rules is a forward chaining procedure known as the \emph{chase}. Briefly, the chase enriches a given instance by applying rules until a fixpoint is reached.
This process may be infinite, as for instance with $\ruleset =\{ h(x) \rightarrow \exists z~p(x,z) \land h(z) \}$ (``every human has a parent who is a human'') and $\instance = \{h(a)\}$.  
 Several variants of the chase are known (see e.g., \citep{DBLP:journals/fuin/GrahneO18}). We consider here the simplest variant, called the \emph{oblivious chase} \citep{DBLP:journals/jair/CaliGK13}. 

Formally, an $\ruleset$-\emph{derivation} from an (extended) instance $\instance$ is a possibly infinite sequence of (extended) instances 
and triggers $\derivation = \instance_0 (= \instance)~(\erule_1,\match_1)~\instance_1
\ldots ~(\erule_n,\match_n)~\instance_n, \dots $, where, for all $i \geq 1$,  $\instance_i$ results from the application of the trigger $(\erule_i,\match_i)$ on $\instance_{i-1}$, with $ \erule_i \in \ruleset$, 
and no trigger appears twice in $\derivation$. It is simply called a derivation when $\ruleset$ is clear from the context; {and it is written  $\instance_0 (= \instance)~\instance_1
\ldots ~\instance_n, \dots $ when the triggers are not needed. }
 The \emph{result} of $\derivation$ is the set of atoms obtained {along the sequence},  i.e., $\bigcup_{i \geq 0} \instance_i$, and we denote it by 
 $atoms(\derivation)$. 
The derivation $\derivation$ is \emph{fair} if, for any $\instance_i$ and trigger $(\erule,\match)$ on $\instance_i$, 
there is $\instance_j$ in $\derivation$ such that $\instance_j$ results from the application of $(\erule,\match)$. 
An (oblivious) \emph{chase sequence} of $\kb = \kblong$  is a fair $\ruleset$-derivation from $\instance$. 
 
It is known that oblivious chase sequences on a {given} KB are either all finite or all infinite, and that they result in isomorphic sets of atoms \citep{DBLP:journals/jair/CaliGK13}. Hence, given a KB $\kb = (\instance,\ruleset)$, we  denote by chase$(\kb)$, or chase$(\instance, \ruleset)$, the result of any oblivious chase sequence of $\kb$.  Furthermore, the name chase is classically used to denote both the forward chaining process and its result.

The (result of the) chase of $\kb$ can be seen as a logical interpretation, which is a model of $\kb$. This model $\interpretation_\kb = (\Delta_{\interpretation_\kb}, .^{\interpretation_\kb})$ has for domain $\Delta_{\interpretation_\kb} = \terms{\text{chase}(\kb)}$ and its interpretation function $.^{\interpretation_\kb}$ is defined by the atoms in chase$(\kb)$ (i.e., for each constant $c$, $c^{\interpretation_\kb} = c$ holds, and, for each predicate $p$, $p^{\interpretation_\kb}$ is the set of  tuples $(t_1,\ldots t_k)$ such that  $p(t_1,\ldots t_k) \in \text{chase}(\kb)$). Furthermore, $\interpretation_\kb$ has the fundamental property of being a \emph{universal} model of $\kb$,  i.e., it homomorphically maps to any other model of $\kb$  \citep{DBLP:journals/tcs/FaginKMP05,DBLP:conf/pods/DeutschNR08}.
 It follows that it can act as a representative of all models of $\kb$ to check entailment of conjunctive queries, and the more general conjunctive regular path queries, as explained next (see Section \ref{sec-queries}).

We will consider two specific classes of existential rules, namely linear and guarded.  
 A rule is \emph{linear} if its body is atomic. 
 A set of atoms $\atomSet$ is \emph{guarded} if it contains an atom $\atom$, called a guard, such that $\terms{\atom} = \terms{\atomSet}$. Since $\atomSet$ may contain several guards, we use the notation $(\atomSet,\atom)$ to specify that $\atom$ is the considered guard.	
 A rule is \emph{guarded} if its body is guarded. 
Note that a linear rule is trivially guarded. 

  \paragraph{Single-head translation}\label{singleheadtransformation} In the following, we will make the assumption that the head of an existential rule is atomic. It is well-known that any existential rule $\erule$ on a vocabulary $\vocabulary$ can be decomposed into a set of atomic-head rules by adding a fresh predicate $p_\erule$, whose arity is the number of variables in the head of $\erule$. E.g., the rule $\erule = r(x,y) \rightarrow \exists z ~q(x,z) \wedge q(y,z)$ can be decomposed into three atomic-head rules: $r(x,y) \rightarrow \exists z~p_\erule(x,y,z)$;  $p_\erule(x,y,z) \rightarrow q(x,z)$; $p_\erule(x,y,z) \rightarrow q(y,z)$. This polynomial translation preserves entailment of formulas on $\vocabulary$.  Moreover, the decomposition of a guarded (resp. linear) rule yields a set of guarded (resp. linear) rules. 
  Hence, the assumption that rules have an atomic head can be made without loss of generality regarding combined complexity with unbounded predicate arity, as well as data complexity. 

\subsection{Properties of Linear Rules}\label{sec-prelim-linear}

A well-known property of linear rules is that chasing independently each atom of an instance yields a result equivalent to the chase of 
this instance, i.e., $\cup_{\atom \in \instance} \chase{\{\atom\}}{\ruleset}$ and $\chase{\instance}{\ruleset}$  are homomorphically equivalent. 
 Indeed, the key property of linear rules is that  a pair $(\erule,\match)$ is a trigger on a set of atoms $\atomSet$ iff there is $\atom \in \atomSet$ such that $(\erule,\match)$ is a trigger on $\{\atom\}$. Hence, the sequence of triggers in a derivation from $\instance$ can be split into sequences of triggers that each define a derivation from an atom in $\instance$. It follows that $\chase{\instance}{\ruleset}$ is isomorphic to a subset of $\cup_{\atom \in \instance} \chase{\{\atom\}}{\ruleset}$. In the other direction, the sequence of triggers in a derivation from any atom $\atom \in \instance$ defines a derivation from $\instance$. It follows that 
$ \chase{\{\atom\}}{\ruleset}$ is isomorphic to a subset of $\chase{\instance}{\ruleset}$. However, while $\cup_{\atom \in \instance} \chase{\{\atom\}}{\ruleset}$ and $\chase{\instance}{\ruleset}$ are homomorphically equivalent, they are not necessarily isomorphic, as illustrated by the next example. 

 \begin{example} Let $\instance= \{p(a,b),q(a,b)\}$ and $\ruleset = \{\erule_1,\erule_2\}$ with $\erule_1 = p(x,y) \rightarrow q(x,y)$ and $\erule_2 = q(x,y) \rightarrow \exists z~r(y,z)$. Then $\chase{\instance}{\ruleset} = \{ p(a,b), q(a,b), r(b,z_0)\}$, where $z_0$ is a null, by the sequence of triggers $\trigger_1 = (\erule_1, \{ x \mapsto a, y \mapsto b\})$, $\trigger_2 =(\erule_2, \{ x \mapsto a, y \mapsto b\})$ (or the reverse sequence). Trigger  $\trigger_1$ does not generate any atom because it produces $q(a,b)$ and $q(a,b) \in \instance$. Let  $\atom_1 = p(a,b) \in \instance$: $\chase{\{\atom_1\}}{\ruleset} = \{p(a,b), q(a,b), r(b,z_0)\}$ by the same sequence of triggers. Now, $\trigger_1$  \emph{generates} the atom $q(a,b)$. Let $\atom_2 = q(a,b) \in \instance$: $\chase{\{\atom_2\}}{ \ruleset} = \{q(a,b), q(b,z_1)\}$ by $\trigger_2$.
 The union of the two chases  is:
 $$\{ p(a,b), q(a,b), r(b,z_0), r(b,z_1)\}$$
 which maps to $\chase{\instance}{\ruleset}$ by the homomorphism $\match = \{ z_1 \mapsto z_0 \}$, but is not isomorphic to it.

 \end{example}
 
 Let us say that an atom $\atom_j$ is \emph{generated} from an atom $\atom_i$ in $\chase{\instance}{\ruleset}$ if \emph{(i)} $\atom_j$ is directly generated by a trigger on $\atom_i$, or \emph{(ii)} $\atom_j$ is directly generated by a trigger on an atom itself generated from $\atom_i$. In the rest of this paper, we will specifically rely on the following two immediate properties. 
 First, atoms generated from distinct (ground) atoms $\atom$ and $\atom'$ in $\instance$ do not share any null (Proposition \ref{prop-distinct-nulls}). Second, the set of atoms $A_\atom$ generated from an atom $\atom \in \instance$ forms a subset of $\chase{\{\atom\}}{\ruleset}$, up to the choice of nulls (Proposition \ref{cor-atomic-chase}). 
We will in fact only use a corollary of Proposition \ref{cor-atomic-chase}: every finite set of atoms that are all generated from an atom $\atom \in \instance$ forms a subset of $\chase{\{\atom\}}{\ruleset}$. 

\begin{proposition} 
\label{prop-distinct-nulls}
Given any $\kb = \kblong$, let $\atom_i$ and $\atom_j$ be generated from distinct atoms in $\instance$ by $\chase{\instance}{\ruleset}$. Then, $\vars{\atom_i} \cap \vars{\atom_j} = \emptyset$. 
\end{proposition}

\begin{proof} 
If two (distinct) atoms $\atom_i$ and $\atom_j$ share a ({necessarily} fresh) variable, then they are generated from the same ground atom from $\instance$. Indeed, let $\alpha_k$ be the atom in which this variable is introduced. If $\atom_k = \atom_i$ or $\atom_k = \atom_j$, then one atom is generated from the other. Otherwise, $\atom_i$ and $\atom_j$ are both generated from $\atom_k$. In all the cases, $\atom_i$ and $\atom_j$ are generated from the same atom $\atom \in \instance$ from which $\atom_k$ is generated. 
\end{proof}

\begin{proposition} 
\label{cor-atomic-chase}
For any KB $\kb = \kblong$ with $\ruleset$ a linear ruleset, atom $\atom \in \instance$ and set $\atomSet_{\atom}$ of all the atoms generated from $\atom$ in  $\chase{\instance}{\ruleset}$, there is an isomorphism from $\{\atom\} \cup \atomSet_{\atom}$ to a subset of $\chase{\{\atom\}}{\ruleset}$. 
\end{proposition}

\begin{proof} Let $D_{\atom} = \instance_0 (= \atom)~(\erule_1,\match_1)~\instance_1
\ldots ~(\erule_n,\match_n)~\instance_n, \dots $ be the (possibly infinite) derivation resulting in $\{\atom\} \cup \atomSet_{\atom}$, which is extracted from the derivation associated with $\chase{\instance}{\ruleset}$. 
We inductively build a sequence of substitutions $\phi_0, \ldots, \phi_n, \ldots$, such that for all $i \geq 0$,   $\phi_i$ is an isomorphism from $\instance_i$ to a subset of $\chase{\instance}{\ruleset}$.
For $i=0$, $\phi_i$ is the empty substitution (i.e., $\phi_0(\atom) = \atom$).
 For $i > 0$, let $\atom_i = \instance_i \setminus \instance_{i-1}$, and let $t_{i}=(\erule_i,\match_i)$ 
be the trigger that produces $\atom_{i}$. Let $\atom_{j_i}$ be the atom on which $t_{i}$ is applied. If $\atom_{i}$ does not contain a new null,  $\phi_i = \phi_{i-1}$. 
Otherwise, let $\atom' = \phi_{i-1}(\atom_{j_i})$:
the pair $t' = (\erule_i, \phi_{i-1} \circ \match_i)$
is a trigger on $\atom'$ and it necessarily \emph{generates} an atom since a new null is created. Then, $\phi_i$ is obtained by extending $\phi_{i-1}$ in the obvious way to bijectively map the nulls introduced by $t_{i}$ to those introduced by $t'$. By construction,  
$\phi_{i}$ is an isomorphism from $\instance_i$ to a subset of $\chase{\{\atom\}}{\ruleset}$. 
Furthermore, $\cup_{ i=0}^{\infty} \phi_i$ is an injective substitution of $\vars{\{\atom\} \cup \atomSet_{\atom}}$ by $\vars{\chase{\{\atom\}}{\ruleset}}$, i.e., an isomorphism from $\{\atom\} \cup \atomSet_{\atom}$ to a subset of $\chase{\{\atom\}}{\ruleset}$. 
\end{proof}

\subsection{Queries}
\label{sec-queries}

 We now define the classes of queries we consider: conjunctive queries, regular path queries and their common generalisation, namely conjunctive regular path queries. In short, a conjunctive query is an existentially quantified conjunction of standard atoms; a regular path query is a single ``path atom'', i.e., an atom on a binary predicate corresponding to a regular language; and a conjunctive regular path query is an existentially quantified conjunction of standard and path atoms.

A regular language can be represented by a regular expression or by a non-deterministic finite automaton (NFA).  Let $\alphabet$ be a finite set of symbols. A regular expression $\mathcal{E}$ over $\alphabet$ is defined by the grammar:
 $\mathcal{E} \rightarrow \emptyWord \mid \letter \mid \mathcal{E}\cdot\mathcal{E} \mid \mathcal{E} + \mathcal{E} \mid \mathcal{E}^*$,
 where $\letter \in \alphabet$ and $\emptyWord$ denotes the empty word. 
 An NFA over $\alphabet$ is a tuple $\automaton = (S,\alphabet,\delta,s_0,F)$, where $S$ is a finite set of states,  $\delta \subseteq S \times \alphabet \times S$ is the transition relation, $s_0 \in S$ is the initial state and $F \subseteq S$ is the set of final states. For convenience, we sometimes use the notation $s' \in \delta(s,a)$ to denote that $(s,a,s') \in \delta$.
{Since NFAs are exponentially more succinct than regular expressions \citep{DBLP:journals/jcss/EhrenfeuchtZ76}, the complexity proofs should use the NFA representation for upper bounds and the regular expression representation for lower bounds, so that the results hold regardless of the representation. We indeed use the NFA representation for upper bounds. For lower bounds, we mostly rely on previous results, which indeed use the regular expression representation; for  our only proof of a lower bound (Section \ref{sec-linear-lower}), the query is of the form $p^*$, hence the choice of a representation does not matter. }

We denote by $\mathcal L(\mathcal{E})$  the language defined by the regular expression $\mathcal{E}$. Simarly, we denote by $\mathcal L(\automaton)$ the language recognized by the automaton $\automaton$. Furthermore, we denote by  $\subLanguage{\automaton}{s_i}{s_{j}}$ the set of words that 
take $\automaton$ to the state $s_j$ when they are read from the state $s_i$. When $s_i = s_0$ and $s_j \in F$, $\subLanguage{\automaton}{s_i}{s_{j}} \subseteq \mathcal L(\automaton)$. 

 We denote by $\predicateSet_2$ the subset of binary predicates in the considered vocabulary $\vocabulary = (\predicateSet, \constantSet)$. Given a binary predicate $r$, we also consider its inverse predicate $r^-$ (i.e., for all terms $\term_1$ and $\term_2$, $r(\term_1,\term_2)$ holds iff $r^-(\term_2,\term_1)$ holds). We set  $\predicateSet_2^\pm = \predicateSet_2 \cup \{r^- \mid r \in \predicateSet_2\}$.  
 A \emph{path predicate} $\lang$ is a binary predicate given by an NFA or a regular expression defining a regular language $L(\lang)$ over $\predicateSet_2^\pm$. A \emph{path atom} is an atom built on a path predicate, i.e., it  takes the form $\lang(\term_1,\term_2)$, where $\lang$ is a path predicate and $\term_1, \term_2$ are terms. Note that standard binary predicates are a special case of path predicates. We assume without loss of generality that automata associated with distinct path predicates have disjoint sets of states. 

 A \emph{conjunctive (two-way) regular path query} (CRPQ)\footnote{As we only consider the two-way variant, we will use the abbreviation (C)RPQ instead of the more usual (C)2RPQ. } has the form $q(\mathbf{x}) = \exists \mathbf{y} B$, where $\mathbf{x}$ and $\mathbf{y}$ are disjoint tuples of variables, and $B$ is a conjunction of standard and path atoms with $\terms{B} = \mathbf{x} \cup \mathbf{y}$.  The free variables of  $q$, i.e., the $\mathbf{x}$ variables, are called \emph{answer variables}. A CRPQ is \emph{Boolean} if it is a closed formula, i.e., $\mathbf{x} = \emptyset$. 
 A \emph{conjunctive query} (CQ) is a CRPQ $q(\mathbf{x}) = \exists \mathbf{y} B$ where $B$ contains only standard atoms. 
 A \emph{regular path query} (RPQ) is a CRPQ of the form $q(x_1,x_2) = \lang(x_1,x_2)$, where $\lang(x_1,x_2)$ is a path atom.  Hence, an RPQ contains exactly two variables, which are answer variables.

 Given an interpretation $\interpretation = (\Delta, .^{\interpretation})$, we extend the set of interpreted symbols as follows. First, for every $\predicate \in \predicateSet_2$, we set $({\predicate^-})^\interpretation = \{ (e_2,e_1) | (e_1,e_2) \in r^\interpretation \}$.  
Second, we call \emph{path} (from $e_0$ to $e_n$)  in $\interpretation$ a (finite) sequence $e_0 r_1 e_1 \ldots r_n e_n$, with $n \geq 0$ such that $e_0 \in \Delta_\interpretation$ and for every $1 \leq i \leq n$, $e_i \in \Delta_\interpretation$, $r_i \in  \predicateSet_2^\pm$ and $(e_{i-1}, e_i) \in r_i^\interpretation$. The \emph{label} $\lambda(p)$ of a path $p = e_0 r_1 e_1 \ldots r_n e_n$ is the word $r_1 \ldots r_n$. If $n = 0$,  $\lambda(p)$ is the empty word  $\emptyWord$. 
Then, for any path predicate $\lang$, we set:
\begin{center}
$\lang^\interpretation = \{(e_0,e_n) |$ there is a path $p$ from $e_0$ to $e_n$  in $\interpretation$ such that $\lambda(p) \in  \mathcal L(\lang)\}$
\end{center}

The notion of model of a Boolean CRPQ is the classical logical notion, up to the above extensions of the set of interpreted symbols. The notion of match in an interpretation is extended in the same way, i.e., a \emph{match} of a CRPQ $q$ in an interpretation $\interpretation$ is a mapping $\match$ from $\terms{q}$ to elements of $\Delta_\interpretation$ such that {\it (i)} $\match(\const) = \const^\interpretation$ for each constant $\const$; {\it (ii)} $\match(\vect t) \in r^\interpretation$ for each (standard or path) atom $r(\vect t)$ in $q$. It follows that an interpretation $\interpretation$ is a model of a Boolean CRPQ $q$ iff there is a \emph{match} of $q$ in $\interpretation$. 

Given a CRPQ $q(\mathbf{x}) = \exists \mathbf{y} B$ with $\mathbf{x} = (x_1 \ldots x_k)$, a tuple of constants $\textbf{a} =(a _1 \ldots a_k)$ is an \emph{answer} to $q$ in $\interpretation$ if there is a match $\match$ of $q$ in $\interpretation$ such that $\match(x_i) = a_i^{\interpretation}$ for $1 \leq i \leq k$. Equivalently, $\interpretation$ is a model of the Boolean CRPQ $q()$ obtained from $q(\mathbf{x})$ by substituting each $x_i$ with $a_i$. 
 
We can now define the notion of a certain answer to a CRPQ on a KB. Given a CRPQ $q(\mathbf{x}) = \exists \mathbf{y} B$ and a KB $\kb = (\instance,\ruleset)$, a tuple of constants $\vect{a} = s(\vect x)$, where $s$ denotes a substitution, is a \emph{certain answer} to a CRPQ $q(\mathbf{x})$ over $\kb$ if  $\instance \cup \ruleset \models q(\vect a)$, where $q(\vect a) = \exists \vect y .s(B)$. In other words, $\vect{a}$ is an answer to $q$ in each model of $\kb$. 
Concerning the specific case of $q$ being a Boolean CRPQ, we have that the empty tuple $()$ is a certain answer to $q$ on  $\kb$ if $\instance \cup \ruleset \models q$. 

Although certain answers are defined with respect to all models of the KB, it is actually sufficient to consider answers in the unique model defined by the chase of the KB. This immediately follows from two properties: 
 \emph{(i)} the chase of a KB $\kb$ is a universal model of $\kb$, 
and \emph{(ii)} CRPQs are closed under homomorphism, which means that for any Boolean CRPQ $q$ and model $\interpretation$ of $q$, any interpretation $\interpretation'$ to which $\interpretation$ homomorphically maps is also a model of $q$. 

\begin{proposition}
For any KB $\kb$ and CRPQ $q(\mathbf{x})$, a tuple of constants $(\vect a)$ is a certain answer to $q$ over $\kb$ if and only if there is a match $\match$ of $q$ in chase($\kb$) such that $\match(\mathbf{x}) =  (\vect a)$. 
\end{proposition}

In the following, we will identify the chase with the interpretation naturally associated with it. In particular, a \emph{path of terms} in the chase is the natural translation of a path in the associated interpretation. We denote a path of terms by $\path = \term_0 r_1 \term_1 \cdots r_n \term_n$, where the $\term_i$ are terms, and the $r_i$ elements of  $\predicateSet_2^\pm$, and simply by $\term_0 \cdots \term_n$ when only the extremities are needed.

Finally, we point out that, whenever there is a match of a CRPQ in the possibly infinite chase, this match can be found in a finite portion of the chase. 

\begin{proposition}
For any KB $\kb=\kblong$ and CRPQ $q$, if there a match $\match$ of $\query$ in $\chasekb{\kb}$, then for any chase sequence $\derivation = \instance_0 (= \instance) ~\instance_1
\ldots ~\instance_n, \dots $ of  $\kb$, there is $\instance_i$ in $\derivation$ such that $\match$ is a match of $q$ in $\instance_i$. 
\end{proposition}

\begin{proof}(Sketch) Let $\match$ be a match of $\query$ in chase($\kb$). Let $\derivation = \instance_0 (= \instance) ~\instance_1
\ldots ~\instance_n, \dots $  be a chase sequence of $\kb$. Given any standard atom $\atom$ in $\query$, let rank($\atom$) be the smallest $i$ such that $\match(\atom) \in \instance_i$. Given any path atom $\atom = \lang(t,t')$ in $\query$, let rank($\atom$) be the smallest $i$ such that $\instance_i$ contains all the (standard) atoms of a path $\term_0 (= \match(t)) r_1 \term_1 \cdots r_n \term_n (=\match(t'))$.  Let $j$ be the maximal value of rank($\atom$) among all standard and path atoms $\atom$ in $\query$. Then, $\match$ is a match of $\query$ in $\instance_j$. 
\end{proof}

Hence, given a Boolean CRPQ $\query$ and a KB $\kb = \kblong$, it holds that $\kb \models \query$ if and only if there a finite $\ruleset$-derivation from $\instance$  that results in a set of atoms $\instance_n$ such that $\instance_n \models \query$.

\subsection{Query Answering Problem}\label{subsection:QA}

Let $\mathbf{C} \in \{\text{CQ, RPQ, CRPQ}\}$ denote a class of queries. The  \emph{$\mathbf{C}$ answering problem} asks, given a KB $\kb= \kblong$, a query $\query$ 
from the class $\mathbf{C}$ and a tuple of constants $\vect a$, whether $\vect a$ is an answer to $\query$ over $\kb$. Since ground atom entailment from an existential rule KB is undecidable in general 
(see \citep{beeri-vardi:81,chandra-lewis-makowsky:81} for the first undecidability results about equivalent problems on tuple-generating dependencies), the $\mathbf{C}$ answering problem is undecidable for any query class $\mathbf{C}$. 

We will study the complexity of (C)RPQ answering in the case of linear and guarded rules, according to two complexity measures: \emph{combined} complexity, where $\kb$ and $\query$ are both part of the problem input, and \emph{data} complexity, where $\query$ and $\ruleset$ are fixed, and only $\instance$  is part of the input. This second complexity measure is specially relevant when it can be assumed that the sizes of the query and the set of rules are small compared to the size of the instance. With respect to combined complexity, we will in turn distinguish between two cases, depending on whether the predicate arity is \emph{bounded} or \emph{unbounded}.

We will refer to standard complexity classes: \textsc{NL} (problems solvable in {non-deterministic} logarithmic space), {also called \textsc{NLogSpace}}, \textsc{PTime} (problems solvable in deterministic polynomial time), \textsc{NP}
(problems solvable in non-deterministic polynomial time), \textsc{PSpace} (problems solvable in polynomial space), \textsc{ExpTime} (problems solvable in deterministic exponential time)
and 2\textsc{ExpTime} (problems solvable in deterministic doubly exponential time). In Table \ref{table-linear-complexity}, we furthermore mention  \textsc{AC$_0$}, a  class from circuit complexity, which is a strict subclass of \textsc{NL}. 
{To sum up, these classes are ordered as follows: 
$ \textsc{AC$_0$} \subset  \textsc{NL}  \subseteq \textsc{PTime} \subseteq  \textsc{NP}  \subseteq \textsc{PSpace}  \subseteq \textsc{ExpTime} \subseteq 2\textsc{ExpTime}$.
}

\medskip
The complexities of CQ answering under linear or guarded rules are recalled in Tables \ref{table-linear-complexity} and \ref{table-guarded-complexity}. 
It is well known that CQ evaluation (or CQ answering with $\ruleset = \emptyset$) is in AC$_0$ for data complexity and NP-complete for combined complexity, regardless of the assumption on predicate arity. 
Under linear rules, CQ answering remains in AC$_0$ for data complexity \citep{pods-09-cgl} and NP-complete for combined complexity with bounded arity (follows from \citep{jcss-84-jk,DBLP:conf/kr/GottlobS12})\footnote
{\citep{jcss-84-jk} proves that inclusion dependencies (a subclass of linear rules) of bounded predicate arity enjoy the Polynomial Witness Property and \citep{DBLP:conf/kr/GottlobS12} points out that the proof extends to linear rules. This implies the NP membership. NP-hardness follows from the complexity of CQ evaluation.}  
 but it becomes \textsc{PSpace}-complete for combined complexity with unbounded arity \citep{swim-09-cgl}. Under guarded rules, CQ answering is  \textsc{PTime}-complete for data complexity \citep{pods-09-cgl}, \textsc{ExpTime}-complete for combined complexity with bounded arity \citep{DBLP:journals/jair/CaliGK13} and 2\textsc{ExpTime}-complete for combined complexity with unbounded arity \citep{DBLP:journals/jair/CaliGK13}. 

{For the sake of comparison, let us further recall the complexity\footnote{The reason that we do not provide references for the complexity results for graph databases is that they are considered folklore in the area (and explicitly noted as such in the survey chapter \citep{DBLP:conf/rweb/Figueira21}).  } of answering (C)RPQs over graph databases, in the absence of any ontology. 
RPQ evaluation is known to be \textsc{NL}-complete in both data and combined complexity. For CRPQs, the data complexity remains \textsc{NL}-complete, but the combined complexity rises to \textsc{NP}-complete (matching that of CQ evaluation in databases). } 

\medskip
To prove the results about CRPQ answering, it will be convenient to consider the following problem, called \emph{CRPQ Entailment}, which asks, given a KB $\kb= \kblong$ and a Boolean CRPQ $\query$, if $\instance \cup \ruleset \models \query$. Both problems are linearly reducible one to the other, since, given a Boolean CRPQ $\query$, it holds that $\instance \cup \ruleset \models \query$ if and only if $()$ is a certain answer to $\query$, and, in turn, a tuple of constants $(\vect a)$ is a certain answer to a CRPQ $\query(\vect x)$  if and only if $\instance \cup \ruleset \models \query_{(\vect x \mapsto \vect a)}$, where  $\query_{(\vect x \mapsto \vect a)}$ is the Boolean CRPQ obtained from $\query(\vect x)$ by substituting each $x_i$ in $\vect x$ with $a_i$ in $\vect a$.

\newpage

\section{RPQ Answering under Linear Rules: Upper Bound}
\label{sec-rpq-linear-upper}

We consider the problem of computing the certain answers to a regular path query in the presence of a set of linear rules and present an algorithm that is inspired by a related algorithm for DL-Lite ontologies \citep{DBLP:journals/jair/BienvenuOS15} (see Section \ref{sec-related-work} for discussion). 
The algorithm takes as input an RPQ whose regular language is given by an NFA and works roughly as follows: 
\begin{itemize}
\item a path in the chase is guessed step by step, keeping in memory only the current constant of the instance and current state of the automaton;
\item when a path passes through nulls in the chase, these terms are not guessed, instead the state of the automaton when the path returns to constants of the instance is guessed.
\end{itemize}
The first item corresponds to a standard algorithm for RPQ answering over bare data. The real difficulty lies in the second item, which requires a characterization of when it is possible to reach an automaton state through a path involving null elements.

\medskip

The following proposition makes a first step towards such a characterization by observing that every `anonymous' path in the chase can be localized within the chase of a single ground atom. It exploits the fact that for linear rules, each atom can be `chased' independently (see Section \ref{sec-prelim-linear}). 

\begin{proposition}
\label{ref-prop-atom-loop}
 Let $\ruleset$ be a set of linear rules and $\instance$ be an instance. Then for every path $p = \domel_0 r_1 \domel_1 \cdots r_n \domel_n$ in $\chase{\instance}{\ruleset}$ such that $\domel_0, \domel_n  \in \terms{\instance}$ and $\domel_i \not \in \terms{\instance}$ for $0 <i  <n$, there exists an atom $\atom \in \instance$ with $\domel_0,\domel_n \in \terms{\atom}$ such that there is a path $p'$ in $\chase{\{\atom\}}{\ruleset}$ from $\domel_0$ to $\domel_n$ with 
$\lambda(p')=\lambda(p)$.
\end{proposition}
\begin{proof}
Let $A_p$ be the set of atoms in $p$. Any consecutive atoms in $A_p$, i.e., $r_i(\domel_{i-1}, \domel_{i})$ and $r_{i+1}(\domel_{i}, \domel_{i+1})$, $0< i < n$, share a null, hence, by Proposition \ref{prop-distinct-nulls}, there is a unique atom $\atom \in \instance$ from which all the atoms in $A_p$ are generated in $\chase{\instance}{\ruleset}$. 
From Proposition \ref{cor-atomic-chase}, $A_p$ is isomorphic to a subset of $\chase{\{\atom\}}{\ruleset}$, which yields the path $p'$. 
\end{proof}

We can therefore concentrate on identifying the paths that occur within the chase of a single atom. In order to abstract from the particular terms occurring in an atom, we introduce the following notion of type: 

\begin{definition}[Type of an atom]
\label{def-type}
A \emph{type} is a pair $\ttype=(\predicate,\samepart_{\ttype})$ where $\predicate$ is a predicate of arity $k$ and $\samepart_{\ttype}$ is an equivalence relation on $\{1,\ldots,k\}$. The type of an atom $\atom$, denoted by $\type{\atom}$, is the pair $T=(\predicate,\samepart_{\ttype})$ where $\predicate$ is the predicate of $\atom$ and $i \samepart_{\ttype} j$ %
iff the $i^\mathrm{th}$ and the $j^\mathrm{th}$ arguments of $\atom$ are equal. 
\end{definition}

For instance, the atoms $\predicate(x,y)$, $\predicate(y,a)$ and $\predicate(a,b)$, where $a$ and $b$ are constants, have the same type 
 $\ttype = (\predicate, \{\{1\},\{2\}\})$, while the atom
 $\predicate(x,x)$ has type $(\predicate, \{\{1,2\}\})$; here, $\samepart_{\ttype}$ is given by its set of equivalence classes. 
 
When atoms $\atom_1$ and $\atom_2$ have the same type, there exists a bijective mapping $\theta_{12}$ from $\terms{\atom_1}$ to $\terms{\atom_2}$ such that $\atom_2 = \theta_{12}(\atom_1)$. We call this mapping the \emph{natural mapping} from $\atom_1$ to $\atom_2$ and observe that $\theta_{21} = \theta_{12}^{-1}$. With this notion in hand, we can formalize the idea that atoms with the same type behave the same regarding rule applications.  

\begin{proposition} \label{prop-natural-mapping}
  Let $\atom_1$ and $\atom_2$ be two atoms of the same type, and let  $\theta_{12}$ be the natural mapping from $\atom_1$ to $\atom_2$. If $\erule \in \ruleset$ is applicable to $\atom_1$ by $\match$ creating $\atom'_1$, then $\erule$ is applicable to $\atom_2$ by $\theta_{12}\circ\match$ creating $\atom'_2$, where $\atom_1'$ and $\atom_2'$ have the same type, and the natural mapping $\theta'_{12}$ from $\atom'_1$ to $\atom'_2$ coincides with $\theta_{12}$ on the terms on which they are both defined. 
 \end{proposition}
 
A corollary of this proposition is that $\chase{\{\atom_1\}}{\ruleset}$ and $\chase{\{\atom_2\}}{\ruleset}$ will contain the same kinds of paths, as formalized next. 

\begin{corollary}\label{cor-types-same-paths}
 Let $\atom$ and $\atom'$ be atoms with $\type{\atom}=\type{\atom'}$. Then 
there is a path in $\chase{\{\atom\}}{\ruleset}$ from term $\atom[i]$ to term $\atom[j]$ with label $l$
 if and only if there is a path in $\chase{\{\atom'\}}{\ruleset}$ from $\atom'[i]$ to $\atom'[j]$ with label $l$. 
\end{corollary}

\begin{proof} From Proposition \ref{prop-natural-mapping}, there is a bijective mapping $\theta$ from  the terms in $\chase{\{\atom\}}{\ruleset}$ to the terms in $\chase{\{\atom'\}}{\ruleset}$, such that $\theta(\chase{\{\atom\}}{\ruleset}) = \chase{\{\atom'\}}{\ruleset}$. Hence, every path $p$ in $\chase{\{\atom\}}{\ruleset}$ from term $\atom[i]$ to term $\atom[j]$ is mapped by $\theta$
to a path $p'$ in $\chase{\{\atom'\}}{\ruleset}$ from $\atom'[i]$ to $\atom'[j]$ with the same label as $p$. 
\end{proof}

The following proposition shows that it is possible to take paths in the chase of a single atom and reproduce them in a larger chase that contains an atom of the same type. 

\begin{proposition}
\label{proposition-linear-chase}
 Let $\atom$, $\atom'$, $\beta$ be atoms such that $\type{\atom}=\type{\atom'}$ and $\atom \in \chase{\{\beta\}}{\ruleset}$. 
 If there is a path in $\chase{\{\atom'\}}{\ruleset}$ from term $\atom'[i]$ to term $\atom'[j]$ with label $l$,
 then there is a path in $\chase{\{\beta\}}{\ruleset}$ from $\atom[i]$ to $\atom[j]$ with label $l$.
\end{proposition}

\begin{proof} Let $p'$ be a path in $\chase{\{\atom'\}}{\ruleset}$ from term $\atom'[i]$ to term $\atom'[j]$ with label $l$. Since $\type{\atom}=\type{\atom'}$, by Corollary \ref{cor-types-same-paths}, there is a path $p$ in $\chase{\{\atom\}}{\ruleset}$ from term $\atom[i]$ to term $\atom[j]$ with label $l$. We then show that $\chase{\{\atom\}}{\ruleset}$ is isomorphic to a subset of $\chase{\{\beta\}}{\ruleset}$, from which we conclude that $\chase{\{\beta\}}{\ruleset}$ contains a path $p''$ from $\atom[i]$ to $\atom[j]$ with label $l$. 
We prove  that there is an isomorphism from the result of any $\ruleset$-derivation from $\atom$ to a subset of $\chase{\{\beta\}}{\ruleset}$, by induction on the length of this derivation. 
At rank 0, the property is true since $ \atom \in \chase{\{\beta\}}{\ruleset}$. Assume it is true until rank $n$. 
Consider a derivation from $\atom$ of length $n+1$: let $A_n$  be the set of atoms obtained after applying the $n^\mathrm{th}$ trigger; let $(\erule, \pi)$ be the $n^\mathrm{th}+1$ trigger, $a = \pi(\body{\erule})$ and $a'$ be the resulting atom. Let $\psi_n$ be the isomorphism from $A_n$ to a subset of $\chase{\{\beta\}}{\ruleset}$. Since $a$ and $\psi_n(a)$ have the same type, $(\erule, \psi_n \circ \pi)$ is a trigger on $\psi_n(a)$ and occurs in the derivation that builds $\chase{\{\beta\}}{\ruleset}$, creating atom $a''$ of the same type as $a'$ (as in Proposition \ref{prop-natural-mapping}). We build the mapping $\psi_{n+1}$ from $A_n \cup \{a'\}$ to $\chase{\{\beta\}}{\ruleset}$ by extending $\psi_{n}$ according to the natural mapping from $a'$ to $a''$.
Since $\psi_{n}$ is an isomorphism from $A_n$ to a subset of $\chase{\{\beta\}}{\ruleset}$, 
$\psi_{n+1}$ is an isomorphism from $A_n \cup \{a'\}$ to a subset of $\chase{\{\beta\}}{\ruleset}$. 
\end{proof}

It now remains to determine which paths are available starting from an atom of a given type.
More specifically, we want to know, 
given a type $T$, positions $j,j'$ in any atom $\atom$ of type $T$, automaton $\automaton$, and states $s_i,s_i'$, whether $\chase{\atom}{\ruleset}$ contains a path $p$ from $\atom[j]$ to $\atom[j']$ whose label $\lambda(p)$ takes $\automaton$ from state $s_i$ to  state $s'_i$. 
To this end, we formalize the notion of a type admitting a loop:

\begin{definition}
We say that a \emph{type $\ttype$ admits an $(s_i,j,s'_i,j')$-loop (w.r.t.\ a linear ruleset $\ruleset$ and NFA $\automaton$)} if there is a path $p$ in $\chase{\{\attype\}}{\ruleset}$ from $\attype[j]$ to $\attype[j']$ such that $\pathLabel{p} \in \subLanguage{\automaton}{s_i}{s_{i'}}$, where $\attype$ is any atom of type $\ttype$. 
\end{definition}
\noindent Any choice of the atom $\attype$ will give the same result, due to Corollary~\ref{cor-types-same-paths}. 
{Before detailing the algorithms, we provide below a high-level overview of how they work using the running example.}

\medskip

\begin{example}[Running example---continued] Let  $\ruleset$ be the following set of linear rules:
\begin{quote}
${(\erule_1)} \quad \fun{isFriendOf}(x,y) \rightarrow  \fun{isFriendOf}(y,x) $
\\
$(\erule_2) \quad \fun{isFriendOf}(x,y) \rightarrow  \fun{follows}(x,y)$
\\
$(\erule_3) \quad \fun{follows}(x,y) \rightarrow  \exists m ~\fun{message}(m,x,y)$
\\
$(\erule_5) \quad \fun{message}(m,x,y) \rightarrow  \fun{sends}(x,m)$ \\
$(\erule_6) \quad\fun{message}(m,x,y) \rightarrow  \fun{receives}(y,m)$ 
\end{quote}
The first three rules were already introduced in Example \ref{ex-running-general} and the two additional rules allow us to reformulate the CRPQ $q_6$  in the form of an RPQ. Recall that $q_6$ asks whether user $\fun{Alice}$ (abbreviated as $\fun{A}$ in the following) received a message from someone she follows directly or indirectly. Instead of the CRPQ $q_6$, we can consider the following RPQ $q'_6$ and check whether $(\fun{A},\fun{A})$ is a certain answer:
$$q'_6(x,y) = 
\fun{follows\cdot follows}^{*} \cdot \fun{sends} \cdot \fun{receives}^-(x,y) $$
We assign to  $q'_6$ the NFA representation $\automaton$ with states $\{s_0, s_1, s_2, s_f\}$ (with $s_0$ initial and $s_f$ final) and four transitions:  $(s_0, \fun{follows}, s_1)$, $(s_1, \fun{follows}, s_1)$, $(s_1, \fun{sends}, s_2)$ and $(s_2, \fun{receives}^-, s_f)$. 

Let us consider again the instance $\instance = \{\fun{follows(B,A)}, \fun{isFriendOf(C,A)}, \fun{isFriendOf(C,B)}\}$. 
Query $q'_6$ can be matched to the following path of terms in $\chase{\instance}{\ruleset}$:
$$(\fun{A}~\fun{follows}~\fun{C}~\fun{follows}~\fun{B}~\fun{sends}~\fun{m_0}~\fun{receives}^-\fun{A})$$
 with $\fun{m_0}$ the null introduced when applying the rule $(\erule_3)$ to $\fun{follows(B,A)}$. 
 The facts $\fun{follows(A,C)}$ and $\fun{follows(C,B)}$, witnessing the subpath $(\fun{A}~\fun{follows}~\fun{C}~\fun{follows}~\fun{B})$,  
 are derived from the facts $\fun{isFriendOf(C,A)}$ and $\fun{isFriendOf(C,B)}$, respectively, using rules $\erule_1$ and $\erule_2$. As these are ground facts, the corresponding subpath $(A~\fun{follows}~C~\fun{follows}~B)$ will be entirely guessed by the algorithm. 
This subpath takes $\automaton$ from state $s_0$ to $s_1$. 
By contrast, the subpath $(B~\fun{sends}~m_0~\fun{receives}^-A)$, whose witnessing atoms contain the null $m_0$, will not be guessed. Instead, we will rely on Proposition \ref{ref-prop-atom-loop} and find a ground fact in $\instance$, here $\atom=\fun{follows}(B,A)$, such that $(B~\fun{sends}~m_0~\fun{receives}^-A)$ occurs in 
 $\chase{\atom}{\ruleset}$. More precisely: the  type $(\fun{follows}, \{\{1\},\{2\}\})$, which is the type of $\atom$, admits a loop from 
 the term in position 1 to the term in position 2 that takes $\automaton$ from state $s_1$ to $s_f$ (in short: an $(s_1,1,s_f,2)$-loop). 

We will use Algorithm 1, detailed later, to compute the set of loops admitted by 
 each of the types associated with $\ruleset$. %
 In particular, we have here that \emph{(i)} type $(\fun{sends}, \{\{1\},\{2\}\})$ admits an $(s_1,1,s_2,2)$-loop and \emph{(ii)} type $(\fun{receives}^-, \{\{1\},\{2\}\})$ admits an $(s_2,2,s_f,1)$-loop, which is directly obtained from the transitions in $\automaton$. From  \emph{(i)} and rule $(\erule_5)$ considered ``from head to body'', we obtain that type $(\fun{message}, \{\{1\},\{2\},\{3\}\})$ admits an $(s_1,2,s_2,1)$-loop. Similarly, from \emph{(ii)} and rule $(\erule_6)$, we obtain that type $(\fun{message}, \{\{1\},\{2\},\{3\}\})$ admits an $(s_2,1,s_f,3)$-loop. Concatenating the two loops admitted by type $(\fun{message}, \{\{1\},\{2\},\{3\}\})$, we add an $(s_1,2,s_f,3)$-loop to that type, which, using rule $(\erule_3)$, yields the $(s_1,1,s_f,2)$-loop for the type $(\fun{follows}, \{\{1\},\{2\}\})$ of $\atom=\fun{follows}(B,A)$. 

Algorithm 2, which is the global algorithm, takes as input $\instance$, $\ruleset$, $\automaton$, as well as a candidate answer (the pair $(\fun{A},\fun{A})$ in our example). 
 It first calls Algorithm 1 to build the table of loops associated with types, then tries to verify the existence of a witnessing path in 
$\chase{\instance}{\ruleset}$, while being guided by the transitions in $\automaton$ and loops assigned to types. Here, starting from state $s_0$ and constant $\fun{A}$, it will guess the atoms $\fun{follows(A,C)}$ and $\fun{follows(C,B)}$ using transitions in $\automaton$ and check that these atoms are indeed entailed by $(\instance, \ruleset)$. This corresponds to the subpath $(A~\fun{follows}~C~\fun{follows}~B)$, which takes $\automaton$ from state $s_0$ to $s_1$. Then, the algorithm guesses constant $A$ and state $s_f$ together with the $(s_1,1,s_f,2)$-loop for type $(\fun{follows}, \{\{1\},\{2\}\})$, and checks that $\instance$ contains the atom $\fun{follows}(B,A)$. 
  \end{example}
  \color{black}

We shall now explain how to compute the loops admitted by a given type. 
Algorithm \ref{algo-loop-table-creation} builds a 
table \texttt{Loop} whose cells are indexed by tuples $(s_i,j,s'_i,j')$, with the cell $(s_i,j,s'_i,j')$ storing the types $\ttype$ that admit an $(s_i,j,s'_i,j')$-loop. We use the notation $\texttt{Loop}(s_i,j,s'_i,j')$ to denote the set of types stored in the cell $(s_i,j,s'_i,j')$. 
We assume all the cells are empty when they are created.
Line \ref{line-trivial} initializes the table by stating that one can go from a term in a position to the same term in another position 
without reading any word (and thus not moving in the automaton). 
Lines \ref{line-binary-direct} and \ref{line-binary-reverse} 
correspond to traversing a single 
binary atom, 
reading its label either as an $r$ or an $r^-$, in the case where both terms 
 are distinct, i.e., $\ttype = (\predicate, \{\{1\},\{2\}\})$. 
 Lines \ref{line-begin-init-binary-merged} to \ref{line-end-init-binary-merged} are similar but handle the case where both arguments are equal, i.e., $\ttype = (\predicate, \{\{1, 2\}\})$.
Finally, Lines \ref{line-transitivity} and \ref{line-propagation} serve to combine these basic paths.
On the one hand, the table is saturated  through  
path concatenation: if one can go from a position $j_1$ to a position $j_2$ by a path taking the automaton from state $s_1$ to state $s_2$, and from $j_2$ to a position $j_3$ by a path taking the automaton from $s_2$ to state $s_3$, then one can go from $j_1$ to $j_3$ by a path taking the automaton from $s_1$ to $s_3$. On the other hand, rules are taken into account: for any rule $\alpha \rightarrow \beta$, if $\alpha$ and $\beta$ both contain variables $x$ and $y$ (with possibly $x=y$) and one can go from the position of $x$ in $\beta$ to the position of $y$ in $\beta$ taking the automaton from state $s_1$ to state $s_2$ then the same move is possible from the position of $x$ in $\alpha$ to the position of $y$ in $\alpha$. Indeed, in the chase of any atom $\atom_1$ of type $\type{\alpha}$, there is an atom $\beta_1$ of type $\type{\beta}$, 
which has the same terms as $\atom_1$ in the positions of $x$ and $y$.  
Next Example \ref{ex-algo-1-steps} further illustrates  the different steps of  Algorithm 1.  

\newpage

\begin{algorithm}[t]
\DontPrintSemicolon
\KwData{A set of linear rules $\ruleset$ and NFA $\automaton= (S,\alphabet,\delta,s_0,F)$} 
\KwResult{A table whose cells are indexed by tuples 
$(s,j,s',j') \in S \times [1,\maxArity] \times S \times [1,\maxArity] $,
 with $\maxArity$ the maximal predicate arity in $\ruleset$,
and each cell contains a set of types on predicates from $\ruleset$}

\tcc{Initialization step}
\For{type $\ttype$ of predicate of arity $k$}{
	\For{pair $(j,j') \in \{1,\ldots,k\}^2$ with $j\samepart_{\ttype}j'$} {
			\For{$s \in S$}{
		$\texttt{Loop}(s,j,s,j') \leftarrow \texttt{Loop}(s,j,s,j') \cup \{T\}$;\label{line-trivial}
			}
		}
	}

\For{type $\ttype$ based on $r(x,y)$}{
	\For{$(s_1,r,s_2) \in \delta$}{
		$\texttt{Loop}(s_1,1,s_2,2) \leftarrow \texttt{Loop}(s_1,1,s_2,2) \cup \{\ttype\}$;\label{line-binary-direct}
	}
	\For{$(s_1,r^-,s_2) \in \delta$}{
		$\texttt{Loop}(s_1,2,s_2,1) \leftarrow \texttt{Loop}(s_1,2,s_2,1) \cup \{\ttype\}$;\label{line-binary-reverse}
	}
}
\For{type $\ttype$ based on $r(x,x)$}{
	\For{$(s_1,r,s_2) \in \delta$ or $(s_1,r^-,s_2) \in \delta$}{
		$\texttt{Loop}(s_1,1,s_2,1) \leftarrow \texttt{Loop}(s_1,1,s_2,1) \cup \{\ttype\}$;\\
		\label{line-begin-init-binary-merged}
		$\texttt{Loop}(s_1,1,s_2,2) \leftarrow \texttt{Loop}(s_1,1,s_2,2) \cup \{\ttype\}$;\\
		$\texttt{Loop}(s_1,2,s_2,1) \leftarrow \texttt{Loop}(s_1,2,s_2,1) \cup \{\ttype\}$;\\
		$\texttt{Loop}(s_1,2,s_2,2) \leftarrow \texttt{Loop}(s_1,2,s_2,2) \cup \{\ttype\}$;\\
		\label{line-end-init-binary-merged}
	}
}

\tcc{Saturation step}
\While{something added}{\label{line-begin-while}
	\For{$\ttype \in \mathtt{Loop}(s_1,j_1,s_2,j_2) \cap \mathtt{Loop}(s_2,j_2,s_3,j_3)$ 
	}{\label{line-begin-transitive}
				$\mathtt{Loop}(s_1,j_1,s_3,j_3) \leftarrow \mathtt{Loop}(s_1,j_1,s_3,j_3) \cup \{\ttype\}$;\label{line-transitivity}\\		
	}\label{line-end-transitive}
	\For{$\alpha \rightarrow \beta \in \ruleset$}{\label{line-begin-propagation}
		\uIf{$\alpha[i_\alpha]=\beta[i_\beta]$, $\alpha[j_\alpha]=\beta[j_\beta]$, and $\type{\beta} \in \mathtt{Loop}(s_1,i_\beta,s_2,j_\beta)$}{	
			$\texttt{Loop}(s_1,i_\alpha,s_2,j_\alpha) \leftarrow \texttt{Loop}(s_1,i_\alpha,s_2,j_\alpha) \cup \{\type{\alpha}\}$;\label{line-propagation}
		}
	}\label{line-end-propagation}
}\label{line-end-while}
\caption{Creating the \texttt{Loop} table}
\label{algo-loop-table-creation}
\end{algorithm}

\begin{example} 
 \label{ex-algo-1-steps}  
 Let the ruleset $\mathcal{R}$ consist of the following three linear rules:
$$p(x,x,z) \rightarrow \exists w \, q(x,z,w) \qquad q(x,z,w) \rightarrow u(w,x) \qquad q(x,z,w) \rightarrow r(w,z)$$
and consider the types associated with the atoms appearing in these rules: 
$$T_p=(p,\{\{1,2\}, \{3\}\}) \quad T_q = (q, \{\{1\}, \{2\}, \{3\}\} \quad T_u = (u, \{\{1\}, \{2\}\}) \quad T_r = (r, \{\{1\}, \{2\}\})$$
For our query, we take the RPQ $r^- \!\!\cdot u^*(t_1,t_2)$ where $r,u \in \predicateSet_2$. To apply Algorithm \ref{algo-loop-table-creation}, we consider the corresponding NFA representation $\automaton$ with states $\{s_0, s_f\}$ (with $s_0$ initial and $s_f$ final) and two transitions:  $(s_0, r^-, s_f)$
and $(s_f, u, s_f)$. We indicate which parts of the algorithm allow us to add type $T_p$ to $\texttt{Loop}(s_0,3,s_f,1)$:
\begin{itemize}
\item Type $T_r$ is added to $\texttt{Loop}(s_0,2,s_f,1)$ in Line \ref{line-binary-reverse} due to transition $(s_0, r^-, s_f)$. 
\item Type $T_u$ is added to $\texttt{Loop}(s_f,1,s_f,2)$ in Line \ref{line-binary-direct} due to transition $(s_f, u, s_f)$. 
\item Type $T_q$ is added to $\texttt{Loop}(s_0,2,s_f,3)$ in Line  \ref{line-propagation} due to rule $q(x,z,w) \rightarrow r(w,z)$,  $T_r \in \texttt{Loop}(s_0,2,s_f,1)$, and the 2nd and 3rd arguments of $q(x,z,w)$ coinciding with  the 2nd and 1st arguments of $r(w,z)$. 
\item Type $T_q$ is added to $\texttt{Loop}(s_f,3,s_f,1)$ in Line \ref{line-propagation} due to rule $q(x,z,w) \rightarrow u(w,x)$, 
$T_u \in \texttt{Loop}(s_f,1,s_f,2)$, and the 3rd and 1st arguments of $q(x,z,w)$ coinciding with the 1st and 2nd arguments of $u(w,x)$. 
\item Type $T_q$ is added to $\texttt{Loop}(s_0,2,s_f,1)$ in Line \ref{line-end-transitive} because it is in both $\texttt{Loop}(s_0,2,s_f,3)$ and $\texttt{Loop}(s_f,3,s_f,1)$. 
\item Type $T_p$ is added to $\texttt{Loop}(s_0,3,s_f,1)$ in Line  \ref{line-propagation} due to rule $p(x,x,z) \rightarrow \exists w \, q(x,z,w)$, 
$T_q \in \texttt{Loop}(s_0,2,s_f,1)$, and the 3rd and 1st arguments of $p(x,x,z)$ coinciding with the 2nd and 1st arguments of $q(x,z,w)$. Note that since the 1st and 2nd arguments of $p(x,x,z)$ are the same, the algorithm will also add $T_p$ to $\texttt{Loop}(s_0,3,s_f,2)$. 
\end{itemize}
The inclusion of type $T_p$ in $\texttt{Loop}(s_0,3,s_f,1)$ allows one to infer for example that 
$\{p(a,a,c)\}, \mathcal{R} \models r^- \!\!\cdot u^*(c,a)$. 
\end{example}

The following propositions establish the correctness of Algorithm  \ref{algo-loop-table-creation} and provide upper bounds on its running time. 
Note that the algorithm is independent from the data.

\begin{proposition}
Let $\ruleset$ be a set of linear rules, $\automaton$ be an NFA, %
$\ttype$ be a type  whose predicate occurs in $\ruleset$, 
and $\mathtt{Loop}$ be the table constructed by Algorithm \ref{algo-loop-table-creation} on input ($\ruleset, \automaton$). Then $\ttype \in \mathtt{Loop}(s,i,s',j)$ iff $\ttype$ admits an $(s,i,s',j)$-loop. 
\end{proposition}

\begin{proof}
Throughout the proof, we assume that for each type $\ttype$, we have selected some atom $\attype$ of type $\ttype$. \smallskip

\noindent ($\Rightarrow$) We prove, by induction on the order of addition of types that whenever a type $\ttype$ is added to a cell in $\texttt{Loop}(s,i,s',j)$, then $\ttype$ admits an $(s,i,s',j)$-loop. 
First, suppose that $\ttype$ is added to $\texttt{Loop}(s,j,s,j')$ at Line \ref{line-trivial}. Then the trivial length-0 path $\attype[j]$ (=$\attype[j']$) in $\chase{\{\attype\}}{\ruleset}$  witnesses that $\ttype$ admits an $(s,j,s,j')$-loop, since it is labeled by the empty word, which does not change the state of the automaton. Next, consider the case in which $\ttype$ is added to $\texttt{Loop}(s_1,1,s_2,2)$ at Line \ref{line-binary-direct}. 
Then $\attype$ takes the form $r(t_1,t_2)$, so $\chase{\{\attype\}}{\ruleset}$ contains an $r$-labeled path from 
$\attype[1]$ to $\attype[2]$, which belongs to $\subLanguage{\automaton}{s_1}{s_{2}}$ since $s_2 \in \delta(s_1,r)$; hence, $\ttype$ admits an $(s_1,1,s_2,2)$-loop. 
 The reasoning is similar for types added at Line \ref{line-binary-reverse} and Lines \ref{line-begin-init-binary-merged} to \ref{line-end-init-binary-merged}.  
 
 If $\ttype$ is added to $\texttt{Loop}(s_1,j_1,s_3,j_3)$ at Line \ref{line-transitivity}, then there must exist a state $s_2$ for which $\ttype$ has already been added to $\texttt{Loop}(s_1,j_1,s_2,j_2)$ and $\texttt{Loop}(s_2,j_2,s_3,j_3)$. By the induction assumption, 
 there is a word $\word_1$ (resp.\ $\word_2$) in $\subLanguage{\automaton}{s_1}{s_2}$ (resp.\ $\subLanguage{\automaton}{s_2}{s_3}$)  that labels a path in $\chase{\{\attype\}}{\ruleset}$ 
 from $\attype[j_1]$ (resp.\ $\attype[j_2]$) to $\attype[j_2]$ (resp.\ $\attype[j_3]$). By concatenating these paths, we obtain %
  a path from $\attype[j_1]$ to $\attype[j_3]$ whose label $\word_1\word_2$ belongs to $\subLanguage{\automaton}{s_1}{s_3}$, which proves that $\ttype$ admits an $(s_1,j_1,s_3,j_3)$-loop. 
  
Finally, assume that $\type{\atom}$ is added to $\texttt{Loop}(s_1,i_{\alpha},s_2,j_{\alpha})$ at Line \ref{line-propagation} during the examination of rule $\atom \rightarrow \beta \in \ruleset$. Then it must be the case that 
$\alpha[i_\alpha]=\beta[i_\beta]$, $\alpha[j_\alpha]=\beta[j_\beta]$, and $\type{\beta} \in \texttt{Loop}(s_1,i_\beta,s_2,j_\beta)$. By the induction assumption, $\type{\beta}$ admits an $(s_1,i_\beta,s_2,j_\beta)$-loop, 
which implies that there is a path $p$ in $\chase{\beta}{\ruleset}$ from $\beta[i_{\beta}]$ to $\beta[j_{\beta}]$
such that $\lambda(p) \in \subLanguage{\automaton}{s_1}{s_2}$. Since $\beta \in \chase{\{\atom\}}{\ruleset}$ we can apply Proposition \ref{proposition-linear-chase} to infer that there is a path $p'$ in $\chase{\{\atom\}}{\ruleset}$ from $\beta[i_{\beta}]$ to $\beta[j_{\beta}]$
such that $\lambda(p') \in \subLanguage{\automaton}{s_1}{s_2}$.
As $\alpha[i_\alpha]=\beta[i_\beta]$, $\alpha[j_\alpha]=\beta[j_\beta]$, it follows that $p'$ is also a path from 
$\alpha[i_{\alpha}]$ to $\alpha[j_{\alpha}]$, which witnesses that $\type{\alpha}$ admits an $(s_1,i_{\alpha},s_2,j_{\alpha})$-loop. \smallskip

\noindent($\Leftarrow$) We suppose that $\ttype$ admits an $(s,i,s',j)$-loop and our aim is to show that $\ttype$ is added to 
$\texttt{Loop}(s,i,s',j)$. We proceed by induction on the length of the shortest path 
in $\chase{\{\attype\}}{\ruleset}$ that witnesses that $\ttype$ admits an $(s,i,s',j)$-loop. \smallskip

\noindent \textbf{Base case, path of length 0:} In this case, there is length-0 path $p$ in $\chase{\{\attype\}}{\ruleset}$ from $\attype[i]$ to $\attype[j]$ that takes $\automaton$ from state $s$ to state $s'$. We thus have $\lambda(p)=\emptyWord$, which implies that $s=s'$ and $i \samepart_{\ttype}j $. Hence $\ttype$ will be added to 
$\texttt{Loop}(s,i,s',j)$ at Line \ref{line-trivial}.\\
\textbf{Base case, path of length 1:} The fact that $\ttype$ admits an $(s,i,s',j)$-loop is witnessed by a length-one path. Let
$\attype[i] q \attype[j]$ be the witnessing path, where $q \in \predicateSet_2^\pm$ and $q \in \subLanguage{\automaton}{s}{s'}$ (i.e., $s' \in \delta(s,q)$). There are three possibilities:
\begin{itemize}
\item[(a)] if $q = r \in \predicateSet_2$ and $\attype[i]\neq\attype[j]$,
then  $\beta = r(\attype[i],\attype[j]) \in \chase{\{\attype\}}{\ruleset}$ and $\type{\beta}$ is added to $\texttt{Loop}(s,1,s',2)$ in Line 7;
\item[(b)] if $q=r^-$ for $r \in \predicateSet_2$ and $\attype[i]\neq\attype[j]$, then  $\beta = r(\attype[j],\attype[i]) \in \chase{\{\attype\}}{\ruleset}$, so $\type{\beta}$ is added to $\texttt{Loop}(s,2,s',1)$ in Line 9;
\item[(c)] if  $\attype[i]=\attype[j]$, then the atom
$$\beta = r(\attype[i],\attype[i]) = r(\attype[j],\attype[j]) = r(\attype[i],\attype[j]) = r(\attype[j],\attype[i])$$ belongs to  $\chase{\{\attype\}}{\ruleset}$, so
$\type{\beta}$ will be added to $\texttt{Loop}(s,1,s',1)$, $\texttt{Loop}(s,1,s',2)$, $\texttt{Loop}(s,2,s',1)$ and $\texttt{Loop}(s,2,s',2)$ in Lines 12-15. 
\end{itemize}
In all three cases, the atom $\beta$  belongs to $\chase{\{\attype\}}{\ruleset}$, which means that there exists a finite sequence of atoms $\attype = \atom_0,\ldots,\atom_m = \beta$ such that $\atom_{i+1}$ is generated by an application of $\erule_i \in \ruleset$ mapping \body{\erule_i} to $\atom_i$. 
We proceed by induction on the length $m$ of the sequence. 

First suppose that $m=0$, which means that $\beta=\attype$. In case (a), we have $T = \type{\beta}$, which means that $i=1$, $j=2$, and $T$ is added to $\texttt{Loop}(s,i,s',j)$ as desired. If we are in case (b), then we again have $T = \type{\beta}$ and $T$ is added to $\texttt{Loop}(s,2,s',1)$. Again in case (c), $\type{\beta} = T$ and $T$ is added to $\texttt{Loop}(s,1,s',1)$, $\texttt{Loop}(s,1,s',2)$, $\texttt{Loop}(s,2,s',1)$, and $\texttt{Loop}(s,2,s',2)$.

Next suppose that $m>1$ and that the property holds for all sequences of length at most $m-1$. As $\attype[i]$ and $\attype[j]$ appear in $\beta$, it holds that there exist $i'$ and $j'$ such that $\atom_1[i'] = \attype[i]$ and $\atom_1[j'] = \attype[j]$. As $\beta \in \chase{\{\atom_1\}}{\ruleset}$, $\type{\atom_1}$ admits an $(s,i',s',j')$-loop, and there exists a derivation of length $m-1$ that generates $\beta$ from $\atom_1$.  By our induction assumption for sequences of length at most $m-1$, it holds that $\type{\atom_1}$ has been added to $\texttt{Loop}(s,i',s',j')$. As $\atom_1$ is generated from $\attype$ by the application of a rule, Line~\ref{line-propagation} adds $T$ to $\texttt{Loop}(s,i,s',j)$.

	\textbf{Induction step:} Let us assume that the result holds for any path of length up to $n-1, n\geq 2$, and consider the path $p = \term_0q_1\term_1\ldots q_n\term_n$. First consider the case in which $\term_k$ is contained in $\atom_T$ for some $1 \leq k < n$, and let $l$ be a position of $\term_k$ in $\atom_T$. There exists a path from $\term_0$ to $\term_k$ of length strictly smaller than $n$, and similarly from $\term_k$ to $\term_n$. By the induction assumption, $T=\type{\atom_T}$ is in both $\texttt{Loop}(s,i,s'',l)$ and $\texttt{Loop}(s'',l,s',j)$ for some state $s''$.
An application of Line \ref{line-transitivity} yields $T=\type{\atom_T} \in \texttt{Loop}(s,i,s',j)$. Next suppose there is no $\term_k$ ($1 \leq k < n$) that occurs in $\atom_T$, which in particular means that $t_k \not \in \{t_0,t_n\}$ for  $1 \leq k < n$. 
Consider a finite derivation that generates all of the binary atoms from the path $p$. Let $t_k$ be such that no other $t_i$ ($1 \leq i <n$) is created earlier in the derivation, and let $\beta$ be the atom in which $t_k$ is created. We prove that the path $p$ is contained in $\chase{\{\beta\}}{\ruleset}$. 
Indeed, the atoms $q_{k-1}(t_{k-1}, t_k)$ and $q_{k}(t_{k}, t_{k+1})$ must occur in $\chase{\{\beta\}}{\ruleset}$ since $\beta$ created $t_k$. By induction, we obtain  that the same holds for all other atoms in $p$; this moreover implies that $\term_0$ (resp. $\term_n$) must occur in $\beta$, let us say at position $i'$ (resp. $j'$). 
By the induction hypothesis, $\type{\beta}$ belongs to $\texttt{Loop}(s,i',s'',k')$ and to $\texttt{Loop}(s'',k',s',j')$ for some state $s''$, where $k'$ is the position of $t_k$ in $\beta$. Hence, by Line \ref{line-transitivity}, $\type{\beta}$ is in the cell $\texttt{Loop}(s,i',s',j')$. Following the same reasoning as in the base case for paths of length $1$, by (repeated) application of Line~\ref{line-propagation}, $\type{\atom_T}$ is in the cell $\texttt{Loop}(s,i,s',j)$, which concludes the proof.
\end{proof}

\begin{proposition}\label{prop-complexity-algo1}
Algorithm \ref{algo-loop-table-creation} runs in exponential time, and in polynomial time if the predicate arity is bounded. 
\end{proposition}

\begin{proof}
There are polynomially many cells in the table, each of which can contain at most all types. The number $n_t$ of distinct types is single exponential in the maximum predicate arity (hence polynomial for bounded-arity predicates). 
The first for loop runs in $\mathcal{O}(n_t)$, the next two run in polynomial time, and the while loop is performed at most $n_t$ times. 
\end{proof}

There is a subtlety however that demands our attention, namely, the case of linear rulesets obtained via the single-head translation (see page~\pageref{singleheadtransformation}). 
Indeed, the preceding result establishes polynomial time complexity for bounded-arity linear rules \emph{with atomic heads}. 
However,
if the original ruleset contains rules with multiple atoms in the head, then we must first apply the single-head translation,
introducing new predicates which may have higher arities. This issue can be addressed by suitably modifying Algorithm \ref{algo-loop-table-creation} 
to take into account such new predicates. Specifically, for each such fresh predicate $p_\erule$ (introduced for rule $\rho$), which occurs in the head of new rule $\body{\erule} \rightarrow \exists y_1 \dots \exists y_\ell \, p_\erule(x_1, \dots, x_k,y_1, \dots, y_\ell)$, 
the modified algorithm only considers types for $p_\erule$ that contain $\{\{k+1\}, \ldots \{k+\ell\} \}$ (i.e.\ the positions that store existential variables are never merged with other positions). 
Note that for the original predicates, the modified algorithm considers all possible types (just like in Algorithm \ref{algo-loop-table-creation}).
The following result shows that the modified algorithm satisfies the required properties. 

\begin{proposition} \label{loop-bounded-arity}The modified Algorithm \ref{algo-loop-table-creation} runs in polynomial time for rulesets obtained by applying the single-head translation 
 to linear rulesets with bounded predicate arity. Moreover, for any 
linear ruleset $\ruleset$, whose single-head translation is $\ruleset'$, any NFA $\automaton$, and any type $\ttype$ based upon a predicate in $\ruleset$,
type $\ttype$ belongs to $\mathtt{Loop}'(s,i,s',j)$ in the table produced by running the modified algorithm on input ($\ruleset', \automaton$) 
iff $\ttype$ admits an $(s,i,s',j)$-loop (w.r.t.\ $\ruleset$). 
\end{proposition}

\begin{proof} The only predicates in $\ruleset'$ whose arity exceeds the maximum predicate arity of $\ruleset$ are the newly introduced predicates, and each such predicate $p_\erule$ can only appear in the head of a single rule of the form $\body{\erule} \rightarrow \exists y_1 \dots \exists y_\ell \, p_\erule(x_1, \dots, x_k,y_1, \dots, y_\ell)$ and in the body of rules $p_\erule(x_1, \dots, x_k,y_1, \dots, y_\ell) \rightarrow q(\vec u)$ (for each $q(\vec u)$ which is a head atom of $\erule$). 
As a consequence, the only types with predicate $p_\erule$ which are necessary to consider in Algorithm \ref{algo-loop-table-creation} are those which contain the trivial partition $\{\{k+1\}, \ldots \{k+\ell\} \}$. Indeed, by using the same arguments as in Prop.~\ref{prop-complexity-algo1}, 
we can show that if $\ttype$ is any type for an original predicate in $\ruleset$, then $\ttype$ belongs to $\mathtt{Loop}'(s,i,s',j)$ in the table produced by the modified algorithm on input ($\ruleset', \automaton$) iff $\ttype$ admits an $(s,i,s',j)$-loop. 
Finally, to establish polynomial complexity, 
it suffices to remark that {there can only be polynomially many types for the new predicates (since $k$ is bounded by a constant).}
\end{proof}

The remainder of the decision procedure relies on the following ideas (Algorithm \ref{algo-rpq-answering}): starting from a constant $a$ and the initial state of $\automaton$, we guess the next constant $d$ in $\instance$ on a path from $a$ to $b$ and the state of $\automaton$ after taking this step (Line \ref{line-guess}). {Note that $a$ and $d$ may be equal.} 
We then check that this choice is valid, i.e., there is indeed a path from $a$ to the guessed constant $d$ which takes the automaton from the initial state to the current guessed state. This can be done either by checking that a corresponding  binary atom is entailed (Lines 6-7), 
or by checking that a path going through the anonymous part of the chase allows us to reach the next constant in the required state, using the \texttt{Loop} table (Lines 8-9). 
We repeat this procedure until we reach the constant $b$ in a final state, or hit the maximal path length. 

\begin{algorithm}[t]
 \KwIn{An NFA {$\automaton = (S,\alphabet,\delta,s_0,F)$}, an instance $\instance$, a set of linear rules $\ruleset$,
 $(a,b) \in \terms{\instance} \times \terms{\instance}$}
 \KwOut{Yes if and only if $(a,b)$ is a certain answer to the RPQ defined by $\automaton$}
{Compute \texttt{Loop} table for $\ruleset$ and $\automaton$ (using Algorithm \ref{algo-loop-table-creation});}\\ 
 \texttt{current} = $(a,s_0)$;\\
 \texttt{count} $= 0$, \texttt{max} $= |S|\times|\instance|$; 
 \\
 \While{\texttt{count} $<$ \texttt{max } and \texttt{current} $\not \in \{(b,s_f) \mid s_f \in F\}$}
 {
Define $(c,s) =$ \texttt{current};\\
 Guess $(d,s')$ together with either $(s,\predicate,s') \in \transitionFunction$ or $T,i_c,i_d$  such that $\ttype \in \texttt{Loop}(s,i_c,s',i_d)$;\\\label{line-guess}
\uIf{$(s,\predicate,s')$ was guessed and $(\instance,\ruleset \not\models \predicate(c,d))$ \label{line-binary-check}}
		{\KwRet{No}}
\uElseIf{$\ttype,i_c,i_d$ was guessed and $\instance$ does not contain any atom $\atom$ of type $\ttype$ such that $\atom[i_c]=c$ and $\atom[i_d]=d$ \label{line-atom-specification}}{\KwRet{No}}\label{line-loop-check}
\Else{
\texttt{current} = $(d,s')$, \texttt{count} = \texttt{count} $+ 1$;\\
}}
\lIf{\texttt{current}= $(b,s_f)$ for some $s_f \in F$}{\KwRet{Yes} \textbf{else} \KwRet{No}}
 \caption{RPQ answering over linear rules  }
 \label{algo-rpq-answering}
  \end{algorithm}

\newpage  
The following lemma proves an invariant that will be used to establish the correctness of Algorithm \ref{algo-rpq-answering}.  

\begin{lemma}\label{prop-invariant-algo-rpq}
At the beginning of each iteration of the while loop of Algorithm \ref{algo-rpq-answering}, it holds that there is a path in $\chase{\instance}{\ruleset}$ from $a$ to the first element of \texttt{current} that takes the NFA $\automaton$ from the initial state $s_0$ to the state in the second argument of \texttt{current}.
\end{lemma}

\begin{proof}
At the beginning of the first iteration of the while loop, \texttt{current}  is equal to $(a,s_0)$. Thus, the path $a$, whose label is $\emptyWord$, goes from $a$ to $a$ and $\emptyWord \in \subLanguage{\automaton}{s_0}{s_0}$.

Let $(a_i,s_i)$ be the content of \texttt{current}  at the beginning of the $i^\textrm{th}$ iteration of the while loop. Let $w_i$ be the label of a path from $a_0$ to $a_i$ such that $w_i \in \subLanguage{\automaton}{s_0}{s_i}$. If there is an $(i+1)^\textrm{th}$ iteration, either $(s,\predicate,s')$ or $(\ttype,i_c,i_d)$ has been guessed, and the corresponding check was successful. Let us consider each case:
\begin{itemize}
\item if $(s,\predicate,s')$ has been guessed and checked, then $\predicate \in \twoWayBinaryPred$, and there is a path from $a_i$ to $a_{i+1}$ in $\chase{\instance}{\ruleset}$ labeled by $\predicate$. Moreover, $\predicate$ labels an edge from $s$ to $s'$ in $\automaton$. We can thus define $w_{i+1} = w_i.\predicate$
\item if $(\ttype,i_c,i_d)$ has been guessed, it means that $\ttype$ belongs to  $\texttt{Loop}(s_i,i_c,s_{i+1},i_d)$. By the definition of $\texttt{Loop}$, there is a path $p$ (in the anonymous part) from any term 
 at position $i_c$ of an atom of type $\ttype$ to the position $i_d$ of an atom of type $\ttype$ such that $\pathLabel{p} \in \subLanguage{\automaton}{s}{s'}$. Let $\atom$ be as defined Line~\ref{line-atom-specification}. As $\instance, \ruleset \models \atom$, where $\type{\atom} = \ttype$, $a_i$ appears at position $i_c$ of $\atom$, and $a_{i+1}$ appears at position $i_d$ of $\atom$, there is such a path from $a_i$ to $a_{i+1}$. We can thus set $w_{i+1} = w_i.p$. \qedhere
\end{itemize}
\end{proof}


The next proposition establishes the correctness of Algorithm \ref{algo-rpq-answering}. To prove the correctness when the algorithm outputs Yes, we rely on the invariant given by the previous lemma, which allows us to construct a witnessing path in $\chase{\instance}{\ruleset}$. 
To prove the correctness when the algorithm outputs No, we consider any path of minimal length from $a$ to $b$ that satisfies the input RPQ and build a sequence of guesses associated with it that is necessarily accepted by the algorithm. 

\begin{proposition}\label{prop-correctness-algo-rpq}
There is an execution of Algorithm \ref{algo-rpq-answering} that outputs Yes if and only if the RPQ given by $\automaton$ 
is entailed from $(\instance, \ruleset)$.
\end{proposition}

\begin{proof}
($\Rightarrow$) If the algorithm %
outputs Yes, the while loop has been exited with \texttt{current}  equal to $(b,s_f)$, with $s_f$ a final state of $\automaton$. By Lemma \ref{prop-invariant-algo-rpq}, this means that there is a path from $a$ to $b$ whose label takes $\automaton$ from $s_0$ to $s_f$, hence is accepted by $\automaton$. This shows that whenever Algorithm \ref{algo-rpq-answering} accepts, $(a,b)$ is a certain answer to the RPQ given by $\automaton$. %

($\Leftarrow$) If $(a,b)$ is a certain answer to the RPQ based upon $\automaton$,  then there is path of minimal length $p = \const'_0r_1\const'_1\ldots r_n\const'_n$ from $a=\const'_0$ to $b=\const'_n$ in $\chase{\instance}{\ruleset}$ such that $\pathLabel{p} = r_1 \ldots r_n \in \subLanguage{\automaton}{s_0}{s_f}$ for some final state $s_f$. Let $s'_0s'_1\ldots s'_n$ be a sequence of states of $\automaton$ such that $s'_n$ is a final state of $\automaton$ and for every $1 \leq i \leq n$, $(s_{i-1}, r_i, s_i) \in \delta$. 
Since $p$ is of minimal length, there is no pair $(i,j)$ with $i \not = j$ such that $(\const_i,s_i) = (\const_j,s_j)$. Let us consider the sequence $p' = ((\const_i,s_i))_i$ such that:
\begin{itemize}
\item for any $i$, $\const_i$ is the $i^\textrm{th}$ constant, say $\const'_{k_i}$, in $p$ belonging to $\terms{\instance}$;
\item for any $i$, $s_i = s'_{k_i}$.
\end{itemize} 
Moreover, for any $i$, if $k_{i+1} = k_i+1$, we define $\aux_i = (s_i,r_{i+1},s_{i+1})$. Otherwise, let $\aux_i = (\type{\atom},i_c,i_d)$,where:
\begin{itemize}
\item $\atom$ is such that $\atom \in \instance$ and $\type{\atom} \in \texttt{Loop}(s_i,i_c,s_{i+1},i_d)$;
\item $\const_{k_{i}}$ appears at position $i_c$ of $\atom$ and $\const_{k_{i+1}}$ appears at position $i_d$ of $\atom$.
\end{itemize}
In the second case, we can define $\aux_i$ in such a way, as the path $p_s = a'_{k_i}r_{k_i+1}\ldots a'_{k_{i+1}}$ goes from $a_{k_i}$ to $a_{k_{i+1}}$ and $\label{p_s}$ belongs to $\subLanguage{\automaton}{s_i}{s_{i+1}}$ by definition of $s_i$. 
We show that the sequence of guesses $(\const_i,s_i,\aux_i)$ leads Algorithm \ref{algo-rpq-answering} to accept. Since $p$ is minimal, the length of $p'$ is less than $|\automaton|\times |\instance|$. Moreover, $\const_n = b$ and $s_f$ is a final state.  Thus, the only way for Algorithm \ref{algo-rpq-answering} to reject with this sequence of guesses is to reject during checks, \ie, one of the checks performed at Lines \ref{line-binary-check} and %
\ref{line-atom-specification} fails. Let $(\const_i,s_i,\aux_i)$ be the guess at one of the steps. If $\aux_i$ is of the form $(s_i, r_{i+1}, s_{i+1})$,  
then $\const_{k_i}$ and $\const_{k_{i+1}}$ are consecutive elements in $p$, and there is an atom $r_{i+1}(\const_{k_i},\const_{k_{i+1}})$ in $\chase{\instance}{\ruleset}$. Thus, $r_{i+1}(\const_{k_i},\const_{k_{i+1}})$ is entailed by $\instance$ and $\ruleset$, and the check at Line \ref{line-binary-check} 
is successful. If $\aux_i$ is of the form $(\type{\atom},i_c,i_d)$, then there is $\atom \in \instance$ such that $\type{\atom} \in \texttt{Loop}(s_i,i_c,s_{i+1},i_d)$, and with $a_{k_i}$ (resp.\ $a_{k_{i+1}}$) appearing at position $i_c$ (resp. $i_d$) of $\atom$. The atom $\atom$ fulfills the conditions of Line~\ref{line-atom-specification}. Thus the defined sequence never triggers a rejection from Algorithm \ref{algo-rpq-answering}, which concludes the proof.
\end{proof}

By analyzing the time and space required by Algorithm \ref{algo-rpq-answering}, we obtain the following upper-bounds for \rpqAnswering\ in the presence of linear existential rules. 

\begin{theorem}
\label{thm-linear-rpq-answering}
\rpqAnswering\ in the presence of linear existential rules 
 is:
\begin{itemize}
\item in \textsc{NL} in data complexity;
\item in \textsc{PTime} in combined complexity with bounded arity;
\item in \textsc{ExpTime} in combined complexity with unbounded arity.
\end{itemize}
\end{theorem}

\begin{proof}
Algorithm \ref{algo-rpq-answering} is a non-deterministic algorithm that needs to keep in memory the current state, the current constant, and the number of iterations done so far. It performs two types of operations: 
checking entailment of an atom with linear rules and accessing the contents of the \texttt{Loop} table (more precisely, deciding whether $\ttype \in \texttt{Loop}(s,i_c,s',i_d)$).  Hence, it can be seen as an NL algorithm making oracle calls whenever an entailment check is performed or a cell of \texttt{Loop} is retrieved. Entailment checks are in NL in data complexity (more precisely, AC$_0$, see Section \ref{subsection:QA}), and \texttt{Loop} is independent from the data: the overall algorithm thus runs in NL in data complexity. In combined complexity with bounded arity, entailment checks can be performed in \textsc{PTime} (Section \ref{subsection:QA}), while \texttt{Loop} can be computed in polynomial time {(even if multiple head atoms are allowed, due to Proposition \ref{loop-bounded-arity})}. The overall algorithm is thus in \textsc{PTime} with bounded arity. In the unbounded arity case, the entailment checks can be performed in \pspace\ (Section \ref{subsection:QA}), while the \texttt{Loop} table can be computed in \textsc{ExpTime}: the overall algorithm thus runs in \textsc{ExpTime}. 
\end{proof}

\section{RPQ Answering under Linear Rules: Lower Bound}
\label{sec-linear-lower}

The data complexity (resp. combined complexity) of RPQs under linear rules (resp. linear rules with bounded arity) is already known to be \textsc{NL}-hard (resp. \ptime-hard) \citep{DBLP:journals/jair/BienvenuOS15}. As these bounds match the upper bounds obtained in the preceding section, we focus on providing a matching \exptime\ lower bound for the combined complexity of evaluating RPQs under linear rules of unbounded arity. 

{We recall that  \exptime\ can be reformulated as \textsc{APSpace}, the class of problems that can be solved in polynomial space by an alternating Turing Machine (ATM). 
Our proof relies on this equivalence and shows that a \pspace\ ATM can be simulated by means of linear rules. 
Without loss of generality, we consider an  ATM where each non-final universal (resp. existential) state has exactly two existential (resp. universal) sucessors. }

It is already known that \pspace\ TMs can be simulated by means of linear rules \citep{DBLP:journals/iandc/GottlobP03}. In the following, we 
  explain how to adapt this construction to simulate ATMs. To simplify the presentation, we will use rules with both constants and multiple atoms in the head.
 {This can be done without loss of generality, since we can apply classical transformations to turn such rules into constant-free atomic-head rules (see, e.g., Definition 1.37 in \citep{DBLP:phd/hal/Rocher16} for the removal of constants, and Section \ref{sec-ER} for the single-head translation). These translations are both polynomial in the size of the input ruleset and preserve the property of being a linear ruleset. }

The key 
point in the construction from \citep{DBLP:journals/iandc/GottlobP03} is the following:  the configuration of a  \pspace\  TM \turingMachine\ is represented by a single atom of polynomial arity. The initial configuration can thus be represented by an 
instance $\instance_\turingMachine$ containing a single atom. Then, for each transition of the TM, polynomially many linear rules are created, each one representing the action of the transition on a cell at a given position. All these rules are part of $\ruleset_\turingMachine$. The initial configuration of the TM is accepted if and only if an atom that encodes a configuration having an accepting state is entailed by $\instance_\turingMachine$ and $\ruleset_\turingMachine$.

We modify this construction in the following way to deal with an ATM: to each atom, we add two positions, that will act as ``input'' and ``output'' positions. Moreover, we maintain the following property: {for any atom $\atom$ encoding a configuration $c$,} there is a path, whose edges are all labeled by the same predicate $\pathAtom$, from the {term $i_c$ in} input position of $\atom$ to the {term $o_c$ in} output position of $\atom$, entailed by $\chase{\instance_{\{\atom\}}}{\ruleset_\turingMachine}$ if and only if the configuration represented by $\atom$ is accepted by $\turingMachine$. Acceptance occurs in the following cases:
\begin{itemize}
\item the state of the current configuration is accepting. It is then enough to add a $\pathAtom$-edge from $i_c$ to $o_c$; this is possible as the Turing Machine is assumed to never leave an accepting state;
\item the current state is existential and
one of the two successor configurations is accepting.  
We thus add $\pathAtom$-edges from the input of the current configuration to the input of the two children, and from the output of the two children to the output of the current configuration;
\item the current state is universal, and both successor configurations are accepting. 
We thus add $\pathAtom$-edges from the input of the current configuration to the input of the first successor configuration, 
then from the output of that configuration to the input of the other successor, 
and lastly from the output of the second successor to the output of the current configuration.
\end{itemize}
We now formalize the construction sketched above, staying as close as possible to the notations in \citep{DBLP:journals/iandc/GottlobP03}.
{Given an ATM $\turingMachine$ and an input $x$, such that $\turingMachine$ runs in $\poly(|x|)$ space for a polynomial $\poly$, we can represent a configuration $c$ reached during the computation  by storing 
 the content of the first $\poly(|x|)$ cells, as well as the position of the head of $\turingMachine$ and the current state of $\turingMachine$. }
 Adding input and output positions, this can be encoded by a predicate $\fun{conf}$ of arity $2\poly(|x|) + 3$:
$$\fun{conf}(i_c,state,cell_1,cur_1,cell_2,cur_2,\ldots,cell_{\poly(|x|)},cur_{\poly(|x|)},o_c),$$
where $state$ contains the state identifier, $cell_i$ represents the content of the $i^\textrm{th}$ cell, $cur_i$ 
is equal to $1$ if the head of $\turingMachine$ is on cell $i$ and $0$ otherwise, and $i_c$ and $o_c$ are the input and output terms 
 of this atom. We say that the above atom \emph{represents} configuration $c$. Given an atom $\atom$, the term at its input (resp. output) position is denoted by $\inputOf{\atom}$ (resp. $\outputOf{\atom}$).  We denote by $\instance_{\turingMachine,x}$ the instance containing a single atom representing the initial configuration  of $\turingMachine$ on input $x$.

For every accepting state $q_f$, we create the following rule {(where $q_f$ denotes a constant)}: 
\begin{align}
\label{rule-accepting-state}
\fun{conf}(i_c,q_f,\ldots,o_c) \rightarrow \pathAtom(i_c,o_c)
\end{align}

For each existential state  $q$ and each transition $\transitionFunction(q,\tapeLetter) = \{(q',\tapeLetter',L),(q'',\tapeLetter'',L)\}$, and for every position $i$ on the tape, we create the following rule {(where $q$, $q'$, $q''$, $\gamma$, $\gamma'$ and $\gamma''$ denote constants)}:
\begin{multline}
\label{rule-existential-state}
\fun{conf}(i_c,q,cell_1,cur_1,\ldots,cell_{i-1},0,\gamma,1,\ldots,o_c) \rightarrow \\
\exists i_{c'},o_{c'},i_{c''},o_{c''}	~ \fun{conf}(i_{c'},q',cell_1,cur_1,\ldots,cell_{i-1},1,\gamma',0,\ldots,o_{c'}), \\
\fun{conf}(i_{c''},q'',cell_1,cur_1,\ldots,cell_{i-1},1,\gamma'',0,\ldots,o_{c''}), \\
 \pathAtom(i_c,i_{c'}), \pathAtom(o_{c'},o_c),\pathAtom(i_c,i_{c''}), \pathAtom(o_{c''},o_c)
\end{multline}

For each universal state $q$ and each transition $\transitionFunction(q,\tapeLetter) = \{(q',\tapeLetter',L),(q'',\tapeLetter'',L)\}$, and for every position $i$ on the tape, we create the following rule (with the same constants as in the preceding rule):
\begin{multline}
\label{rule-universal-state}
\fun{conf}(i_c,q,cell_1,cur_1,\ldots,cell_i,0,\tapeLetter,1,\ldots,o_c) \rightarrow \\
\exists i_{c'},o_{c'}~,i_{c''},o_{c''}~ \fun{conf}(i_{c'},q',cell_1,cur_1,\ldots,cell_i,1,\tapeLetter',0,\ldots,o_{c'}), \\
\qquad \qquad\qquad \qquad\,\, \fun{conf}(i_{c''},q'',cell_1,cur_1,\ldots,cell_i,1,\tapeLetter'',0,\ldots,o_{c''}), \\
 \pathAtom(i_c,i_{c'}), \pathAtom(o_{c'},i_{c''}), \pathAtom(o_{c''},o_c) 
\end{multline}

{We proceed similarly to translate transitions in which the head is moving to the right. }
Figure \ref{fig-existential-gadget} (resp. Figure~\ref{fig-universal-gadget}) illustrates the functioning of rules of type (\ref{rule-existential-state}) (resp.	 type (\ref{rule-universal-state})).
We denote by $\ruleset_{\turingMachine,x}$ the set of all the rules defined above. Note that $x$ is required to determine the arity of $\fun{conf}$.%

The following proposition  establishes the correctness of the reduction. 

\begin{figure}
\begin{center}
\begin{tikzpicture}[every node/.style={block},>=triangle 45]
   \draw (0,0) node (A) {$cell_i$};
   \node (B) [left=of A] {$\ldots$};
   \node (C) [left=of B] {$i_c$};
   \node (D) [right=of A] {$\ldots$};
   \node (E) [right=of D] {$o_c$};

   \draw (-2.5,2) node (Ap) {$cell'_i$};
   \node (Bp) [left=of Ap] {$\ldots$};
   \node (Cp) [left=of Bp] {$i_{c'}$};
   \node (Dp) [right=of Ap] {$\ldots$};
   \node (Ep) [right=of Dp] {$o_{c'}$};
   
   \draw (2.5,2) node (As) {$cell''_i$};
   \node (Bs) [left=of As] {$\ldots$};
   \node (Cs) [left=of Bs] {$i_{c''}$};
   \node (Ds) [right=of As] {$\ldots$};
   \node (Es) [right=of Ds] {$o_{c''}$};
   
   \draw[->] (C.north) -- (Cp.south); 
   \draw[->] (Ep.south) -- (E.north);
   \draw[->] (C.north) -- (Cs.south); 
   \draw[->] (Es.south) -- (E.north);
\end{tikzpicture}
\end{center}
\caption{The existential gadget}
\label{fig-existential-gadget}
\end{figure}
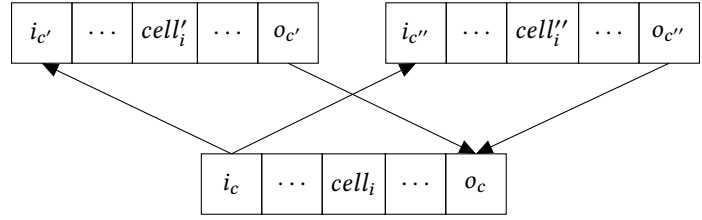

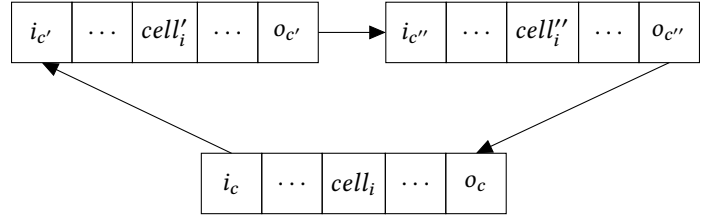
\begin{figure}
\begin{center}
\begin{tikzpicture}[every node/.style={block},>=triangle 45]
   \draw (0,0) node (A) {$cell_i$};
   \node (B) [left=of A] {$\ldots$};
   \node (C) [left=of B] {$i_c$};
   \node (D) [right=of A] {$\ldots$};
   \node (E) [right=of D] {$o_c$};

   \draw (-2.5,2) node (Ap) {$cell'_i$};
   \node (Bp) [left=of Ap] {$\ldots$};
   \node (Cp) [left=of Bp] {$i_{c'}$};
   \node (Dp) [right=of Ap] {$\ldots$};
   \node (Ep) [right=of Dp] {$o_{c'}$};
   
   \draw (2.5,2) node (As) {$cell''_i$};
   \node (Bs) [left=of As] {$\ldots$};
   \node (Cs) [left=of Bs] {$i_{c''}$};
   \node (Ds) [right=of As] {$\ldots$};
   \node (Es) [right=of Ds] {$o_{c''}$};
   
   \draw[->] (C.north) -- (Cp.south); 
   \draw[->] (Ep.east) -- (Cs.west);
   \draw[->] (Es.south) -- (E.north);
\end{tikzpicture}
\end{center}
\caption{The universal gadget}
\label{fig-universal-gadget}
\end{figure}

\begin{proposition}\label{tm-prop}
Let $\turingMachine$ be a \pspace\ ATM and let $\atom$ be an atom of \chase{$\instance_{\turingMachine,x}$}{$\ruleset_{\turingMachine,x}$} representing a configuration $\confOf{\atom}$. Then $\confOf{\atom}$ is an accepting configuration of $\turingMachine$ if and only if there is a path in $\chase{\instance_{\turingMachine,x}}{\ruleset_{\turingMachine,x}}$ from $\inputOf{\atom}$ to $\outputOf{\atom}$ whose label belongs to $\pathAtom^*$. 
\end{proposition}

\begin{proof}
($\Leftarrow$) Let $\atom \in \chase{\instance_{\turingMachine,x}}{\ruleset_{\turingMachine,x}}$ represent a configuration $\confOf{\atom}$, and let $C_\atom$ be the restriction of $\chase{\instance_{\turingMachine,x}}{\ruleset_{\turingMachine,x}}$ {to $\atom$  and the atoms generated from $\atom$}. We show by induction on the number of atoms of $C_\atom$ that when the desired path exists in $C_\atom$, $\confOf{\atom}$ is accepting. Note that the induction is well-founded as the considered Turing Machines terminate.
\begin{itemize}
\item If $C_\atom$ contains a single atom, then there can be no path in $\chase{\instance_{\turingMachine,x}}{\ruleset_{\turingMachine,x}}$ witnessing $\pathAtom^*(\inputOf{\atom},\outputOf{\atom})$. Suppose then that $C_\atom$ contains two atoms. In this case, the only atom in $C_\atom$ other than %
$\atom$  must be $\pathAtom(\inputOf{\atom},\outputOf{\atom})$. The only way to derive such an atom is to apply a rule of the form 
 (\ref{rule-accepting-state}), which is applied if and only if $\confOf{\atom}$ is in an accepting state, hence $\confOf{\atom}$ is an accepting configuration of $\turingMachine$. 
\item 
Next, assume that the result holds for any atom $\atom$ such that $C_\atom$ has less than $n$ atoms,  
and let $\atom$ be an atom such that $C_\atom$ contains $n$ atoms, {with $n>2$}.  
We distinguish two cases:
\begin{itemize}
\item Case 1: the state of $\confOf{\atom}$ is existential. Then, since the rules of type (\ref{rule-existential-state}) must be satisfied, $C_\atom$ contains atoms $\atom_1$ and $\atom_2$  representing the successor configurations of $\confOf{\atom}$. 
The existence of a path from $\inputOf{\atom}$ to $\outputOf{\atom}$ implies that there is either  
a path from $\inputOf{\atom_1}$ to $\outputOf{\atom_1}$ or a path from $\inputOf{\atom_2}$ to $\outputOf{\atom_2}$. 
To see why, observe that every $\pathAtom$-atom involving $\inputOf{\atom}$ or $\outputOf{\atom}$ is added either by the same rule application {that} created $\atom$ 
or by a rule of type ($\ref{rule-existential-state}$) applied to $\atom$. Only atoms of the second kind 
(refer to Fig.\ \ref{fig-existential-gadget}, left) can belong to a shortest path from $\inputOf{\atom}$ to $\outputOf{\atom}$, since atoms of the first kind have $\inputOf{\atom}$ (resp.\ $\outputOf{\atom}$) as second (resp.\ first) argument. 
If we have a path from $\inputOf{\atom_1}$ to $\outputOf{\atom_1}$, then we can apply the induction assumption to $\atom_1$ to get that $c(\alpha_1)$ is an accepting configuration, which implies that $\confOf{\atom}$ is also accepting. We can proceed analogously if we have a path from $\inputOf{\atom_2}$ to $\outputOf{\atom_2}$. 
\item Case 2: the state of $\confOf{\atom}$ is universal. As the rules of type (\ref{rule-universal-state}) must be satisfied, the existence of a path from $\inputOf{\atom}$ to $\outputOf{\atom}$ implies the existence of a path from $\inputOf{\atom_1}$ to $\outputOf{\atom_1}$ and a path from $\inputOf{\atom_2}$ to $\outputOf{\atom_2}$, 
where $\atom_1$ and $\atom_2$ represent the successor configurations of $c(\atom)$ 
{(refer to Fig.\ \ref{fig-universal-gadget}).} 
By the induction assumption, $\confOf{\atom_1}$ and $\confOf{\atom_2}$ are both accepting configurations, which means that $\confOf{\atom}$ is also accepting.
\end{itemize}
\end{itemize}

($\Rightarrow$) We prove the other direction by induction on the number of transitions that need to be performed to prove that $c(\atom)$ is accepted by $\turingMachine$.
\begin{itemize}
\item If no transitions are required, this means that $c(\atom)$ is in an accepting state. Thus, Rule (\ref{rule-accepting-state}) is applicable, and $\pathAtom(\inputOf{\atom},\outputOf{\atom})$ is present 
 in $\chase{\instance_{\turingMachine,x}}{\ruleset_{\turingMachine,x}}$.
\item Assume the result holds up to $n$ required transitions. We distinguish two cases:
\begin{itemize}
\item \sloppypar{Case 1: the state of $\confOf{\atom}$ is existential. As $\confOf{\atom}$ is accepting, %
this means that one of its two successor configurations, say $\confOf{\atom_1}$, %
is accepting. Moreover, the number of transitions required to accept $\confOf{\atom_1}$ is strictly smaller than for $\confOf{\atom}$. By the induction assumption, $\pathAtom^*(\inputOf{\atom_1},\outputOf{\atom_1})$ is present in 
$\chase{\instance_{\turingMachine,x}}{\ruleset_{\turingMachine,x}}$. As $\pathAtom(\inputOf{\atom},\inputOf{\atom_1})$ and $\pathAtom(\outputOf{\atom_1},\outputOf{\atom})$ are also present (since the rules of the form (\ref{rule-existential-state}) generate them), this proves that $\pathAtom^*(\inputOf{\atom},\outputOf{\atom})$ is present as well.}
\item \sloppypar{Case 2:  the state of $\confOf{\atom}$ is universal. As $\confOf{\atom}$ is accepting, %
 this means that its two successor configuration are also accepting.  By the induction assumption, this means that $\pathAtom^*(\inputOf{\atom_1},\outputOf{\atom_1})$ and $\pathAtom^*(\inputOf{\atom_2},\outputOf{\atom_2})$ are present in $\chase{\instance_{\turingMachine,x}}{\ruleset_{\turingMachine,x}}$. As the rules of the form
 (\ref{rule-universal-state}) also generate $\pathAtom(\inputOf{\atom},\inputOf{\atom_1})$, $\pathAtom(\outputOf{\atom_1},\inputOf{\atom_2)}$, and $\pathAtom(\outputOf{\atom_2},\outputOf{\atom})$, this proves that $\pathAtom^*(\inputOf{\atom},\outputOf{\atom})$ is present in $\chase{\instance_{\turingMachine,x}}{\ruleset_{\turingMachine,x}}$.} \qedhere
\end{itemize}
\end{itemize}
\end{proof}

Now let $\turingMachine$ be a \pspace\ ATM, $x$ be an input to $\turingMachine$, and $\atom$ be the unique atom in $\instance_{\turingMachine,x}$. Then by Proposition \ref{tm-prop}, $\confOf{\atom}$ is an accepting configuration of $\turingMachine$ if and only if
 $\instance_{\turingMachine,x},\ruleset_{\turingMachine,x} \models \pathAtom^*(\inputOf{\atom},\outputOf{\atom})$; 
 {in other words: if and only if $(\inputOf{\atom},\outputOf{\atom})$ is a certain answer to the RPQ $\pathAtom^*(x,y)$ on the KB $(\instance_{\turingMachine,x},\ruleset_{\turingMachine,x})$.}
This, together with known results, yields the following lower bounds:

\begin{theorem}
\rpqAnswering\ in the presence of linear existential rules is \textsc{NL}-hard in data complexity, \ptime-hard in combined complexity with bounded arity and \exptime-hard in combined complexity without arity bound, even for a fixed RPQ.
\end{theorem}

\paragraph{Note}
The preceding reduction can be used to show that atomic query entailment under rulesets composed of linear rules and transitivity rules (AQELT)  is \exptime-hard in combined complexity, already with a single transitivity rule. Indeed, a pair of constants $(a,b)$ is an answer to an RPQ $\pathAtom^+(x,y)$ on a KB $\kb= \kblong$ with $\ruleset$ a linear ruleset iff the atomic Boolean query $p'(a,b)$ is entailed by the KB $(\instance, \ruleset \cup 
\{p(x,y)  \rightarrow p'(x,y); p'(x,y) \land p'(y,z) \rightarrow p'(x,z) \}$, where $p'$ is a fresh predicate introduced to avoid interactions between the new rules and those from $\ruleset$.

Assuming \exptime$\neq$\pspace, our result is in contradiction with Theorem 5 in~\citep{DBLP:conf/ijcai/BagetBMR15}, which purports to show a \pspace\ upper bound for the AQELT problem. After reexamining the proofs, the authors of the latter work have identified the flaw, which occurs in the analysis of the combined complexity of their rewriting-based decision procedure. It turns out that their procedure runs in exponential time, rather than in polynomial space (their \textsc{NL} upper bound in data complexity remains valid). Combining our lower bound with their procedure shows that the AQELT problem is \exptime-complete in combined complexity. {The incorrect claim from~\citep{DBLP:conf/ijcai/BagetBMR15} was corrected in the companion report with full proofs \citep{DBLP:journals/corr/BagetBMR15}.}

\section{CRPQ Answering under Linear Rules: Upper Bound}
\label{sec-crpq-linear-upper}

We now study the complexity of CRPQ answering under linear rules. The algorithm we propose to solve this problem is based on two notions. First, the notion of a \emph{proof scheme}, which (under some validity conditions) captures a finite portion of the chase. Second, the notion of a \emph{match} of a (Boolean) CRPQ in a proof scheme, which attests that this CRPQ is entailed by the portion of the chase captured by this proof scheme. The algorithm roughly works as follows: it enumerates all proof schemes of `small' size and checks whether one of them is valid and admits a match of the query.  To obtain our complexity results, we will upper bound both the size of the proof schemes that need to be considered and the cost of the validity check.

Section \ref{sec-crpq-proof-schemes} introduces the fundamental notions.
 It concludes with Proposition\ \ref{prop-match-valid}, which  states that a query $\query$ is entailed by a knowledge base if and only if there exists a match of $\query$ in some valid proof scheme. Section \ref{sec-crpq-upper-proof} proves (the hard direction of) that proposition, by exhibiting a suitable proof scheme whenever $\query$ is entailed by the knowledge base.
 Finally, Section \ref{sec-crpq-complexity} analyzes the computational resources required to enumerate proof schemes and to check their validity as well as the existence of a match of $\query$ in one of them.

\subsection{Proof Schemes}
\label{sec-crpq-proof-schemes}
 We recall that a CRPQ $\query$ may contain both standard atoms and path atoms whose predicate is defined by an automaton. 
Given a pair of atoms, it will be important to keep track of the paths in the chase that link terms occurring in these atoms and correspond to words that can be read by the automata appearing in $\query$, i.e., that take one of these automata from a state $\state_1$ to a state $\state_2$.  
The next definition of \emph{transitions} between atoms plays this role.

\begin{definition}[Transition]
Let $A$ be a set of atoms.
Let $\atom_1$ and $\atom_2$ be two atoms of~$A$ of arity $k_1$ and $k_2$ respectively. Given a query $q$, the \emph{set of transitions from $\atom_1$ to $\atom_2$ in $A$}, denoted by $\transitionAtomSet{\atom_1}{\atom_2}{q}{A}$, is the set of quadruples $(i_1,\state_1,i_2,\state_2)$ with $1 \leq i_1 \leq k_1$ and $1\leq i_2 \leq k_2$, such that there are an automaton $\automaton$ in $\query$ and a path $\path$ in $A$ going from $\atom_1[i_1]$
to $\atom_2[i_2]$ such that $\pathLabel{\path}$ $ \in \subLanguage{\automaton}{s_1}{s_{2}$}.  
More generally, by \emph{transition}, we will mean a quadruple 
$(i_1,\state_1,i_2,\state_2)$ where $\state_1,\state_2$ are states of an automaton in $\query$, and $i_1,i_2$ do not exceed the maximum predicate arity in $\instance \cup \ruleset$.\end{definition}

This allows us to introduce the key technical notion underlying our approach, namely, \emph{proof schemes}. Briefly, a proof scheme is a directed forest whose nodes are atoms, together with  transitions for different pairs of atoms in this forest. 
Intuitively, under some validity conditions defined later (Definition \ref{def-validity}), a proof scheme corresponds to a subset of the chase, with the transitions placing requirements on the paths linking these atoms. 

\begin{definition}[Proof scheme]
\label{def-proof-scheme}
 A \emph{proof scheme} is a pair $\proofScheme = (\psforest,\pstransitions)$ such that:
 \begin{enumerate}
  \item\label{item-ps-df} $\psforest = (V,E)$ is a directed forest;
  \item\label{item-ps-set} $V$ is a set of atoms;
  \item\label{item-ps-connected} for each null $\chasenull$, the set of nodes in which $\chasenull$ occurs form a connected subgraph of $\psforest$;
  \item\label{item-ps-transition} $\pstransitions$ is a set of pairs $(\tau,(\atom_1,\atom_2))$, where $\atom_1$ and $\atom_2$ belong to $V$ and $\tau$ is a transition
   from $\atom_1$ to  $\atom_2$. 
   \item\label{item-ps-local} for each $(\tau,(\atom_1,\atom_2)) \in \pstransitions$ , either $\atom_1 = \atom_2$, $(\atom_1,\atom_2) \in E$, $(\atom_2,\atom_1) \in E$, or both $\atom_1$ and $\atom_2$ are roots.
 \end{enumerate}
 
 We denote by $\vars{\proofScheme}$ and $\terms{\proofScheme}$ the variables and terms occurring in atoms of $V$, i.e., $\vars{\proofScheme} = \vars{V}$ and $\terms{\proofScheme} = \terms{V}$.
 
\end{definition}

In a proof scheme, transitions 
are assigned to  pairs of atoms
that comply with specific conditions: the atoms are equal, or both are roots, or one is the parent of the other in a tree. 
To check the existence of paths that link terms occurring in arbitrary pairs of atoms, we can compose these transitions. The \emph{saturation} of a proof scheme, defined next, enriches the proof scheme with all the transitions that can be obtained by composition.

\begin{definition}[Saturation of a proof scheme]
\label{def-saturation}
\sloppypar{
Let $\proofScheme = (\psforest,\pstransitions)$ be a proof scheme. The \emph{saturation} of $\proofScheme$ is the pair $(\psforest,\pstransitions')$ where $\pstransitions'$ is the smallest set such that $\pstransitions\subseteq \pstransitions'$ and 
$\{((i_1,\state_1,i_2,\state_2),(\atom_1,\atom_2)), ((i_2,\state_2,i_3,\state_3),(\atom_2,\atom_3))\} \subseteq \pstransitions'$ implies $((i_1,\state_1,i_3,\state_3),(\atom_1,\atom_3)) \in \pstransitions'$.}
\end{definition}

\begin{definition}[Match in a proof scheme]
\label{def-match-proof-scheme}
 Let 
 $\proofScheme = (\psforest, \pstransitions)$ be a proof scheme with $\psforest = (V,E)$. A \emph{match} of $\query$ in $\proofScheme$ is a mapping $\match$ from $\vars{\query}$ to $\terms{\proofScheme}$ such that:
 \begin{itemize}
  \item if $\atom$ is a standard atom in $\query$, then $\match(\atom) \in V$;
  \item if $\atom= \lang(x,y)$ is a path atom in $\query$ with associated NFA $\automaton$,
  then there exists a pair $((i_x,\state_x,i_y,\state_y),(\atom_x,\atom_y))$ in the saturation of $\proofScheme$
  such that $\match(x)=\atom_x[i_x]$, $\match(y)=\atom_y[i_y]$  and $\state_x$ (resp.\ $\state_y$) is the initial (resp.~a final) state of~$\automaton$. 
 \end{itemize}
\end{definition}

Note that, in the second point of the previous definition, $x$ and $y$ are not necessarily variables and we silently extend $\match$ to constants by setting $\match(c) = c$ for any constant $c$. The problem of deciding if there is a match of a given query $\query$ in a given proof scheme $\proofScheme$ is NP-complete. Indeed, checking whether a candidate mapping is indeed a match of $\query$ in $\proofScheme$ can be done in polynomial time, since the saturation of a proof scheme can be computed in polynomial time; 
NP-hardness can be shown by reducing the homomorphism problem (given an instance $\instance$ and a CQ $\query$, is there a homomorphism from $\query$ to $\instance$?), with $\instance$ being translated into a trivial proof scheme where $V = \instance$, $E= \emptyset$ and $ \pstransitions = \emptyset$. 

\medskip
As already mentioned, not all proof schemes actually encode a part of the chase. 
More precisely, a proof scheme is said to be \emph{valid} if it can be embedded into the chase by a mapping called a `witnessing embedding'. Then, given a valid proof scheme $\proofScheme$ and a query $q$, the composition of a match of $q$ in $\proofScheme$ and a witnessing embedding of $\proofScheme$ yields a match of $q$ in the chase. We will also show that, whenever there is a match of $q$ in the chase, there is a match of $q$ in some valid proof scheme. 

\begin{definition}[Valid proof scheme and witnessing embedding]
\label{def-validity}
A proof scheme $\proofScheme=(\psforest, \pstransitions)$ with $\psforest = (V,E)$ 
is \emph{valid} w.r.t.\ $(\instance,\ruleset,\query)$ if there is a mapping %
 $\witness$ from $\vars{\proofScheme}$ to $\terms{\chase{\instance}{\ruleset}}$ 
such that:
\begin{enumerate}
\item\label{item-valid-root} for each root atom $\atom \in V$, $\witness(\atom) = \atom$ and $\atom \in \instance$;
\item\label{item-valid-standard} for each atom $\atom \in V$, $\witness(\atom) \in \chase{\instance}{\ruleset}$;%
\item\label{item-valid-ancestor} for each edge $(\atom_1,\atom_2) \in E$, there is an $\ruleset$-derivation from $\{\witness(\atom_1)\}$ to (an extended instance that contains) $\witness(\atom_2)$.
\item\label{item-valid-transition} for each $(\tau,(\atom_1,\atom_2)) \in \pstransitions$: 
$\tau \in \transition{\witness(\atom_1)}{\witness(\atom_2)}{\query}{\instance}{\ruleset}$.
\end{enumerate}
We say $\witness$ is a \emph{witnessing embedding} of $\proofScheme$.
\end{definition}

Note that the saturation of a proof scheme is generally not a proof scheme, as the added transitions may label arbitrary pairs of atoms. However, and this is  a point worth noting even if it is immediate, the saturation of a valid proof scheme also fulfills all the validity conditions from Definition \ref{def-validity}. 


\medskip

\begin{example}[Running Example---continued]
\label{ex-running-proof-scheme} Let  $\ruleset$ be the set of linear rules from Example \ref{ex-running-general}:

\begin{quote}
$(\erule_1)\quad \fun{isFriendOf}(x,y) \rightarrow  \fun{isFriendOf}(y,x)$
\\
$(\erule_2)\quad \fun{isFriendOf}(x,y) \rightarrow  \fun{follows}(x,y)$ 
\\
$(\erule_3)\quad \fun{follows}(x,y) \rightarrow  \exists m~\fun{message}(m,x,y)$ 
\end{quote}
Consider again the instance $\instance = \{\fun{follows(B,A)}, \fun{isFriendOf(C,A)}, \fun{isFriendOf(C,B)}\}$ and the following CRPQ:
 $$q_6() = \exists y,m.~\fun{follows.follows}^{*}(A,y) \land message(m,y,A)$$
  We assign to the path predicate $\fun{follows.follows}^{*}$ the NFA representation $\automaton$ with states $\{s_0, s_f\}$ and two transitions:  $(s_0, \fun{follows}, s_f)$, $(s_f, \fun{follows}, s_f)$. 
Recall that query $q_6$ can be matched to the following part of $\chase{\instance}{\ruleset}$: $\{ \fun{follows}(A,C), \fun{follows}(C,B), \fun{message}(m_0,B,A) \}$.

A proof scheme $\proofScheme=(\psforest, \pstransitions)$ for $q_6$ is pictured in Figure~\ref{fig-proof-scheme-running}. The forest $\psforest$ contains a single tree, having two nodes. The only appearing null is $m_0$ and it occurs in a single node, which trivially forms a connected subgraph of $\psforest$. The set of transitions 
$\pstransitions$ contains a single transition $(\tau, (\atom_1,\atom_1))$, where $\tau = (2,s_0,1,s_f)$.
Proof scheme  $\proofScheme$ is valid w.r.t. $(\instance,\ruleset,q_6)$, as attested by a witnessing embedding $\witness = \{ m_0 \mapsto m_0\}$.  Note that the transition holds due to the atoms $\fun{follows}(A,C)$ and $ \fun{follows}(C,B)$ in $\instance$. Query $q_6$ has a match in $\proofScheme$ given by $\match = \{y \mapsto B, m \mapsto m_0 \}$.  
\end{example}

\begin{figure}
\begin{center}
 \begin{tikzpicture}[line cap=round,line join=round,>=triangle 45,x=1.0cm,y=1.0cm]
\draw (0,0) node(root1) {$\fun{follows(B,A)}$};
\draw (0,-1) node (child1) {$\fun{message(m_0,B,A)}$};
\draw (-2,0) node(rootname) {$\atom_1$};
\draw (-2,-1) node(childname) {$\atom_2$};
\draw (-4,0) node(forestname) {$\psforest$:};
\draw[->] (root1) -- (child1);
\end{tikzpicture}
\end{center}
\caption{{Valid proof scheme $(\psforest, \pstransitions)$, with $\pstransitions = \{ ((2,s_0,1,s_f),(\atom_1,\atom_1)) \}$ (Example~\ref{ex-running-proof-scheme})}}
\label{fig-proof-scheme-running}
\end{figure}

\color{black}


The next proposition states that CRPQ answering can be recast as deciding whether there is a match of the query in a `small' valid proof scheme.

\begin{proposition}\label{prop-match-valid}
 $\instance, \ruleset \models \query$ iff there exists a match of $\query$ in some valid proof scheme of polynomial size in $\query$. 
\end{proposition}

\begin{proof}
 ($\Leftarrow$) Let $\proofScheme$ be a valid proof scheme and $\witness$ be a witnessing embedding of $\proofScheme$ in $\chase{\instance}{\ruleset}$. Let $\match$ be a match of $\query$ in $\proofScheme$.
We claim that $\witness\circ \match$ is a match of $\query$ in $\chase{\instance}{\ruleset}$, which proves that $\instance, \ruleset \models \query$.
Indeed, take any atom $\atom$ of $\query$. If $\atom$ is a standard atom, $\match(\atom) \in V$, hence $\witness \circ \match(\atom) \in \chase{\instance}{\ruleset}$.
If $\atom$ is a path atom, let $\atom = \Lambda(x,y)$, then there is a pair $((i_x,s_x,i_y,s_y), (\atom_x,\atom_y))$  in the saturation of~$\proofScheme$, such that $\match(x)=\atom_x[i_x]$ and $\match(y)=\atom_y[i_y]$, with $s_x$ and $s_y$ initial and final states in the automaton associated with  $\Lambda$.  
As $\proofScheme$ is valid, its saturation also fulfills the validity conditions. Hence, there is a path $\path$ from $\witness(\match(x))$ to $\witness(\match(y))$ in $\chase{\instance}{\ruleset}$ such that $\pathLabel{\path}$ belongs to $\mathcal L(\Lambda)$. We conclude that $\witness\circ\match$ is a match of $\query$ in $\chase{\instance}{\ruleset}$.

($\Rightarrow$)
 This part of the proof is addressed in Section \ref{sec-crpq-upper-proof}.
\end{proof}

\subsection{Proof of Proposition\ \ref{prop-match-valid}}
\label{sec-crpq-upper-proof}

We now prove that if $\instance, \ruleset \models \query$, then one can build a valid proof scheme $\proofScheme$ of polynomial size in $\query$, such that there is a match of $\query$ in $\proofScheme$. 
For that, we rely on a classical tool, namely the \emph{chase graph} associated with a derivation, which keeps track of how atoms are generated in this derivation. 
Nodes are labeled by the atoms occurring in the derivation and there is an edge $(\node, \node')$ if a trigger in the derivation is applied on the atom that labels $\node$ and leads to generate the atom that labels $\node'$.  Note that we only keep track of rule applications that produce new atoms and that the derivation may not be fair (yet). 

\medskip

\begin{definition}[Chase graph associated with a derivation]
Let $\instance$ be an instance and $\ruleset$ be a set of linear rules. 
 The \emph{chase graph} associated with a (possibly infinite) $\ruleset$-derivation $D$ from $\instance$, denoted by $\chasegraph_D$,  is a (possibly infinite) node-labeled directed  forest built as follows (we use $\cglabel$ for the labeling function):
\begin{itemize}  
\item its set of nodes is in bijection with $\atoms{D}$ via $\cglabel$;
\item for each trigger $(\erule_i, \match_i)$ in $D$ that \emph{directly generates} an atom $\atom_i$, there is an edge $(\node,\node')$ where 
$\cglabel(\node) = \match_i(\body{\erule_i})$  and $\cglabel(\node') = \atom_i$.  
\end{itemize} 
The set of atoms that label the nodes of $\chasegraph_D$ is denoted by $\atoms{\chasegraph_D}$.
\end{definition}

Note that $\chasegraph_D$ is indeed a forest because each atom outside $\instance$ is directly generated exactly once in $D$ and there is a single node labeled by this atom. Also note that $\atoms{\chasegraph_D} = \atoms{D}$. 

The next example will be used to illustrate the technical notions of this section, beginning with that of a chase graph. 

\begin{example}
\label{ex-running}
Let $\instance = \{h(a_1,a_2),m(a_2)\}$ and $\ruleset$ containing the following rules:

\begin{tabular}{ll}
 $R_1: h(x,y) \rightarrow \exists z~b(x,y,z)$ &\qquad  \\
 $R_2: b(x,y,z) \rightarrow \exists t\ b'(x,z,t)$ &\qquad  $R_5: b(x,y,z) \rightarrow \exists t~b''(y,z,t)$\\
 $R_3: b'(x,y,z) \rightarrow s(x,z)$ &\qquad 
 $R_6: b''(x,y,z) \rightarrow r(y,z)$ \\
 $R_4: b'(x,y,z) \rightarrow r(z,y)$ &\qquad $R_7: b''(x,y,z) \rightarrow s(x,z)$
\end{tabular}

The chase graph associated with a (fair) $\ruleset$-derivation $D$ from $\instance$ to $\instance'$ is pictured in Figure\ \ref{fig-chase-graph-running}, left. 
Let $\query= \exists x_1\exists x_2\exists x_3\exists x_4\ h(x_1,x_2) \wedge m(x_2) \wedge s^*(x_1,x_3) \wedge r^*(x_3,x_4) \wedge s^*(x_2,x_4)$.
There is a match $\match$ of $\query$ in $\instance'$  defined by: $\{x_1 \rightarrow a_1, x_2 \rightarrow a_2, x_3 \rightarrow t_0, x_4 \rightarrow t_1\}$. 

\end{example}

\begin{figure}
\begin{center}
 \begin{tikzpicture}[line cap=round,line join=round,>=triangle 45,x=1.0cm,y=1.0cm]

\draw (0,0) node(root1) {${h(a_1,a_2)}$};
\draw (2.5,0) node(root2) {${m(a_2)}$};
\draw (0,-1) node (child1) {$b(a_1,a_2,z_0)$};
\draw[->] (root1) -- (child1);

\draw (1.5,-2.5) node (b2a2zy) {$b''(a_2,z_0,t_1)$};
\draw[->] (child1) -- (b2a2zy);
\draw (-1.5,-2.5) node (b1a1zx) {$b'(a_1,z_0,t_0)$};
\draw[->] (child1) -- (b1a1zx);
\draw (-2.2,-4) node (sa1x) {${s(a_1,t_0)}$};
\draw (-0.8,-4) node (rxz) {${r(t_0,z_0)}$};
\draw[->] (b1a1zx) -- (sa1x);
\draw[->] (b1a1zx) -- (rxz);
\draw (2.2,-4) node (sa2y) {${s(a_2,t_1)}$};
\draw (0.8,-4) node (rzy) {${r(z_0,t_1)}$};
\draw[->] (b2a2zy) -- (sa2y);
\draw[->] (b2a2zy) -- (rzy);

\draw (8,0) node(nroot1) {${h(a_1,a_2)}$};
\draw (10.5,0) node(nroot2) {${m(a_2)}$};
\draw (8,-1) node (nchild1) {$b(a_1,a_2,z_0)$};
\draw[->] (nroot1) -- (nchild1);

\draw (9.5,-2.5) node (nrzy) {$r(z_0,t_1)$};
\draw[->] (nchild1) -- (nrzy);
\draw (6.5,-2.5) node (nsa1x) {$s(a_1,t_0)$};
\draw[->] (nchild1) -- (nsa1x);

\end{tikzpicture}
\end{center}
\caption{Chase graph $\chasegraph_D$ on the left and a backbone of $\match$ on the right (Example \ref{ex-running})}
\label{fig-chase-graph-running}
\end{figure}

\medskip
To simplify writing, when a term $\term$ occurs in an atom $\atom = \cglabel(\node)$, we will say that  $\atom$ contains $\term$, and by extension that $\node$ \emph{contains} $\term$. This is extended to a set of nodes and a subgraph. 
 A chase graph $\chasegraph_D$ can also be seen as a tree decomposition of  the set of atoms $atoms(\chasegraph_D)$, in the following sense: 
$\chasegraph_D$ is a forest structure such that \emph{(i)} nodes are labeled by atoms from $atoms(\chasegraph_D)$, \emph{(ii)} any atom of $atoms(\chasegraph_D)$ is the label of some node of $\chasegraph_D$ and \emph{(iii)} for any null $\chasenull$, the 
set of nodes that contain $\chasenull$ forms a connected (rooted) subgraph of $\chasegraph_D$; for any constant $c$, the set of nodes \emph{in a tree that contains} $c$ form a connected (rooted) subgraph of $\chasegraph_D$. 
Next, the item \emph{(iii)} is called the \emph{connectivity property} of the chase graph. 

\medskip
The next proposition relates paths in a chase graph, which correspond to (undirected) paths of atoms, and paths in the chase, which are (directed) paths of terms.
Informally, it says that, given any atoms $\atom_x$ and $\atom_y$ in the same tree of the chase graph, any path of terms in the chase, from a term $x$ occurring in  $\atom_x$ to a term $y$ occurring in $\atom_y$, goes through any atom along the unique shortest path from $\atom_x$ to $\atom_y$ in the chase graph.

\begin{proposition}[Path property]
Let $\chasegraph_D$ be a chase graph associated with a derivation $D$. Let $\node_x$ and $\node_y$ be nodes in the same tree of $\chasegraph_D$, such that  $\node_x$ contains term $x$ and $\node_y$ contains term $y$.  Let $\node$ be any node on the unique shortest (undirected) path from $\node_x$ to $\node_y$ in $\chasegraph_D$. Then, any path (of terms) $p$ from $x$ to $y$ in $atoms(\chasegraph_D)$ can be written $p = p_1.p_2$, where $p_1$ and $p_2$ are paths (of terms) in $\chase{\instance}{\ruleset}$, respectively from $x$ to $z$ and from $z$ to $y$,  such that $\node$ contains $z$. 
\end{proposition}

\begin{proof}
 Let $p = t_0 r_1 t_1 \ldots r_n t_n$ such that $t_0 = x$ and $t_n = y$. 
 If $\node$ contains $x$ (resp. $y$), the statement trivially holds: we take $p_1$ empty and $p_2 = p$ (resp.  $p_1 = p$ and $p_2$ empty). 
Otherwise (we know that $\node \neq \node_x$ and $\node \neq \node_y$), removing $\node$ from $\chasegraph_D$ splits the graph into several connected components, including one, that we call $C_x$ (resp. $C_y$) containing $\node_x$ (resp. $\node_y$). Let us consider the smallest $i$ in $p$ such that $t_i$ occurs in a connected component $C \neq C_x$. There must be such an $i$, since $t_n = y$ occurs in $C_y$.  We show that $t_i$ occurs in $\node$, i.e., $t_i$ is a term $z$ as required by the property. Since $r_i(t_{i-1},t_i) \in atoms(\chasegraph_D)$, there is a node $\node'$ of $\chasegraph_D$  labeled by $r_i(t_{i-1},t_i)$. 
By hypothesis on $t_i$, $t_{i-1}$ does not appear outside $C_x \cup \{\node\}$, hence $\node'$ belongs to $C_x \cup \{\node\}$. Either $\node' = \node$ (and $t_i$ appears in $\node$), or $\node'$ belongs to $C_x$: in this case, $t_i$ appears both in $C_x$ and in $C \neq C_x$, hence, by the connectivity property on $t_i$, $t_i$ must also appear in $\node$, which concludes the proof.
\end{proof}

{\sc Example \ref{ex-running} (continued).} See also Figure \ref{fig-chase-graph-running} (left). Let  
$p = a_1 ~s ~t_0 ~r ~z_0 ~r ~t_1 ~s^{-} ~a_2$ be a path of terms, which occurs in the chase. We can see that $p$ has no direct connection with the structure of the chase graph pictured in Figure \ref{fig-chase-graph-running} (left). Now, take, e.g., the nodes $\node_{a_1}$ labeled by $b'(a_1,z_0,t_0)$ and $\node_{a_2}$ labeled by $s(a_2,t_1)$, and consider the shortest path of atoms from $\node_{a_1}$ to $\node_{a_2}$ in the chase graph, which goes through $b(a_1,a_2,z_0)$, and $b''(a_2,z_0, t_1)$. W.r.t. the atom $b(a_1,a_2,z_0)$,  $p$ can be written $p_1.p_2$ with $p_1 = a_1 ~s ~t_0 ~r ~z_0$ and $p_2 = z_0 ~r ~t_1 ~s^{-} ~a_2$.
W.r.t. the atom $b''(a_2,z_0, t_1)$,  $p$ can be decomposed in the same way, or also $p''_1.p''_2$ with $p''_1 = a_1 ~s ~t_0 ~r ~z_0 ~r ~t_1$ and $p''_2 = t_1 ~s^{-} ~a_2$.

\medskip
When $\instance, \ruleset \models \query$, there is a finite derivation $D$ from $\instance$ to an extended instance $\instance_n$ with a match $\match$ of $\query$ in $I_n$. We will now build a valid proof scheme from $D$ and  $\match$. 
To do so, guided by $\match$, we first select some atoms in $I_n$: \emph{(i)}  the images of the standard atoms of $\query$, \emph{(ii)} for each variable of $\query$ mapped to a null, an atom containing its image, and \emph{(iii)} some atoms that allow to check for the existence of required paths.
Note that some terms of $\query$ may appear only in path atoms, which justifies \emph{(ii)}. Each of these atoms being the label of a unique node in $\chasegraph_D$, the selected set of atoms is naturally structured in a forest induced by $\chasegraph_D$. From this `backbone', we then build the intended proof scheme
 by adding all the transitions that appear between eligible pairs of atoms. 

Next, we use the notation $\node_1 \prec \node_2$ to denote that $\node_1$ is a \emph{strict ancestor} of $\node_2$ in the considered chase graph. Also, for $\node$ and $\node'$ in the same tree of a chase graph, $\glb{\node,\node'}$ denotes their \emph{greatest lower bound} ({i.e.}, their deepest common ancestor) in this graph.

\begin{definition}[Backbone of a match]\label{def-backbone}
Let $D$ be an $\ruleset$-derivation from $\instance$ to $\instance_n$, such that $\instance_n \models \query$. 
Let $\match$ be a match of $\query$ in $\instance_n$. A \emph{backbone} of $\match$ is a directed forest $(V,E)$ defined as follows:

\begin{itemize}
 \item $V$ is a subset of $\instance_n$ restricted to the following atoms:
 \begin{itemize}
 \item for each standard atom $\atom \in \query$, the atom $\match(\atom)$ ;
 \item for each variable $x$ of $\query$ such that $\match(x)$ is a null, some $\atom_x \in \instance_n$ that contains $\match(x)$;
 \item for each atom $\atom\in V$, the atom $\cglabel(\node_r)$, where $\node_r$ is the root of the tree of $\chasegraph_D$ in which the node $\node$ with $\cglabel(\node) = \atom$ occurs;
 \item for each pair $(\atom_1,\atom_2) \in V \times V$ such that $\atom_1 = \cglabel(\node_1), \atom_2 = \ell(\node_2)$, and $\node_1$, $\node_2$ belong to the same tree of $\chasegraph_D$, the atom $\cglabel(\node)$, where $\node= \glb{\node_1,\node_2}$ in $\chasegraph_D$;
 \end{itemize}
 \item $E \subseteq V \times V$ is the set of edges induced by $\chasegraph_D$, i.e., it contains $(\atom_1,\atom_2)$ if and only if $\atom_1 = \cglabel(\node_1)$, $\atom_2 = \cglabel(\node_2)$, $\node_1 \prec \node_2$ in $\chasegraph_D$,  and there is no node $\node_3$ with $\node_1 \prec \node_3 \prec \node_2$ and $\cglabel(\node_3) \in V$.
\end{itemize}

\end{definition}

{\sc Example \ref{ex-running} (continued).}
Consider again the match $\match$ of $\query$ in the instance $\instance'$  (Figure \ref{fig-chase-graph-running}, left)
defined by: $\{x_1 \rightarrow a_1, x_2 \rightarrow a_2, x_3 \rightarrow t_0, x_4 \rightarrow t_1\}$. A backbone of $\match$ is shown on the right of Figure\ \ref{fig-chase-graph-running}: $h(a_1,a_2)$ and $m(a_2)$ belong to the backbone because they each are the image of a 
standard atom of $\query$; $s(a_1,t_0)$ and $r(z,t_1)$ are selected because they contain the nulls $\match(x_3) = t_0$ and $\match(x_4)= t_1$ (note that $x_3$ and $x_4$ only occur in path atoms of $\query$); $b(a_1,a_2,z_0)$ is selected because it is the deepest common ancestor of $s(a_1,t_0)$ and $r(z_0,t_1)$. Then, to define the associated proof scheme induced by $\match$, we will enrich that backbone by all the possible transitions (which is not shown on the figure).

\begin{definition}[Proof scheme induced by a match]
 \label{def-induced-proof-scheme}
Let $\match$ be a match of $\query$ in some $\instance_n$ obtained by an $\ruleset$-derivation from $\instance$. A proof scheme $\proofScheme_\match$ induced by $\match$ is $(\mathcal{F} = (V,E),\pstransitions)$, where:

\begin{itemize}
 \item $\mathcal{F}$ is a backbone of $\match$; 
 \item $\mathcal{T}$ is the set of all pairs $(t,(\atom_1,\atom_2))$ with $(\atom_1,\atom_2) \in V \times V$ and $t \in \transitionAtomSet{\atom_1}{\atom_2}{\query}{\instance_n}$, that are legal in a proof scheme, i.e., $\atom_1 = \atom_2$, $(\atom_1,\atom_2) \in E$, $(\atom_2,\atom_1) \in E$, or both $\atom_1$ and $\atom_2$ are roots.
\end{itemize}
\end{definition}

{\sc Example \ref{ex-running} (continued).} 
 For instance, let us assume an automaton with a single state $\state_r$ to encode $r^*$. The transition $((1,\state_r,2,\state_r),(r(z_0,t_1),r(z_0,t_1)))$ is in $\proofScheme_\match$ because \emph{(i)} $r(z_0,t_1) = r(z_0,t_1)$ and \emph{(ii)} the atom $r(z_0,t_1)$ belongs to $\chase{\instance}{\ruleset}$. It also holds that $((2,\state_r,3,\state_r),(s(a_1,t_0),b(a_1,a_2,z_0))) \in \pstransitions$ is in $\proofScheme_\match$, because \emph{(i)} $s(a_1,t_0)$ is a child of $b(a_1,a_2,z_0)$ and \emph{(ii)} the atom $r(t_0,z_0)$ belongs to $\chase{\instance}{\ruleset}$. 
 
 The transition $((1,\state_r,2,\state_r),(r(t_0,z_0),r(z_0,t_1)))$ does not belong to $\proofScheme_\match$, because $r(t_0,z_0)$ is neither equal to nor a parent or child of $r(z_0,t_1)$, but it belongs to the saturation of $\proofScheme_\match$ because of the above two transitions.
  \medskip
 Let us note that the proof scheme induced by a match of $\query$ has a size polynomial in $\query$. We now prove that it is indeed a proof scheme and that it is valid.

\begin{proposition}
 Let $\match$ be a match of $\query$ in $\instance_n$ obtained by an $\ruleset$-derivation $D$ from $I$. 
 Any proof scheme induced by $\match$ is valid. 
\end{proposition}

\begin{proof}
Let $\proofScheme_\match = (\psforest,\pstransitions)$ be a proof scheme induced by $\match$. We first prove that $\proofScheme_\match$ is indeed a proof scheme. Numbers refer to items of Definition~\ref{def-proof-scheme}. (\ref{item-ps-df}): since the chase graph $\chasegraph_D$ is a directed forest, $\mathcal{F}$ is a directed forest; (\ref{item-ps-set}): $V$ is by definition a set of atoms; (\ref{item-ps-connected}): let $\atom_1$ and $\atom_2$ be two atoms of $\mathcal{F}$ containing $\chasenull$. 
By the connectivity property of $\chasegraph_D$, $\atom_1$ and $\atom_2$ are labels of connected nodes in $\chasegraph_D$. 
Let $\node_1$ and $\node_2$ be these nodes and let $\node_3 = \glb{\node_1,\node_2}$: by construction, $\cglabel(\node_3)$ belongs to $\mathcal{F}$ and is an  ancestor of  $\node_1$ and $\node_2$ in $\mathcal{F}$. Hence, $\atom_1$ and $\atom_2$ are connected in $\mathcal{F}$.
  (\ref{item-ps-transition}) and (\ref{item-ps-local}):  by construction, all the transitions in the induced proof scheme are of the correct syntactic shape. 
 
  We now show that the identity is a witnessing embedding of $\proofScheme_\match$ in $\chasegraph_D$, proving its validity. Numbers refer to items in Definition~\ref{def-validity}. (\ref{item-valid-root}) By construction of $V$, all the root atoms are labels of roots of $\chasegraph_D$, hence they belong to $\instance$.
  (\ref{item-valid-standard}) By construction, $V \subseteq \instance_n$, and $\witness$ is the identity, so $\witness(\atom)$ belongs to $\instance_n$ for all $\atom \in V$.
  (\ref{item-valid-ancestor}) By construction, for each edge $(\atom_1,\atom_2)$ in $E$, $\cglabel^{-1}(\atom_1) \prec \cglabel^{-1}(\atom_2)$ in $\chasegraph_D$,
  hence there is a derivation as required.
  (\ref{item-valid-transition}) By construction, each $(t,(\atom_1,\atom_2)) \in \mathcal{T}$ belongs to $\transitionAtomSet{\atom_1}{\atom_2}{\query}{\instance_n}$, hence to $\transitionAtomSet{\witness(\atom_1)}{\witness(\atom_2)}{\query}{\instance_n}$ because $\witness$ is the identity.
 We conclude that $\proofScheme$ is a valid proof scheme w.r.t $(\instance,\ruleset,\query$).
\end{proof}

We now show that there is indeed a match of $\query$ in the (valid) proof scheme induced by a match of $q$ in the chase.

\begin{proposition}
 \label{prop-match-induced-proof-scheme}
Let $\proofScheme_\match$ be a proof scheme induced by a match $\match$ of $\query$ in $\instance_n$, where $\instance_n$ is the result of some derivation $D$ from $\instance$. There exists a match of $\query$ in $\proofScheme_\match$. 
\end{proposition}

\begin{proof}
We show that $\match$ is a match of $\query$ in $\proofScheme_\match$ = $(\mathcal{F} = (V,E),\pstransitions)$. 
\begin{itemize}
 \item By definition of $\proofScheme_\match$, if $\atom \in \query$ is a standard atom, then $\match(\atom) \in V$.
 \item Let $\atom = \Lambda(x,y)$ be a path atom. As $\match$ is a match of $\query$, there is a path in $\instance_n$ from $\match(x)$ to $\match(y)$, whose label  is recognized by the automaton associated with $\Lambda$. 
By definition of $V$ (Definition \ref{def-backbone}), there exist $\atom_x$ and $\atom_y$ in $V$ containing $\match_x$ and $\match_y$. Let $\node_x$ and $\node_y$ s.t. $\cglabel(\node_x) = \atom_x$ and $\cglabel(\node_y) = \atom_y$. We show by induction that whenever there is a path from $\atom_x[i]$ to $\atom_y[j]$ in $\instance_n$ that moves the automaton from $s_x$ to $s_y$, then $((i_x,s_x,i_y,s_y),(\atom_x,\atom_y))$ is in the saturation of $\proofScheme$. We study the  possible cases regarding the relationship between $\atom_x$ and $\atom_y$ in $\proofScheme_\match$:

 \begin{itemize}
  \item if $\atom_x = \atom_y$, or $\atom_x$ is the parent of $\atom_y$ (or vice-versa):  the property holds 
 by definition of the proof scheme induced by a match (Definition \ref{def-induced-proof-scheme});
  \item if $\atom_x$ is a strict ancestor of $\atom_y$, but not a parent: we show the result by induction on the length of the path from $\atom_x$ to $\atom_y$. The base case has been treated in the previous item. Let $\atom$ be the child of $\atom_x$ that is an ancestor of $\atom_y$. Let $\node$ be the node of $\chasegraph_D$ s.t. $\cglabel(\node) = \atom$. By the path property of $\chasegraph_D$, any path that goes from $x$ to $y$ must go through some $z \in \terms{\atom}$ occurring at some position $i_z$. Let $p = p_1.p_2$, such that $p_1$ goes from $x$ to $z$ and $p_2$ goes from $z$ to $y$. Let $s_z$ be the state of  $\automaton$ after reading $p_1$ starting from $s_x$. By induction assumption,  $((i_x,s_x,i_z,s_z),(\atom_x,\atom))$ and  $((i_z,s_z,i_y,s_y),(\atom,\atom_y))$ belong to the saturation of $\proofScheme_\match$. Hence, $((i_x,s_x,i_y,s_y),(\atom_x,\atom_y))$ is in the saturation of $\proofScheme$;
  \item $\atom_x$ and $\atom_y$ are not in an ancestor relationship but are in the same tree. By the path property of $\chasegraph_D$, any path from $\match(x)$ to $\match(y)$ must go through a term $z$ of $\glb{\node_x,\node_y}$, whose label we denote by $\atom$. Let $p=p_1.p_2$ such that $p_1$ is a path from $x$ to $z$ and $p_2$ is a path from $z$ to $y$. Let $s_z$ be the state in which $\automaton$ is after reading $p_1$ starting from $s_x$. Assuming $z = \atom[i_z]$, by induction assumption, $((i_x,s_x,i_z,s_z),(\atom_x,\atom))$ and $((i_z,s_z,i_y,s_y),(\atom,\atom_y))$ belong to the saturation of $\proofScheme_\match$. Hence, 
    $((i_x,s_x,i_y,s_y),(\atom_x,\atom_y))$ 
  is in the saturation of $\proofScheme_\match$;
  \item $\atom_x$ and $\atom_y$ are not in the same tree: let $\atom_{R_x}$  (resp. $\atom_{R_y}$) be the root of the tree containing $\atom_x$ (resp. $\atom_y$). By the connectivity property of $\chasegraph_D$, any path from $x$ to $y$ must go through $\atom_{R_x}[i_{t_x}]$ for some $i_{t_x}$ (resp. $\atom_{R_y}[i_{t_y}]$ for some $i_{t_y}$). By the preceding cases, it holds that $((i_x,s_x,i_{t_x},s_{t_x}),(\atom_x,\atom_{R_x}))$ and $((i_{t_y},s_{t_y},i_{y},s_{y}),(\atom_{R_y},\atom_y))$ belong to the saturation of $\proofScheme_\match$. By definition of $\proofScheme_\match$, because $\atom_{R_x}$ and $\atom_{R_y}$ are roots of $\proofScheme_\match$, it holds that  $((i_{t_x},s_{t_x},i_{t_y},s_{t_y}),(\atom_{R_x},\atom_{R_y}))$ belongs to $\proofScheme_\match$. Hence, by definition of the saturation of $\proofScheme_\match$, 
    $((i_x,s_x,i_y,s_y),(\atom_x,\atom_y))$ 
  also belongs to the saturation of $\proofScheme_\match$, which concludes the proof.\qedhere
 \end{itemize}
\end{itemize}
\end{proof}

\subsection{Complexity Analysis}
\label{sec-crpq-complexity}

By Proposition \ref{prop-match-valid}, we know that $\query$ is entailed by $\instance$ and $\ruleset$ if and only if there exists a valid proof scheme $\proofScheme$ of polynomial size in $\query$ such that there is a match of $\query$ in $\proofScheme$. We now investigate how to check the validity of a proof scheme. 
To do so, we will consider a variant of the previous chase graph, in which we keep track of all possible ways of producing atoms. 
We call it the chase graph associated with $(\instance,\ruleset)$, because its set of atoms 
is homomorphically equivalent to  $\chase{\instance}{\ruleset}$. 

\begin{definition}[Chase graph associated with $(\instance,\ruleset)$]
Let $\instance$ be an instance and $\ruleset$ be a set of linear rules.
The \emph{chase graph} associated with $(\instance,\ruleset)$, denoted by $\chasegraph(\instance, \ruleset)$,
is the node-labeled directed graph defined as follows (we use $\cglabel$ for the labeling function):
\begin{itemize}
\item The restriction of $\cglabel$ to the root nodes of $\chasegraph(\instance, \ruleset)$ defines a bijection between the root nodes and $I$.
\item For every node $\node$, the children of $\node$ are in bijection with the finite set of triggers $(\erule_i, \match_i)$ on
$\cglabel(\node)$, and the labels of the children are precisely the atoms  $\atom_i$ resulting from the application of these triggers 
to $\cglabel(\node)$ (using globally unique fresh nulls). 
\item Only the nodes and edges required by the preceding two items are present in $\chasegraph(\instance, \ruleset)$. 
\end{itemize}
\end{definition}

Importantly, $\chasegraph(\instance, \ruleset)$ still fulfils the connectivity property (for the same reasons as the chase graph associated with derivations). 
The following lemma  
shows how transitions between an atom in the chase graph and an ancestor atom can be computed using the transitions relating its parent atom to the same ancestor, together with the transitions that can be obtained by considering the atom in isolation.

\begin{lemma}\label{decompose-trans}
Let $\node_1$, $\node_2$ and $\node_3$ be nodes in $\chasegraph(\instance, \ruleset)$ such that $\node_1$ is a (strict) ancestor of $\node_2$, and $\node_3$ is a successor of $\node_2$. Let $\atom_i = \cglabel(\node_i)$ for $1 \leq i \leq 3$. Then:
\sloppypar{
\begin{enumerate}
\item $(i_1, s_1, i_3, s_3) \in \transition{\atom_1}{\atom_3}{\query}{\instance}{\ruleset}$ iff there exist $i_2,s_2,i'_2$  
such that $(i_1, s_1, i_2, s_2) \in \transition{\atom_1}{\atom_2}{\query}{\instance}{\ruleset}$, $(i_2', s_2, i_3, s_3) \in \transition{\atom_3}{\atom_3}{\query}{\{\atom_3\}}{\ruleset}$ and
$\atom_2[i_2] = \atom_3[i_2'];$
\item $(i_1, s_1, i_3, s_3) \in \transition{\atom_3}{\atom_1}{\query}{\instance}{\ruleset}$ iff there exist $i_2,s_2,i'_2$  
such that
$(i_1, s_1, i_2', s_2) \in \transition{\atom_3}{\atom_3}{\query}{\{\atom_3\}}{\ruleset}$, $(i_2, s_2, i_3, s_3) \in \transition{\atom_2}{\atom_1}{\query}{\instance}{\ruleset}$ and $\atom_2[i_2] = \atom_3[i_2'];$
\item $(i_1, s_1, i_4, s_4) \in \transition{\atom_3}{\atom_3}{\query}{\instance}{\ruleset}$ iff there exist $i_2,s_2,i'_2,i_3,s_3,i'_3$ 
such that
$(i_1, s_1, i_2', s_2) \in \transition{\atom_3}{\atom_3}{\query}{\{\atom_3\}}{\ruleset}$,
$(i_2, s_2, i_3, s_3) \in \transition{\atom_2}{\atom_2}{\query}{\instance}{\ruleset}$,
$(i_3', s_3, i_4, s_4) \in \transition{\atom_3}{\atom_3}{\query}{\{\atom_3\}}{\ruleset}$
 and $\atom_3[i_2'] = \atom_2[i_2]$ and $\atom_3[i_3'] = \atom_2[i_3].$
\end{enumerate}}
\end{lemma}

\begin{proof}
The `if' direction of the equivalences are trivial, so we concentrate on proving the `only if' directions.

For the first statement, suppose that $(i_1, s_1, i_3, s_3) \in \transition{\atom_1}{\atom_3}{\query}{\instance}{\ruleset}$.
It follows that there is a path $p$ in $\chase{\instance}{\ruleset}$ that begins in $t_1=\atom_1[i_1]$, ends in $t_3=\atom_3[i_3]$,
and is such that $\lambda(p)$ takes an automaton $\automaton$ in $q$ from state $s_1$ to state $s_3$.
Let $p_2$ be the longest suffix of $p$ that starts in $t_2 \in \terms{\atom_3}$, and whose label labels a path from $t_2$ to $t_3$ in $\chase{\{\atom_3\}}{\ruleset}$.
Note that such a path $p_2$ is guaranteed to exist, as we can always choose the empty path from $t_3$ to $t_3$.
We let $p_1$ be such that $p=p_1.p_2$, and let $s_2$ be a state such $\lambda(p_1)$ takes $\automaton$
from state $s_1$ to $s_2$ and $\lambda(p_2)$ takes $\automaton$ from $s_2$ to $s_3$.
We consider two cases:
\begin{itemize}
\item Case 1: $p=p_2$. Then $t_2 = t_1$, and $t_1$ must belong to $\atom_3$. By the connectivity property of $\chasegraph(\instance, \ruleset)$, $t_1$ must also belong to $\atom_2$.
There are thus indices $i_2$ and $i_2'$
such that $\atom_1[i_1]= \atom_2[i_2] = \atom_3[i_2']$.
We thus have $(i_1, s_1, i_2, s_1) \in \transition{\atom_1}{\atom_2}{\query}{\instance}{\ruleset}$ and $(i_2', s_1, i_3, s_3) \in \transition{\atom_3}{\atom_3}{\query}{\{\atom_3\}}{\ruleset}$.
\item Case 2: $p \neq p_2$. Let $t_2'$ be the (unique) term that precedes $t_2$ in $p$ (hence, in $p_1$), and let $p_2'$ be the length-one path from $t_2'$ to $t_2$ that immediately precedes $p_2$ in $p$. Clearly, $t_2 \in \terms{\atom_3}$, and we claim that $t_2 \in \terms{\atom_2}$. To see why,
let $q$ be the final predicate in $p_1$, which means the atom $q(t_2',t_2)$ occurs in $\chase{\instance}{\ruleset}$.
The connectivity property of $\chasegraph(\instance, \ruleset)$ implies that $t_2$ must occur in all atoms that label a path between any node labeled by $q(t_2',t_2)$ and $\node_3$.
As $p_2$ is the longest suffix of $p$ that is a path in $\chase{\{\atom_3\}}{\ruleset}$, we know that
$q(t_2',t_2)$ does not appear in a descendant of $\node_3$ in $\chasegraph(\instance, \ruleset)$. It follows that such a path necessarily goes through $\node_2$, hence $t_2$ occurs in $\atom_2$.
\end{itemize}

The second statement can be proven analogously. For the third statement, suppose that $(i_1, s_1, i_4, s_4) \in \transition{\atom_3}{\atom_3}{\query}{\instance}{\ruleset}$.
Then there exists a path $p$ in $\chase{\instance}{\ruleset}$ that begins in $t_3 = \atom_3[i_1]$ and ends in $t'_3 = \atom_3[i_4]$. If $p$ is a path in $\chase{\{\atom_3\}}{\ruleset}$, then the statement trivially holds. Otherwise, let $p_1$  (resp. $p_3$) be the longest prefix (resp. longest suffix) of $p$ ending (resp. starting) in a term $t_2$ (resp. $t'_2$) of $\atom_3$ whose label labels a path from $t_1$ to $t_2$ (resp. from $t'_2$ to $t_3$) in $\chase{\{\atom_3\}}{\ruleset}$. Note that since $p$ is not a path in $\chase{\{\atom_3\}}{\ruleset}$,
there exists a non-empty subpath $p_2$ such that $p=p_1 p_2 p_3$.
Arguing as above, we can show that $t_2$ and $t'_2$ must also belong to $\terms{\atom_2}$.
Indeed, $p_2$ starts with an atom that contains $t_2$ but does not appear in $\chase{\{\atom_3\}}{\ruleset}$, hence by the connectivity property, $t_2$ must appear in $\atom_2$, and similarly for the last atom of $p_2$ and term $t'_2$.
We can thus find indices $i_2, i_2', i_3,i_3'$ such that $t_2= \node_2[i_2]=\atom_3[i_2']$ and $t_2'= \node_2[i_3]=\node_3[i_3']$.
Putting everything together, we obtain the following: $(i_1, s_1, i_2', s_2) \in \transition{\atom_3}{\atom_3}{\query}{\{\atom_3\}}{\ruleset}$,
$(i_2, s_2, i_3, s_3) \in \transition{\atom_2}{\atom_2}{\query}{\instance}{\ruleset}$,
and $(i_3', s_3, i_4, s_4) \in \transition{\atom_3}{\atom_3}{\query}{\{\atom_3\}}{\ruleset}$.
\end{proof}

\def\computetrans{\ensuremath{\textsc{Trans}}}
\def\unexamined{\ensuremath{\textsc{Unexamined}}}
\def\countsteps{\ensuremath{\textsc{Count}}}
\def\maxsteps{\ensuremath{\textsc{Max}}}
\def\mapped{\ensuremath{\textsc{Mapped}}}

We now exploit %
this lemma and existing complexity results on RPQ answering to design a procedure for checking validity of a proof scheme, with the following complexity:

\begin{propositionrep}\label{validity-exp}
Checking the validity of a proof scheme can be done in single exponential time, in polynomial space if predicate arity is bounded, and in NL if $q$ and $\ruleset$ are fixed.
\end{propositionrep}

\begin{proofsketch}
Let $\proofScheme=(\psforest = (V,E), \pstransitions)$ be a proof scheme whose validity w.r.t.\ $(\instance,\ruleset,\query)$ we wish to test.
To establish the complexity bounds, we will show that a proof scheme is valid just in the case that some execution of the following non-deterministic procedure returns `valid'.
For the bound on the number of iterations of the inner while loop, we will use
$\maxsteps=P*(|A|+|\terms{I} \cup \terms{\atom_p} \cup \terms{\atom_c}|)^{A}*2^{3*N^2*A^2}$, where $P$ and $A$ are respectively the number of predicates and maximum predicate arity for the ruleset $\ruleset$,  and $N$ is the total number of states across all of the automata in $q$. (We will explain later the origin of this bound.)

\medskip

\noindent\textbf{Step 1}: Check that every root atom in $V$ belongs to $I$, and return `not valid' if not.
Next consider each pair $(\atom, \atom')$ of (not necessarily distinct) root atoms in turn.
For each possible transition $\tau=(i_1, \state_1, i_2, \state_2)$ w.r.t.\ $q$ and $(\atom, \atom')$, construct the corresponding RPQ $q_\tau= \subLanguage{\automaton}{\state_1}{\state_2}(\atom[i_1], \atom'[i_2])$ (with $\automaton$ the automaton containing states $\state_1,\state_2$) and use the RPQ Entailment oracle to decide whether $\query_\tau$ is entailed from $(\instance, \ruleset)$.
Store in $\computetrans(\atom,\atom')$ the transitions $\tau$ for which $\query_\tau$ is entailed.
If there is some $(\tau, (\atom,\atom')) \in \mathcal{T}$ such that $\tau \not \in \computetrans(\atom,\atom')$, return `not valid', else continue to Step~2. \smallskip

\newpage
\noindent\textbf{Step 2}: Let $\unexamined$ be the set of non-root nodes in $V$.
\smallskip \newline While $\unexamined \neq \emptyset$, choose some $\alpha_c \in \unexamined$ whose parent $\alpha_p$ is such that $\alpha_p \not \in \unexamined$, and proceed as follows:
\begin{itemize}
\item Remove $\alpha_c$ from $\unexamined$.
\item Set $\countsteps=0$ and $\alpha_i=\alpha_p$. %
\item While $\countsteps\leq \maxsteps$ and $\atom_i \neq \alpha_c$%
 \begin{enumerate}[label=(\alph*)]
 \item Guess a rule $\rho \in \ruleset$ that is applicable to $\alpha_i$ using substitution $\sigma$ (return `not valid' if no such rule).
 \item Let $\alpha_i'$ be the result of applying $\rho$ to $\alpha_i$ under $\sigma$. We assume that existential variables in the head of $\rho$ are replaced with either nulls
 from  $\terms{\alpha_c} \setminus (\terms{\alpha_p} \cup \terms{\alpha_i})$ or fresh nulls (not in $\terms{\mathcal{F}}$) using the next available ids.
\item For each possible transition $\tau=(k_1, s_1, k_2, s_2)$ w.r.t.\ $q$ and $(\atom_i', \atom_i')$
\begin{itemize}
\item construct the RPQ $q_\tau= \subLanguage{\automaton}{s_1}{s_2}(\atom_i'[k_1], \atom_i'[k_2])$ (with $\automaton$ the automaton containing states $s_1,s_2$)
\item use the RPQ Entailment oracle to decide if $q_\tau$ is entailed from $(\{\atom_i'\}, \ruleset)$
\item if $q_\tau$ is entailed, then store $\tau$ in $\computetrans^{\downarrow}(\atom_i',\atom_i')$
\end{itemize}
\item For every possible transition $\tau=(k_1, s_1, k_3, s_3)$ w.r.t.\ $q$ and $(\atom_p, \atom_i')$, add $\tau$ to  $\computetrans(\atom_p,\atom_i')$ if
there exist transitions $(k_1, s_1, k_2, s_2) \in \computetrans(\atom_p,\atom_i)$ and $(k_2', s_2, k_3, s_3) \in \computetrans^{\downarrow}(\atom_i',\atom_i')$
such that $\atom_i[k_2] = \atom_i'[k_2']$
\item For every possible transition $\tau=(k_1, s_1, k_3, s_3)$ w.r.t.\ $q$ and $(\atom_i', \atom_p)$, add $\tau$ to  $\computetrans(\atom_i',\atom_p)$ if
there exist transitions $(k_1, s_1, k_2', s_2) \in \computetrans^{\downarrow}(\atom_i',\atom_i')$ and $(k_2, s_2, k_3, s_3) \in \computetrans(\atom_i,\atom_p)$
such that $\atom_i'[k_2'] = \atom_i[k_2]$
\item For every possible transition $\tau=(k_1, s_1, k_4, s_4)$ w.r.t.\ $q$ and $(\atom_i', \atom_i')$, add $\tau$ to  $\computetrans(\atom_i',\atom_i')$ whenever there exist transitions $(k_1, s_1, k_2', s_2) \in \computetrans^{\downarrow}(\atom_i',\atom_i')$,
$(k_2, s_2, k_3, s_3) \in \computetrans(\atom_i,\atom_i)$, and $(k_3', s_3, k_4, s_4) \in \computetrans^{\downarrow}(\atom_i',\atom_i')$
such that $\atom_i'[k_2'] = \atom_i[k_2]$ and $\atom_i'[k_3'] = \atom_i[k_3]$
\item Set $\alpha_i = \alpha'_i$ and increment $\countsteps$.
  \end{enumerate}

\item  Return `not valid' if one of the following conditions does not hold:
\begin{itemize}
\item $\alpha_i = \alpha_c$
\item $\{\tau \mid (\tau,(\alpha_c,\alpha_c)) \in \mathcal{T}\} \subseteq \computetrans(\atom_c,\atom_c) = \computetrans(\atom_i,\atom_i)$
\item $\{\tau \mid (\tau,(\alpha_p,\alpha_c)) \in \mathcal{T}\} \subseteq \computetrans(\atom_p,\atom_c)= \computetrans(\atom_p,\atom_i)$
\item $\{\tau \mid (\tau,(\alpha_c,\alpha_p)) \in \mathcal{T}\} \subseteq \computetrans(\atom_c,\atom_p) = \computetrans(\atom_i,\atom_p)$
\end{itemize}
\end{itemize}
Return `valid'.
\medskip

We remark that the preceding procedure runs in polynomial space, except that it makes some calls to an oracle for RPQ Entailment.
As RPQ entailment is in \textsc{ExpTime} and \textsc{PSpace} = \textsc{NPSpace} $\subseteq$ \textsc{Exptime}, the procedure can be made to run in \textsc{Exptime}.
If we consider rules of bounded arity, then RPQ entailment is in \textsc{PSpace}, so the procedure runs in polynomial space. Finally, let us consider
the case where $q$ and $\ruleset$ are fixed and show that the procedure can be made to run in non-deterministic logarithmic space. For Step 1,
it suffices to iterate over the polynomially many combinations of possible transitions and pairs of root atoms, and for every such combination,
test whether the corresponding RPQ is entailed. As RPQ Entailment is in \textsc{NLogSpace} in data complexity, and it is known that $\textsc{NLogSpace}^\textsc{NLogSpace} \subseteq \textsc{NLogSpace}$,  Step 1 can be implemented  in  \textsc{NLogSpace}. For Step 2, by proceeding in a depth-first manner, we can examine each of the nodes in the proof scheme using only logarithmic space to keep track of which nodes remain to be visited. For each node $\alpha_c$ with parent $\alpha_p$, the maximal number of iterations of the inner while loop is bounded by a constant (as it depends only on $q$ and $\ruleset$). Moreover, it is easily verified that all operations in the inner while loop can be performed using constant space, except for Step 2(c), which involves a constant number of  RPQ Entailment checks, each of which can be handled by an $\textsc{NLogSpace}$ oracle. Thus, the entire procedure can be implemented in non-deterministic logspace, yielding the desired $\textsc{NLogSpace}$ upper bound.
We have thus shown that the procedure gives the required complexity bounds, and in the appendix, we establish its correctness.\qedhere
\end{proofsketch}

\begin{proof}
We prove here the correction of the procedure presented in the main body of the article. 
To simplify notations, we will use $\chasegraph$ to abbreviate $\chasegraph(\instance, \ruleset)$. 

\medskip

\noindent$(\Rightarrow)$ First we show that if the above procedure returns `valid', then $\proofScheme$ is a valid proof scheme. Let us thus consider an execution of the above procedure that returns `valid'. We need to show how to construct a witnessing embedding $\witness$ that satisfies all of the validity conditions. In fact, it will prove convenient to construct instead a mapping
$\mu$ from the atoms in $\proofScheme$ to nodes in $\chasegraph$, and then let $\witness$ be the unique mapping from $\terms{\proofScheme}$ to $\terms{\chase{\instance}{\ruleset}}$ that satisfies $\witness(\atom)=\cglabel(\mu(\atom))$. Note that by definition, every atom $\witness(\atom)$ belongs to $\chase{\instance}{\ruleset}$.

We begin by letting $\mu$ be the mapping that maps every root atom $\atom$ in $\proofScheme$ to the (unique) root atom in $\chasegraph$ satisfying $\cglabel(\node)=\atom$. It follows that the initial mapping $\witness$ has as its domain all terms appearing in a root atom of $\proofScheme$ and maps every such term to itself. 
Observe that %
since all root atoms belong to $I$ (else the procedure would have returned `not valid' in Step 1), $\witness$ trivially satisfies the first three validity conditions when restricted to the root atoms of $\proofScheme$. It is easily seen that for every pair $\atom, \atom'$ of root atoms in~$\proofScheme$, the set $\computetrans(\atom,\atom')$ computed in Step 1 using RPQ Entailment is exactly equal to $\transition{\witness(\atom)}{\witness(\atom')}{\query}{\instance}{\ruleset}$. Since Step 1 ensures that every $(\tau, (\atom,\atom')) \in \mathcal{T}$ is such that $\tau \in \computetrans(\atom,\atom')$, the fourth validity condition is satisfied for transitions between root atoms.

We next show how to extend $\mu$ (hence $\witness$) to cover the other atoms in $\proofScheme$.
Let us take some atom $\alpha_c$ for which $\mu$ has not yet been defined, but whose parent $\alpha_p$ is in the domain of $\mu$. We suppose that the current mapping $\mu$ induces a mapping $\witness$ that satisfies all of the validity conditions as regards the atoms and terms in its domain (in particular, with respect to the atom $\alpha_p$). We further suppose that the set $\computetrans(\atom_p,\atom_p)$ computed by the procedure is equal to $\transition{\witness(\atom_p)}{\witness(\atom_p)}{\query}{\instance}{\ruleset}$.
We know that at some point during the execution, the atom $\alpha_c$ was removed from $\unexamined$, and the inner while loop produced a sequence of atoms $\atom_1, \ldots, \atom_n$
such that $\atom_1=\atom_p$ and $\atom_n=\atom_c$ (the latter holds because of the check $\alpha_c=\alpha_i$ performed after exiting the inner while loop), where $n \leq \maxsteps+1$.
The procedure will have also constructed for every $1 \leq j \leq n$, sets of transitions $\computetrans(\atom_p,\atom_j)$, $\computetrans(\atom_j,\atom_p)$, and $\computetrans(\atom_j,\atom_j)$.
We will show how, for each $1 < j \leq n$, the mapping $\mu$ (and the induced mapping $\witness$) can be extended to the terms in $\atom_j$ in such a way that the following conditions hold:
\begin{enumerate}[label=(\roman*)]
\item $\mu(\atom_{j})$ is a child of $\mu(\atom_{j-1})$ in $\chasegraph$;
\item $\atom_j$ and $\witness(\atom_j)$ are isomorphic\footnote{
Two atoms $\atom_1$ and $\atom_2$ are \emph{isomorphic with respect to a bijection} $b$ from a set of terms $\termSet_1$ to a set of terms $\termSet_2$ if they have the same type 
and, 
for any position $i$ of their shared predicate, if $\atom_1[i] \in \termSet_1$ then $\atom_2[i] = b(\atom_1[i])$, otherwise $\atom_2[i] \not \in \termSet_2$,  and $\atom_1[i]$  and $\atom_2[i]$ are both nulls.
} w.r.t. the identity function on $\terms{\instance}$;
\item $\atom_j$ and $\witness(\atom_j)$ are isomorphic w.r.t. the bijection from $\terms{\atom_{j -1}}$ to $\terms{\witness(\atom_{j - 1})}$ induced by their common type;
\item $\computetrans(\atom_p,\atom_j) = \transition{\witness(\atom_p)}{\witness(\atom_j)}{\query}{\instance}{\ruleset}$;
\item $\computetrans(\atom_j,\atom_p) = \transition{\witness(\atom_j)}{\witness(\atom_p}{\query}{\instance}{\ruleset}$;
\item $\computetrans(\atom_j,\atom_j) = \transition{\witness(\atom_j)}{\witness(\atom_j)}{\query}{\instance}{\ruleset}$
\end{enumerate}

Let us thus assume that we have already defined $\mu(\atom_{j-1})$ in such a way as to satisfy the preceding properties,
and show how to do the same for $\atom_{j}$.
Let $b$ be the bijection from $\terms{\atom_{j -1}}$ to $\terms{\witness(\atom_{j - 1})}$ induced by their common type.
We know that $\atom_{j}$ was obtained by applying a rule $\rho = \varphi \rightarrow \chi \in \ruleset$ to $\atom_{j-1}$ with some substitution $\sigma$.
More precisely, $\sigma(\varphi)=\atom_{j -1}$ and $\atom_{j}$ is obtained from $\chi$ by replacing each frontier variable $z$ by $\sigma(z)$
and every existential variable $y$ by a null
 from  $\terms{\alpha_c} \setminus (\terms{\alpha_p} \cup \terms{\alpha_i})$ or a fresh null (not in $\terms{\mathcal{F}}$) using the next available id.
For convenience, we may assume that the nulls that occur in $\mathcal{F}$ or are generated by the procedure do not appear in $\chase{\instance}{\ruleset}$.
Because $\atom_{j-1}$ is isomorphic to $\witness(\atom_{j-1})$ with bijection $b$,
we have $b(\sigma(\varphi))=\witness(\atom_{j-1})$, so $\rho$ can be applied to $\witness(\atom_{j-1})$ under substitution $b \circ \sigma$.
It follows that the atom produced by this rule application is a child of $\mu(\atom_{j-1})$ in $\chasegraph$. 
We let $\mu(\atom_j)$ be this child. We observe that by construction and due to properties (ii) and (iii) holding for  $\atom_{j-1}$ and $\mu(\atom_{j-1})$, the atoms
$\atom_{j}$ and $\witness(\atom_j)=\cglabel(\mu(\atom_j))$ are isomorphic with respect to $b$ and with respect to the identity function on $\terms{\instance}$.
Thus, properties (i)-(iii) are satisfied for $\atom_j$.

It remains to show that our definition of $\witness(\atom_j)$ satisfies properties (iv)-(vi).
For property (iv), first suppose that  $\tau= (k_1, \state_1, k_3, \state_3) \in \computetrans(\atom_p,\atom_j)$.
Then $\tau$ must have been added to $\computetrans(\atom_p,\atom_j)$ in Step 2(d) due to the presence of
transitions $(k_1, \state_1, k_2, \state_2) \in \computetrans(\atom_p,\atom_{j-1})$ and $(k_2', \state_2, k_3, \state_3) \in \computetrans^{\downarrow}(\atom_j,\atom_j)$
with $\atom_{j-1}[k_2] = \atom_j[k_2']$.
By assumption, we have
$$(k_1, \state_1, k_2, \state_2) \in \transition{\witness(\atom_p)}{\witness(\atom_{j-1})}{\query}{\instance}{\ruleset}$$
Moreover, $(k_2', \state_2, k_3, \state_3) \in \computetrans^{\downarrow}(\atom_j,\atom_j)$ and the construction of $\computetrans^{\downarrow}(\atom_j,\atom_j)$ in Step 2(c) imply that
the RPQ $q_\tau= \subLanguage{\automaton}{\state_2}{\state_3}(\atom_j[k_2'], \atom_j[k_3])$  is entailed from $(\{\atom_j\}, \ruleset)$.
As $\atom_j$ and $\witness(\atom_j)$ are isomorphic,
it must also be the case that
$\subLanguage{\automaton}{\state_2}{\state_3}(\witness(\atom_j)[k_2'], \witness(\atom_j)[k_3])$  is entailed from $(\{\witness(\atom_j)\}, \ruleset)$.
It follows that %
$$(k_2', \state_2, k_3, \state_3) \in \transition{\witness(\atom_j)}{\witness(\atom_{j})}{\query}{\{\witness(\atom_j)\}}{\ruleset}$$
Since $\witness(\atom_j)$ appears in $\chase{\instance}{\ruleset}$, this implies in turn
$$(k_2', \state_2, k_3, \state_3) \in \transition{\witness(\atom_j)}{\witness(\atom_{j})}{\query}{\instance}{\ruleset}$$
From $\atom_{j-1}[k_2] = \atom_j[k_2']$ and the isomorphism holding between atoms $\atom_{j-1}$ and $\witness(\atom_{j-1})$ and between atoms $\atom_{j}$ and $\witness(\atom_{j})$,
we must have $\witness(\atom_{j-1})[k_2] = \witness(\atom_j)[k_2']$. We can thus
combine the preceding statements to obtain
$$(k_1, \state_1, k_3, \state_3) \in \transition{\witness(\atom_p)}{\witness(\atom_{j})}{\query}{\instance}{\ruleset}$$
which establishes that $$\computetrans(\atom_p,\atom_j) \subseteq \transition{\witness(\atom_p)}{\witness(\atom_j)}{\query}{\instance}{\ruleset}.$$
For  the second inclusion in (iv), let $\tau= (k_1, \state_1, k_3, \state_3)$ belong to $\transition{\witness(\atom_p)}{\witness(\atom_j)}{\query}{\instance}{\ruleset}$,
with $\automaton$ the automaton containing states $\state_1, \state_3$.
By Lemma \ref{decompose-trans},
there exist
$$(k_1, \state_1, k_2, \state_2) \in \transition{\witness(\atom_p)}{\witness(\atom_{j-1})}{\query}{\instance}{\ruleset}$$
and
$$(k_2', \state_2, k_3, \state_3) \in \transition{\witness(\atom_{j})}{\witness(\atom_{j})}{\query}{\{\witness(\atom_{j})\}}{\ruleset}$$
such that $\witness(\atom_{j-1})[k_2] = \witness(\atom_j)[k_2']$.
Applying our assumption, we obtain
$$(k_1, \state_1, k_2, \state_2) \in \computetrans(\atom_p,\atom_{j-1})$$
From
$(k_2', \state_2, k_3, \state_3) \in \transition{\witness(\atom_{j})}{\witness(\atom_{j})}{\query}{\{\witness(\atom_{j})\}}{\ruleset}$
and the fact that $\atom_j$ and $\witness(\atom_{j})$ are isomorphic,
we can infer
$(k_2', \state_2, k_3, \state_3) \in \transition{\atom_{j}}{\atom_{j}}{\query}{\{\atom_{j}\}}{\ruleset}$.
It follows that the RPQ $\query_\tau= \subLanguage{\automaton}{\state_2}{\state_3}(\atom_j[k_2'], \atom_j[k_3])$ is entailed from $(\{\atom_j\}, \ruleset)$,
so $$(k_2', \state_2, k_3, \state_3)\in \computetrans^{\downarrow}(\atom_j,\atom_j)$$
 because of how $\computetrans^{\downarrow}$ is constructed in  Step 2(c).
It follows that $(k_1, \state_1, k_3, \state_3)$ was added to $\computetrans(\atom_p,\atom_j) $ in Step 2(d),
which yields the second inclusion of property (iv).
Properties (v) and (vi) can be shown analogously, utilizing Lemma~\ref{decompose-trans}.

Since $\alpha_n=\alpha_c$, we can conclude that the definition of $\witness(\atom_c)$ satisfies properties (i)-(vi).
It follows that the definition of $\witness(\alpha_c)$ satisfies the third validity condition
(the first condition only concerns root nodes, and the second holds by definition, as explained earlier).
To show the fourth condition, take some $(\tau,(\atom_p,\atom_c)) \in \mathcal{T}$. Since the execution succeeded,
we know that $\{\tau \mid (\tau,(\alpha_p,\alpha_c)) \in \mathcal{T}\} \subseteq \computetrans(\atom_p,\atom_c)$,
and by property (iv), we have
$\computetrans(\atom_p,\atom_c) = \transition{\witness(\atom_p)}{\witness(\atom_c)}{\query}{\instance}{\ruleset}$.
We thus have $\tau \in \transition{\witness(\atom_p)}{\witness(\atom_c)}{\query}{\instance}{\ruleset}$. Transitions of the forms $(\tau,(\atom_c,\atom_p)) \in \mathcal{T}$
and $(\tau,(\atom_c,\atom_c)) \in \mathcal{T}$
are handled analogously, using properties (v) and (vi). 

\bigskip

\noindent$(\Leftarrow)$ We suppose that $\proofScheme=(\psforest = (V,E),\pstransitions)$ is a valid proof scheme w.r.t.\ $(\instance,\ruleset,\query)$ with witnessing embedding $\witness$  and show that there is an execution of the above procedure that returns `valid'.

Because of the first validity condition, we know that all root atoms in $V$ belong to $\instance$. It follows that this check will succeed in Step 1 of the procedure. Next
let $(\atom, \atom')$ be a pair of (not necessarily distinct) root atoms in $V$. As argued in the first direction of the proof, the set $\computetrans(\atom,\atom')$ computed in Step 1 using RPQ Entailment is exactly equal to $\transition{\atom}{\atom'}{\query}{\instance}{\ruleset}$ (recall that $\witness(\atom)=\atom$ and $\witness(\atom')=\atom'$, due to the first validity condition). We also know by the final validity condition that for each $(\tau,(\atom,\atom')) \in \pstransitions$, it holds that $\tau \in \transition{\atom}{\atom'}{\query}{\instance}{\ruleset}$. Combining the latter statements, we have that every $(\tau, (\atom,\atom')) \in \pstransitions$ is such that $\tau \in \computetrans(\atom,\atom')$. Thus, all checks in Step 1 succeed.

Let us now consider Step 2. At the start of Step 2, we initialize $\unexamined$ to the set of non-root nodes in $V$. Whenever we arrive at the beginning of the outer while loop, we choose arbitrarily some $\alpha_c \in \unexamined$ whose parent $\alpha_p$ is such that $\alpha_p \not \in \unexamined$. A simple examination of the procedure shows that no ancestor of $\atom_c$ belongs to $\unexamined$, hence every ancestor of $\atom_c$ is either a root node (and thus was considered during Step 1) or a node in $V$ that was selected during some iteration of the outer while loop. It follows that for every ancestor $\atom'$ of $\atom_c$, the set $\computetrans(\atom',\atom')$ will have been computed by the procedure.
We know from the previous paragraph that if $\atom'$ is a root node, then $\computetrans(\atom',\atom')$ contains precisely the transitions in $\transition{\witness(\atom')}{\witness(\atom')}{\query}{\instance}{\ruleset}$. We will assume that the same holds for every predecessor $\atom'$ of $\atom_c$, in particular, its parent $\alpha_p$,
and our objective will be to show that we can perform the while loop for $\alpha_c$ in such a way that at the end of the inner while loop, we have $\alpha_i = \alpha_c$ and the computed sets
$\computetrans(\atom_c,\atom_c)$, $\computetrans(\atom_p,\atom_c)$, and $\computetrans(\atom_c,\atom_p)$ satisfy:
\begin{itemize}
\item $\computetrans(\atom_c,\atom_c)=\transition{\witness(\atom_c)}{\witness(\atom_c)}{\query}{\instance}{\ruleset}$
\item $\computetrans(\atom_p,\atom_c)=\transition{\witness(\atom_p)}{\witness(\atom_c)}{\query}{\instance}{\ruleset}$
\item $\computetrans(\atom_c,\atom_p)=\transition{\witness(\atom_c)}{\witness(\atom_p)}{\query}{\instance}{\ruleset}$
\end{itemize}
Note that the preceding conditions, together with the final validity condition, imply that:
\begin{itemize}
\item $\{\tau \mid (\tau,(\alpha_c,\alpha_c)) \in \mathcal{T}\} \subseteq \computetrans(\atom_c,\atom_c)$ 
\item $\{\tau \mid (\tau,(\alpha_p,\alpha_c)) \in \mathcal{T}\} \subseteq \computetrans(\atom_p,\atom_c)$ 
\item $\{\tau \mid (\tau,(\alpha_c,\alpha_p)) \in \mathcal{T}\} \subseteq \computetrans(\atom_c,\atom_p)$ 
\end{itemize}
It follows that the checks for $\alpha_c$ at the end of the outer while loop will succeed. Since we show how to successfully perform the steps within the outer while loop for every non-root atom $\atom_c$, we will eventually reach a state in which $\unexamined=\emptyset$, at which point the procedure will return `valid', as desired.

We thus let $\atom_c$ be the non-root node from $\unexamined$ that has just been selected for examination, and let $\alpha_p$ be its unique parent, for which we assume that
$\computetrans(\atom_p,\atom_p)= \transition{\witness(\atom_p)}{\witness(\atom_p)}{\query}{\instance}{\ruleset}$. 
The third  
validity condition ensures that $\witness(\atom_c)$ can be derived from $\witness(\atom_p)$, which ensures that there is a descendant $\node_n$ of $\node_0 = \mu(\atom_p)$ in $\chasegraph$ such that $\cglabel(\node_n) = \witness(\atom_c)$. 

\medskip

We let $\node_0, \node_1, \ldots, \node_n$ be the unique shortest path in from $\node_0$ to $\node_n$ in $\chasegraph$, and set $\beta_j=\cglabel(\node_j)$ for every $1 \leq j \leq n$. We will use the sequence of atoms $\beta_0, \ldots, \beta_n$ in $\chase{\instance}{\ruleset}$ to decide which rules and substitutions to guess in Step 2(a). We do not assume $n \leq \maxsteps$, but will instead show later that if $n > \maxsteps$, we can extract a subsequence of the required length with the desired properties. Thus, for the moment, our goal is to show how to successfully execute Step 2 if there were no bound on the number of iterations of the inner while loop, and afterwards we will show how the execution can be modified so as to also satisfy this bound.

Set  $\gamma_0=\atom_p$ and denote by $\gamma_j$ the atom $\atom_i'$ obtained in Step 2(b) after performing $j$ rule applications (with $1 \leq j \leq n$). We will show that for every $0 \leq j \leq n$,
the following conditions hold:
\begin{enumerate}[label=(\roman*)]
\item $\computetrans(\gamma_j,\gamma_j)=\transition{\beta_j}{\beta_j}{\query}{\instance}{\ruleset}$;
\item $\computetrans(\gamma_0,\gamma_j)=\transition{\beta_0}{\beta_j}{\query}{\instance}{\ruleset}$;
\item $\computetrans(\gamma_j,\gamma_0)=\transition{\beta_j}{\beta_0}{\query}{\instance}{\ruleset}$;
\item $\gamma_j$ and $\beta_j$ are isomorphic with respect to the bijection $b$ from $\terms{\gamma_j}$ to $\terms{\beta_j}$ induced by their common type, and also w.r.t. the identity function on $\terms{\instance}$;
\item $\beta_j[m]=\witness(\alpha_c)[m']$ iff $\gamma_j[m]=\alpha_c[m']$.
\end{enumerate}
It follows from the last condition that $\gamma_n=\alpha_c$.
We also point out that the algorithm will only keep the most recent $\computetrans$ sets in memory, so in the above conditions, by $\computetrans(\gamma_j,\gamma_j)$, we mean the value of the set $\computetrans(\alpha_i',\alpha_i')$ at the end of the iteration in which $\alpha_i'=\gamma_j$ (and likewise for the sets
$\computetrans(\gamma_0,\gamma_j)$ and $\computetrans(\gamma_j,\gamma_0)$).
Further note that since $\beta_0=\witness(\alpha_p)$, $\beta_n=\witness(\alpha_c)$, $\gamma_0=\alpha_p$, and $\gamma_n=\alpha_c$, the above conditions imply that
\begin{itemize}
\item $\computetrans(\atom_c,\atom_c)=\transition{\witness(\atom_c)}{\witness(\atom_c)}{\query}{\instance}{\ruleset}$
\item $\computetrans(\atom_p,\atom_c)=\transition{\witness(\atom_p)}{\witness(\atom_c)}{\query}{\instance}{\ruleset}$
\item $\computetrans(\atom_c,\atom_p)=\transition{\witness(\atom_c)}{\witness(\atom_p)}{\query}{\instance}{\ruleset}$
\end{itemize}
which is exactly what we need.

We first note that when $j=0$, we have $\gamma_j=\gamma_0=\atom_p$, so conditions (i)-(v) hold by our earlier assumptions about $\atom_p$. Next suppose that we have already shown how to perform $j$ iterations of the while loop (hence $j$ rule applications), for some $0 \leq j < n$, in such a way that conditions (i)-(v) hold for $\gamma_j$. We now show how to perform the $j+1^\textrm{st}$ iteration so that these conditions hold also for $\gamma_{j+1}$. Since $\node_{j+1}$ is a child of $\node_j$ in $\chasegraph$, it follows that there exists a rule $\rho \in \ruleset$ that is applicable to the atom $\beta_j=\cglabel(\node_j)$ under a substitution $\sigma$ and whose application yields the atom $\beta_{j+1} = \cglabel(\node_{j+1})$. We also know from (iv) that $\gamma_j$ and $\beta_j$ are isomorphic with respect to the bijection $b$ from $\terms{\gamma_j}$ to $\terms{\beta_j}$ induced by their common type. It follows that $\rho$ is applicable to $\gamma_j$ under the substitution mapping a variable $z$ in the body of $\rho$ to $b^{-1}(\sigma(z))$. We let $\gamma_{j+1}$ be the atom obtained in Step 2(b) from applying $\rho$ to $\gamma_j$ under this substitution, with existential variables in the head instantiated as follows:
\begin{itemize}
\item If an existential variable $z_0$ occurs at position $k$ in the head atom and $\beta_{j+1}[m]=\witness(\alpha_c)[m']$,
then we instantiate $z_0$ with $\alpha_c[m']$.
\item All other existential variables in the rule head are instantiated with fresh nulls (not in $\terms{\mathcal{F}}$ nor already used earlier in the procedure) using the next available ids.
\end{itemize}
Observe that the preceding instantiation %
 is well defined as whenever an existential variable $z_0$ is instantiated with $\beta_{j+1}[m]=\witness(\alpha_c)[m']$,
 it must be the case that $\beta_{j+1}[m]=\witness(\alpha_c)[m']$ is a null value that was created in $\beta_{j+1}$  in $\chasegraph$. 
In particular, this means that $\alpha_c[m']$ is a null that does not occur in
 $\gamma_j$ nor $\gamma_0=\alpha_p$ (indeed, 
 if $\alpha_c[m']$ were to occur in $\gamma_j$, then by property (v), $\psi(\alpha_c)[m']$ would appear in $\beta_j$).  %
We further observe that condition (iv) holds: $\gamma_{j+1}$ is isomorphic to $\beta_{j+1}$ under the bijection from $\terms{\gamma_{j+1}}$ to $\terms{\beta_{j+1}}$ induced by their common type, as well as w.r.t. the identity function on $\terms{\instance}$ (for the latter, note that any constant from $\terms{\instance}$ that occurs in $\gamma_{j+1}$ must have appeared in $\gamma_j$, and hence in the same positions in $\beta_j$, and vice-versa). Moreover, by construction, 
condition (v) is also satisfied: by induction for values from $\atom_c$ or $\psi(\atom_c)$ that appear already in $\gamma_j$ or $\beta_j$, 
and by the policy for null naming for freshly instantiated nulls.

Now that we have the new atom $\gamma_{j+1}$, we will consider in Step 2(c) each possible transition $\tau=(k_1, s_1, k_2, s_2)$ w.r.t.\ $q$ and $(\gamma_{j+1}, \gamma_{j+1})$
and construct the corresponding RPQ $q_\tau= \subLanguage{\automaton}{s_1}{s_2}(\gamma_{j+1}[k_1], \gamma_{j+1}[k_2])$ (with $\automaton$ the automaton containing states $s_1,s_2$). We will then call an RPQ Entailment oracle to decide whether $q_\tau$ is entailed from $(\{\gamma_{j+1}\}, \ruleset)$, and if
$q_\tau$ is entailed, we will store $\tau$ in $\computetrans^{\downarrow}(\gamma_{j+1},\gamma_{j+1})$. We note that since $\gamma_{j+1}$ and $\beta_{j+1}$ are isomorphic, $q_\tau$ is entailed from $(\{\gamma_{j+1}\}, \ruleset)$ iff $q_\tau' = \subLanguage{\automaton}{s_1}{s_2}(\beta_{j+1}[k_1], \beta_{j+1}[k_2])$ is entailed from $(\{\beta_{j+1}\}, \ruleset)$. The latter implies that $(k_1, s_1, k_2, s_2) \in \transition{\beta_{j+1}}{\beta_{j+1}}{\query}{\{\beta_{j+1}\}}{\ruleset}$.
Putting this together, we have that $\computetrans^{\downarrow}(\gamma_{j+1},\gamma_{j+1})=\transition{\beta_{j+1}}{\beta_{j+1}}{\query}{\{\beta_{j+1}\}}{\ruleset}$.

We now compute in (d), (e), and (f) the sets $\computetrans(\gamma_{j+1},\gamma_{j+1})$,
$\computetrans(\gamma_0,\gamma_{j+1})$, and $\computetrans(\gamma_{j+1},\gamma_0)$, by using the transitions present in
$\computetrans^{\downarrow}(\gamma_{j+1},\gamma_{j+1})$, $\computetrans(\gamma_0,\gamma_j)$,
$\computetrans(\gamma_j,\gamma_0)$, and $\computetrans(\gamma_j,\gamma_j)$.
We know from our induction hypothesis that conditions (i)-(iii) hold, meaning that:
$\computetrans(\gamma_j,\gamma_j)=\transition{\beta_j}{\beta_j}{\query}{\instance}{\ruleset}$,
$\computetrans(\gamma_0,\gamma_j)=\transition{\beta_0}{\beta_j}{\query}{\instance}{\ruleset}$,
and $\computetrans(\gamma_j,\gamma_0)=\transition{\beta_j}{\beta_0}{\query}{\instance}{\ruleset}$.
It follows from Lemma \ref{decompose-trans} that conditions (i)-(iii) hold also for $\gamma_{j+1}$.
Indeed, to obtain (ii), we need to show that $\computetrans(\gamma_0,\gamma_{j+1})=\transition{\beta_0}{\beta_{j+1}}{\query}{\instance}{\ruleset}$.
First suppose $(k_1, s_1, k_3, s_3) \in \computetrans(\gamma_0,\gamma_{j+1})$. Then because of the construction in (d),
we know that there exist transitions $(k_1, s_1, k_2, s_2) \in \computetrans(\gamma_0,\gamma_j)$ and $(k_2', s_2, k_3, s_3) \in \computetrans^{\downarrow}(\gamma_{j+1},\gamma_{j+1})$ such that $\gamma_j[k_2] = \gamma_{j+1}[k_2']$.
Using the earlier equalities, we get $(k_1, s_1, k_2, s_2) \in \transition{\beta_0}{\beta_j}{\query}{\instance}{\ruleset}$
and $(k_2', s_2, k_3, s_3) \in \transition{\beta_{j+1}}{\beta_{j+1}}{\query}{\{\beta_{j+1}\}}{\ruleset}$,
which yields $(k_1, s_1, k_3, s_3) \in \transition{\beta_0}{\beta_{j+1}}{\query}{\instance}{\ruleset}$ by the first item of Lemma \ref{decompose-trans}.
For the other direction, suppose $(k_1, s_1, k_3, s_3) \in \transition{\beta_0}{\beta_{j+1}}{\query}{\instance}{\ruleset}$.
Then by the first item of Lemma \ref{decompose-trans}, there must exist
$(k_1, s_1, k_2, s_2) \in \transition{\beta_0}{\beta_j}{\query}{\instance}{\ruleset}$ and
$(k_2', s_2, k_3, s_3) \in \transition{\beta_{j+1}}{\beta_{j+1}}{\query}{\{\beta_{j+1}\}}{\ruleset}$ with $\beta_j[k_2]=\beta_{j+1}[k_2']$.
Using the earlier equalities, we get
$(k_1, s_1, k_2, s_2) \in \computetrans(\gamma_0,\gamma_j)$ and $(k_2', s_2, k_3, s_3) \in \computetrans^{\downarrow}(\gamma_{j+1},\gamma_{j+1})$,
which implies that the transition $(k_1, s_1, k_3, s_3)$ was added to $\computetrans(\gamma_0,\gamma_{j+1})$ during Step  2(d).
The arguments for conditions (i) and (iii) proceed similarly.

\medskip

To complete the proof, we must argue that we can make the described execution respect the bound $\maxsteps$ on the number of iterations of the inner while loop.
Consider some non-root node $\alpha_c$ and its parent $\alpha_p$, and let $\beta_0, \ldots, \beta_n$ (with $\beta_0=\alpha_p$ and $\beta_n=\alpha_c$) be the sequence of atoms  in $\chase{\instance}{\ruleset}$ that were used to decide which rules and substitutions to guess in Step 2(a). We let $\rho_1, \ldots, \rho_n$ and $\sigma_1, \ldots, \sigma_n$ be the chosen sequences of rules and substitutions.
If $n\leq \maxsteps$, then the execution of the inner while loop we defined for $\alpha_c$ already respects the bound.
Let us thus suppose that $n > \maxsteps$.
Thus, since $n > \maxsteps$, there must exist $0 \leq j < k \leq n$ such that:
\begin{itemize}
\item[(a)] $\atom_j$ and $\atom_{k}$ are isomorphic w.r.t.\ the identity on $\terms{I} \cup \terms{\atom_p} \cup \terms{\atom_c}$
\item[(b)] $\computetrans(\atom_j,\atom_j) = \computetrans(\atom_{k},\atom_{k})$
\item[(c)] $\computetrans(\atom_p,\atom_j) = \computetrans(\atom_p,\atom_{k})$
\item[(d)] $\computetrans(\atom_j,\atom_p) = \computetrans(\atom_{k},\atom_p)$
\end{itemize}
Indeed, there can be at most $P*(|A|+|\terms{I} \cup \terms{\atom_p} \cup \terms{\atom_c}|)^{A}$ ways of selecting a type and deciding which terms from $\terms{I} \cup \terms{\atom_p} \cup \terms{\atom_c}$ appear in which positions, and there are $N^2*A^2$ possible transitions between any pair of atoms, hence at most $2^{N^2*A^2}$ possible sets of transitions
(we recall that $P$ and $A$ are respectively the number of predicates and maximum predicate arity for the ruleset $\ruleset$, and $N$ is the total number of states across all of the automata in $q$).
Intuitively, conditions (a)-(d) ensure that $\atom_j$ and $\atom_k$ behave the same, so the idea is to jump straight from $\atom_j$ to $\atom_{k+1}$, and in this manner obtain
a shorter sequence that still defines a successful execution of the inner while loop but now respects the $\maxsteps$ bound.
Of course, since $\atom_j$ and $\atom_k$ may use different values for terms not in $\terms{I} \cup \terms{\atom_p} \cup \terms{\atom_c}$,
we need to do some renaming of nulls to obtain the new shorter sequence.
So, let $\pi_k$ be the bijection from $\terms{\atom_k}$ to $\terms{\atom_j}$ witnessing the isomorphism from item (a).
Then we can apply the rule $\rho_k$ to $\atom_j$ using the substitution $\sigma_k'$ that maps $t$ to $\pi_k(\sigma_k(t))$.
Call $\atom_{k+1}'$ the result of this rule application. It follows from item (a) that
$\atom_{k+1}$ and $\atom_{k+1}'$ are isomorphic w.r.t.\ the identity on $\terms{I} \cup \terms{\atom_p} \cup \terms{\atom_c}$ (we let $\pi_{k+1}$ be a witnessing bijection from
$\atom_{k+1}$ to $\atom_{k+1}'$).
Moreover, by items (b)-(d) and the way $\computetrans$ is constructed in the procedure, we must also have:
\begin{itemize}
\item $\computetrans(\atom_{k+1},\atom_{k+1}) = \computetrans(\atom'_{k+1},\atom'_{k+1})$
\item  $\computetrans(\atom_p,\atom_{k+1}) = \computetrans(\atom_p,\atom'_{k+1})$
\item $\computetrans(\atom_{k+1},\atom_p) = \computetrans(\atom'_{k+1},\atom_p)$.
\end{itemize}
Thus, $\atom_{k+1}$ and $\atom_{k+1}'$ behave the same in all relevant respects, and in particular,
we can apply rule $\rho_{k+1}$ to $\atom_{k+1}'$ using the substitution $\sigma_{k+1}'$ defining by setting $\sigma_{k+1}'(t)=\pi_{k+1}(\sigma_{k+1}(t))$.
It is readily verified, by repeating the preceding argument, that
$\atom_{k+1}$ and $\atom_{k+1}'$ are isomorphic w.r.t.\ the identity on $\terms{I} \cup \terms{\atom_p} \cup \terms{\atom_c}$.
Continuing in this manner, we can define a sequence $\atom_{k+1}', \ldots, \atom_n'$ such that for every $k+1 \leq u < n$:
\begin{itemize}
\item $\atom_{u+1}$ and $\atom_{u+1}'$ are isomorphic w.r.t.\ the identity on $\terms{I} \cup \terms{\atom_p} \cup \terms{\atom_c}$, as witnessed by a bijection $\pi_{u}$
from $\atom_{u+1}$ to $\atom'_{u+1}$
\item $\atom_{u+1}'$ is obtained from $\atom_u'$ by applying $\rho_u$ using substitution $\sigma'_u$ defined by $\sigma_u'(t)= \pi_u(\sigma_u(t))$
\item $\computetrans(\atom_{u+1},\atom_{u+1}) = \computetrans(\atom'_{u+1},\atom'_{u+1})$
\item  $\computetrans(\atom_p,\atom_{u+1}) = \computetrans(\atom_p,\atom'_{u+1})$
\item $\computetrans(\atom_{u+1},\atom_p) = \computetrans(\atom'_{u+1},\atom_p)$.
\end{itemize}
It follows that by guessing the sequence of rules $\rho_1, \ldots, \rho_{j-1}, \rho_{k}, \rho_{k+1}, \ldots, \rho_n$
and sequence of substitutions $\sigma_1, \ldots, \sigma_{j-1}, \sigma_k', \sigma_{k+1}', \ldots, \sigma_n'$, we are able to correctly perform the inner while loop. Moreover, by iterating this shortening operation, we will eventually obtain sequences of rules and substitutions 
which allow us to terminate within $\maxsteps$ steps while satisfying the other three conditions at the end of Step 2.

We have thus shown that it is possible to define for every valid proof scheme w.r.t.\ $(\instance,\ruleset,\query)$ an execution of the procedure that returns `valid', which completes the proof.
\end{proof}

{Note that the proposition can be extended to arbitrary linear rules with multiple head atoms (translated into atomic-head linear rules). 
Indeed, for combined complexity in the bounded predicate arity case, we know that RPQ entailment remains in \textsc{PTime}, see Proposition \ref{loop-bounded-arity} and Theorem \ref{thm-linear-rpq-answering}; hence, checking the validity of a proof scheme remains in \textsc{PSpace}.
}

It follows from Proposition \ref{validity-exp} that enumerating polysize proof schemes and checking for their validity and for the existence of a match of $\query$ in them yields a sound and complete algorithm for CRPQ answering.

\begin{theorem}
\label{thm-linear-complexity}
CRPQ answering under linear rules is \textsc{ExpTime}-complete in combined complexity,  \textsc{PSpace}-complete in combined complexity with bounded-predicate arity, and \textsc{NL}-complete in data complexity. 
\end{theorem}

\begin{proof}
By Proposition\ \ref{prop-match-valid}, CRPQ entailment comes down to deciding whether there exists a polynomial (in the query) proof scheme that is valid and contains a match for the query.  

Though polysize proof schemes are defined on an infinite number of atoms, and are thus infinite in number, we can enumerate in exponential time those that are different up to a renaming of nulls. Thus, for the general case, we can enumerate all polysize proof schemes that are distinct up to a renaming of nulls in polynomial space (and exponential time) and check whether there exists a valid proof scheme that contains a match for the query. As already noted, deciding existence of a match in a proof scheme is in NP, and by Proposition\ \ref{validity-exp},
validity checking is in \textsc{ExpTime}. Thus, the described procedure runs in single-exponential time. When the arity is bounded, validity checking is in \textsc{PSpace} (by Proposition\ \ref{validity-exp}), and so the overall procedure runs in polynomial space. Finally, if both $q$ and $\ruleset$ are fixed, 
the number of relevant proof schemes and their size are constants. Thus, we can enumerate all proof schemes in constant time, and checking for a match is also in constant time. By 
Proposition\ \ref{validity-exp}, the validity checking is in \textsc{NL} w.r.t.\ the size of the data, so the overall procedure runs in non-deterministic logspace for data complexity. 

The \textsc{ExpTime} lower bound in combined complexity and the \textsc{NL} lower bound in data complexity are immediate consequence of the \textsc{ExpTime}-hardness (resp.\ \textsc{NL}-hardness) of RPQ answering for linear rules w.r.t. combined (resp.\ data) complexity. The \textsc{PSpace} lower bound in the bounded arity case is inherited from the \textsc{PSpace}-hardness of CRPQ answering over DL-Lite$_\mathcal{R}$ knowledge bases \citep{DBLP:journals/jair/BienvenuOS15}. 
\end{proof}

\section{CRPQ Answering under Guarded Rules: Upper Bound}
\label{sec-crpq-guarded-upper}

To upper bound the complexity of  CRPQ answering under guarded rules, we  reduce this problem to CRPQ answering  under linear rules. For simplicity, we assume that all the predicates in $\instance$ also occur in $\ruleset$ (however, the reduction is easily extended to drop this restriction).

Two main ideas underly the reduction. 
First, we define an alternative notion of derivation, called `locally complete' derivation, which ensures the following property:  
at each step $k$ of the derivation, the considered extended instance $\hat{\instance}_k$ contains exactly the subset of the chase restricted to its terms, \emph{i.e.}, $\chase{\instance}{\ruleset}|_{\terms{\hat{\instance}_k}}$. 
 Of course such a derivation is not computable for an arbitrary set of rules, because it requires the decidability of (atom) entailment, but it is for guarded rules. 
Second, the application of guarded rules (in a locally complete derivation) is simulated through linear rules (in a classical derivation) expressed on an extended vocabulary, in which each predicate represents a guarded set of atoms. We use a finite set of canonical variables to rename terms occurring in the chase atoms. This ensures  there are finitely many guarded sets to be considered.   

Let us first formally introduce the notion of locally complete derivation.
A locally complete derivation is a sequence of extended instances $\hat{\instance}_{i}$ that have the property of being closed with respect to the atoms on $\terms{\hat{\instance}_{i}}$ that can be derived from $\hat{\instance}_{i}$. Starting from an instance $\instance$, $\hat{\instance}_0$ is obtained from $\instance$ by closing it under atomic entailment. At each step $i$, a rule is applied on $\hat{\instance}_{i}$ and the resulting set is again closed to yield $\hat{\instance}_{i+1}$.

\newpage
\begin{definition}[Locally complete derivation]
\label{def-locally-complete}
 Let $\instance$ be an instance and $\ruleset$ be a set of existential rules. A \emph{locally complete $\ruleset$-derivation} 
 from $\instance$ resulting in $\hat{\instance}_n$ is a sequence $\hat{\instance}_0,
  \hat{\instance}_1,\ldots,\hat{\instance}_n$ with $\hat{\instance}_0 = 
\chase{\instance}{\ruleset}_{\mid \terms{\instance}} $
and such that for all $i \geq 0$, there are $\erule_i \in \ruleset$ and a homomorphism $\match_i$ from $\body{\erule_i}$ to $\hat{\instance}_i$ such that  $\hat{\instance}_{i+1} = \chase{\hat{\instance}'_i}{\ruleset}_{\mid \terms{\hat{\instance}'_i}}$ where $\hat{\instance}'_i = \hat{\instance}_i \cup \match_i^{\mathrm{safe}}(\head{\erule_i})$. 
\end{definition}

Note that a locally complete derivation could have been defined as a possibly infinite sequence, but we will only need finite sequences to state our results. 
Next, we show that (finite) classical derivations and locally complete derivations are equivalent for our purposes, meaning that they entail the same Boolean CRPQs. This result holds more generally  for homomorphism-closed queries, as stated in Proposition~\ref{corollary-classical-alternative}.

\medskip
{Let us first illustrate the notion of locally complete derivation on the running example.}

\begin{example}[Running Example---continued] Let  $\ruleset = \{\erule_1, \ldots, \erule_6\}$ be the considered set of rules, which we recall below for reading convenience. 
\begin{quote}
$(\erule_1) \quad \fun{isFriendOf}(x,y) \rightarrow  \fun{isFriendOf}(y,x)$
\\
$(\erule_2) \quad \fun{isFriendOf}(x,y) \rightarrow  \fun{follows}(x,y)$ 
\\
$(\erule_3) \quad \fun{follows}(x,y) \rightarrow  \exists m~\fun{message}(m,x,y)$ 
\\
$(\erule_4) \quad \fun{follows}(x,y) \land  \fun{follows}(y,x) \rightarrow  \fun{isPaired}(x,y)$ 
\\
$(\erule_5) \quad \fun{message}(m,x,y) \rightarrow  \fun{sends}(x,m)$ 
\\
$(\erule_6) \quad\fun{message}(m,x,y) \rightarrow  \fun{receives}(y,m)$ 
\end{quote}

Consider again the instance $\instance = \{\fun{follows(B,A)}, \fun{isFriendOf(C,A)}, \fun{isFriendOf(C,B)}\}$.
Then, $\hat{\instance}_0$ is obtained from $\instance$ by adding all the entailed facts on terms in $\{\fun{A,B,C}\}$:
\begin{quote}
$ \fun{isFriendOf(A,C)}, \fun{isFriendOf(B,C)}$,	\\
$\fun{follows(C,A)}, \fun{follows(C,B)}, \fun{follows(A,C)}, \fun{follows(B,C)}$, \\
$\fun{isPaired(C,A)},\fun{isPaired(A,C)},  \fun{isPaired(C,B)},\fun{isPaired(B,C)}$. 
\end{quote}
In a classical derivation, $\hat{\instance}_0$ would be obtained from $\instance$ by triggering the rules $\erule_1$, $\erule_2$ and $\erule_4$. 
Now, assume Rule $\erule_3$ is applied on $\hat{\instance}_0$ by homomorphism $\match = \{x \mapsto \fun{B}, y \mapsto \fun{A} \}$. Then, 
$\instance'_1 = \hat{\instance}_0 \cup \{\fun{message(m_0,\fun{B},\fun{A})}\}$, and 
$\hat{\instance}_1 = \instance'_1 \cup \{ \fun{sends}(\fun{B}, m_0), \fun{receives}(\fun{A}, m_0)\}$. In a classical derivation, $\hat{\instance}_1$ would be obtained from $\instance'_1$  by triggering the rules $\erule_5$ and $\erule_6$.
The locally complete derivation can be continued by applying $\erule_3$ on the other facts with predicate $\fun{follows}$. 

Let us outline the main ideas of the reduction presented next. 
We define a finite set of \emph{complex predicates}, which are canonical encodings of all the guarded sets on a vocabulary. 
Then, each atom $\atom$ in $\instance$ is replaced by a complex atom (i.e., an atom with a complex predicate) encoding $\atom$ and its guarded set in $\hat{\instance}_0$. For example,  
$\atom =  \fun{isFriendOf(C,A)}$ is replaced by a complex atom encoding $\atom$ and the following set: 
$$\{ \fun{isFriendOf(C,A)}, \fun{isFriendOf(A,C)}, \fun{follows(C,A)}, \fun{follows(A,C)}, \fun{isPaired(C,A)},\fun{isPaired(A,C)} \}$$
Furthermore, each rule from $\ruleset$ is replaced by a set of \emph{complex rules}, which are linear rules using complex predicates, such that each application of a complex rule simulates a step of a locally complex derivation.  Finally, \emph{reconstruction rules} translate back atoms on the complex predicates into atoms on the initial predicates. We will illustrate the multiple intricacies of this translation using specific technical examples.  
 \end{example}
 
 \color{black}

The following Lemmas \ \ref{prop-alternative-complete} and \ \ref{prop-alternative-sound} allow us to show that classical derivations and locally complete derivations are able to derive exactly the same conclusions.

\begin{lemma}
\label{prop-alternative-complete}
For any classical $\ruleset$-derivation from $\instance$ to $\instance_n$, there exists a locally complete $\ruleset$-derivation from $\instance$ to $\hat{\instance}_{m}$ such that $\hat{\instance}_{m} \models \instance_n$. 
\end{lemma}

\begin{proof}
We prove the result by induction on the length of the classical derivation. If $n = 0$,  
 the result holds since $\instance_0 \subseteq \hat \instance_0$. Otherwise, let us assume that the result holds for any classical derivation of length $n \geq 0$. Let $\instance = \instance_0, \ldots, \instance_n,\instance_{n+1}$ be a classical derivation of length $n+1$, and $\erule_n$ and $\match_n$ be such that $\instance_{n+1}$ is obtained from $\instance_n$ by applying the trigger $(\erule_n,\match_n)$. By induction assumption, there is a locally complete $\ruleset$-derivation from $\instance$ resulting in $\hat{\instance}_{m}$ such that $\hat{\instance}_{m} \models \instance_n$. 
  Hence, there is a homomorphism $h_n$ from $\instance_n$ to $\hat{\instance}_m$. Thus $(\erule_n, h_n \circ \match_n)$ is a trigger on $\hat{\instance}_m$, and $h_n$ can be extended to a homomorphism from $\instance_{n+1}$ to $\hat{\instance}'_{m+1} = \hat{\instance}_m \cup (h_n\circ \match_n)^{\mathrm{safe}}(\head{\erule_n})$ by mapping the existential variables in $\match_n^{\mathrm{safe}}(\head{\erule_n})$
    to the corresponding existential variables in $(h_n\circ \match_n)^{\mathrm{safe}}(\head{\erule_n})$. 
  Hence, the locally complete derivation from $\instance$ resulting in $\hat{\instance}_{m}$ can be extended to a locally complete derivation from $\instance$ resulting in $\hat{\instance}_{m+1} = \chase{\hat{\instance}'_{m+1}}{\ruleset}_{\mid \terms{\hat{\instance}'_{m+1}}}$, 
  and  $\hat{\instance}_{m+1} \models \hat{\instance}'_{m+1} \models \instance_{n+1}$, 
  which concludes this proof.
\end{proof}

\begin{lemma}
\label{prop-alternative-sound}
For any locally complete $\ruleset$-derivation from $\instance$ resulting in $\hat{I}_n$, there exists a classical $\ruleset$-derivation from $\instance$ resulting in $\instance'$ such that $\instance' \models \hat{\instance}_n$.
\end{lemma}

\begin{proof}
We prove the result by induction on the length of a locally complete derivation $\hat{\instance}_0,\ldots,\hat{\instance}_n$.
For $n= 0$, $\hat{\instance}_n = \chase{\instance}{\ruleset}_{\mid \terms{\instance}}$. As $\chase{\instance}{\ruleset} \models \chase{\instance}{\ruleset}_{\mid \terms{\instance}}$, there is $\instance_m$ derivable from $\instance$ such that $\instance_m \models \hat{\instance}_n$, which concludes that case. For the induction step, let us assume the result for any locally complete derivation of length up to $n \geq 0$, and let $\hat{\instance}_{n+1}$ be obtained from a locally complete derivation of length $n+1$. We obtained $\hat{\instance}_{n+1}$ from $\hat{\instance}_n$ by a single locally complete rule application of rule $\erule_{n}$ through $\match_{n}$. Let $\instance = \instance_0, \ldots, \instance_{m}$ be a derivation such that $\instance_{m} \models \hat{\instance}_n$, and $\varphi_n$ be a homomorphism from $\hat{\instance}_n$ to $\instance_{m}$. Then $\erule_{n}$ is applicable to $\instance_{m}$ using $\varphi_n \circ \match_{n}$, creating $\instance_{m + 1}$. Let $\varphi_n^+$ be the extension of $\varphi_n$ such that each variable introduced by the application of $\erule_{n}$ to $\hat{\instance}_{n}$ through $\match_{n}$ is mapped to the corresponding variable in $\instance_{m+1}$. We show that there exist an extension $\instance_0, \ldots, \instance_{m+1}, \ldots, \instance_{m'}$ of the derivation $\instance_0, \ldots, \instance_{m+1}$ and a homomorphism  $\varphi_{n+1}$ from $\hat{\instance}_{n+1}$ to $\instance_{m'}$. 

First observe that $\varphi_n^+$ is a homomorphism from $\hat{\instance}'_n=\hat{\instance}_n \cup \match_{n}^{\mathrm{safe}}(\head{\erule_{n}})$ to $\instance_{m+1}$. 
By definition, $\hat{\instance}_{n+1} = \chase{\hat{\instance}'_n}{\ruleset}_{\mid \terms{\hat{\instance}'_n}}$. There is thus a sequence of rule applications $\{(\erule'_i,\match'_i)\}_{1 \leq i \leq l}$ starting from $\hat{\instance}'_n$ whose result entails $\hat{\instance}_{n+1}$; let us denote by $\hat{\instance'}_n^i$ the set of atoms obtained after the application of $(\erule'_i,\match'_i)$. We build by induction on $l$ a sequence of pairs $(\varphi^i,\instance_{m+1+i})$ such that $\varphi^i$ is a homomorphism from $\hat{\instance'}_n^i$ to  $\instance_{m+1+i}$: 
\begin{itemize}
 \item If the sequence of rule applications is empty, one can take $\varphi_n^+$ and $\instance_{m+1}$. 
 \item If we have built $\varphi^i$ and $\instance_{m+1+i}$ fulfilling the induction condition, then $\erule'_{i+1}$ is applicable to $\instance_{m+1+i}$ by $\varphi^i\circ\match'_{i+1}$, and let us denote by $\instance_{m+1+i+1}$ the result of that application. Then $\varphi^i$ can be extended to a homomorphism $\varphi^{i+1}$ from $\hat{\instance'}_n^{i+1}$ to $\instance_{m+1+i+1}$, by mapping the introduced fresh nulls in $\hat{\instance'}_n^{i+1}$ to the corresponding fresh nulls in $\instance_{m+1+i+1}$. 
\end{itemize}

We finally define $\instance_{m'}$ as $\instance_{m+l}$ and $\varphi_{n+1}$ as $\varphi^l$, and we have that $\instance_{m'} \models \hat{\instance}_{n+1}$, as $\varphi_{n+1}$ is a homomorphism from $\hat{\instance}_{n+1}$ to $\instance_{m'}$.
\end{proof}

\begin{proposition}
\label{corollary-classical-alternative}
For any KB $(\instance, \ruleset)$ and homomorphism-closed Boolean query $\query$, there is a (classical) $\ruleset$-derivation from $\instance$ resulting in $I_n$ with $I_n \models q$ if and only if there is a locally complete $\ruleset$-derivation from $\instance$ resulting in $\hat{I}_m$ with $\hat{I}_m \models \query$.
\end{proposition} 

\begin{proof}
 This is a direct consequence of Lemmas\ \ref{prop-alternative-complete} and \ref{prop-alternative-sound}.
\end{proof}

The previous correspondence between classical and locally complete derivations holds for arbitrary instances $\instance$ and rulesets  $\ruleset$. 
However, even $\hat{\instance}_0$ is not computable in the general case (whereas each $\hat{I}_i$ is computable in the case of guarded rules).  
Next, we want to simulate locally complete derivations under guarded rules by classical derivations under linear rules. 
To that end, we will finitely encode the infinite set of all guarded sets that can be defined on a KB vocabulary $\mathcal V = (\mathcal P, \mathcal C)$. Note that a guarded set itself is necessarily finite because only a finite set of atoms can be defined on a finite set of terms (those of the guard) and a finite set of predicates (as $\mathcal P$ is finite).  

To finitely encode all guarded sets, we first rename them in a canonical way. 
We consider a new set of \emph{canonical variables} $\mathcal X = \{\normVar_1, \ldots, \normVar_{2w}\}$, where $w$ is the maximal predicate arity in $\mathcal V$ (then, $2w$ is be the maximal number of variables in a guarded rule and this will also hold for the linear rules built by the reduction). 
The set $\mathcal X$  is linearly ordered according to $1 \ldots 2w$.
 The canonical renaming of an atom on $\mathcal V$ is obtained by substituting all its terms by canonical variables. We will need to consider atoms in which canonical variables from a subset $\mathcal{Y} \subseteq \mathcal{X}$ may already appear. In this case, its canonical renaming is obtained by substituting each other term by the next variable in $\mathcal{X} \setminus \mathcal Y$.

\begin{definition}[Canonical renaming]
\label{def-canonical-renaming}
Let $\atom$ be an atom that may contain both classical terms and canonical variables from $\mathcal X$. The \emph{canonical renaming}  of $\atom$ w.r.t $\mathcal{Y} \subseteq \mathcal{X}$ is a substitution $\canonRenaming{\atom,\mathcal{Y}}$ from $\terms{\atom}$ to $\mathcal X$ such that $\canonRenaming{\atom,\mathcal{Y}}(\term_i) = \term_i$ if $\term_i \in \mathcal X$, otherwise $\canonRenaming{\atom,\mathcal{Y}}(\term_i) = \normVar_j$ where $j$ is the smallest integer such that $\normVar_j \not \in \mathcal{Y}$,  $\normVar_j $ does not occur in $\atom$, and  $\normVar_j  \neq \canonRenaming{\atom,\mathcal{Y}}(\term_k)$ for all $k < i$. If $\mathcal{Y} = \emptyset$, we simply write $\canonRenaming{\atom}$.	
\end{definition}

Note that, for a guarded set $G$ with guard $\atom$ (we recall that such set is denoted by the pair $(G,\atom)$),  the canonical renaming of $\atom$ yields a renaming of all terms in $G$.

Given a vocabulary $\mathcal V = (\mathcal P, \mathcal C)$, we denote by $\mathcal{G}$ the (finite) set of all guarded sets of atoms with predicates in $\mathcal P$ and terms in the canonical set of variables $\mathcal X$. 
For any guarded set $(G,\atom)$ on $\mathcal V$, it holds that that $(\canonRenaming{\atom}(G),\canonRenaming{\atom}(\atom)) \in \mathcal{G}$. 
Furthermore, to each guarded set $(G,\alpha)$ we assign a \emph{complex predicate}, of the form $p_{{\canonRenaming{\atom}(G)}}$ with the same arity as $\atom$. As $\mathcal{G}$ is finite, so is the set of complex predicates.
A \emph{complex atom} is an atom with a complex predicate (and classical terms). The \emph{canonical atom}  associated with a guarded set $(G,\atom)$ is $\fun{can}(G,\atom)$ = $p_{\canonRenaming{\atom}(G)}(\terms{\atom})$.

In this way, an instance can be encoded by the canonical atoms associated with all the guarded sets of its atoms. Similarly, a guarded rule is  encoded by two canonical atoms respectively associated with its (guarded) body and head, which yields a linear rule. However, as illustrated by the next example, additional information will be needed to mimic guarded rule derivations with linear rule derivations.
\begin{example}
\label{ex-naive-translation}
Consider $\instance = \{p(a),r(a,b),r(b,c),q(c)\}$ and $\ruleset = \{r(x,y) \wedge q(y) \rightarrow q(x); p(x) \wedge q(x) \rightarrow \exists y\ s(x,y); q(x) \wedge s(x,y) \rightarrow h(y)\}$. The way we intend to encode the instance $\instance$ works as follows. Consider $r(a,b)$ for instance. It guards the set of atoms $\{p(a),r(a,b)\}$. As such, we would encode it by $p_{\{r(\normVar_1,\normVar_2),p(\normVar_1)\}}(a,b)$. Similarly, we would also encode the set $\{r(b,c),q(c)\}$ guarded by $r(b,c)$ through the atom $p_{\{r(\normVar_1,\normVar_2),q(\normVar_2)\}}(b,c)$. 

One could try to simulate the effect of $r(x,y) \wedge q(y) \rightarrow q(x)$ by linear rules of the shape $p_{\{r(\normVar_1,\normVar_2),q(\normVar_2)\}}(x,y) \rightarrow p_{\{r(\normVar_1,\normVar_2),q(\normVar_2), q(\normVar_1)\}}(x,y)$. The rule $r(x,y) \wedge q(y) \rightarrow q(x)$ is applicable on $\instance$ by mapping $x$ to $b$ and $y$ to $c$ producing $q(b)$; similarly, $p_{\{r(\normVar_1,\normVar_2),q(\normVar_2)\}}(x,y) \rightarrow p_{\{r(\normVar_1,\normVar_2),q(\normVar_2), q(\normVar_1)\}}(x,y)$ is applicable on the encoded instance by mapping $x$ to $b$ and $y$ to $c$, producing the atom $p_{\{r(\normVar_1,\normVar_2),q(\normVar_2), q(\normVar_1)\}}(b,c)$. However, the next application of the same rule, mapping $x$ to $a$ and $y$ to $b$, cannot be simulated in the translation, because $q(b)$ is not encoded in the atom $p_{\{r(\normVar_1,\normVar_2),p(\normVar_1)\}}(a,b)$, which prevents further rule applications.
\end{example}

As witnessed by Example \ref{ex-naive-translation}, starting from the guarded sets present in $\instance$ is not enough to mimic guarded rule derivations by linear rule derivations. A first idea would be to start from saturated guarded sets, as supported by the following proposition (whose proof is given in the appendix).

\begin{propositionrep}
 \label{prop-splitting-guarded-chase}
  Let $\ruleset$ be a set of guarded rules and $\instance$ be an instance. For each atom $\atom \in \instance$, we note $\instance^*_\atom = \chase{\instance}{\ruleset}_{\mid \terms{\atom}}$. Then:
  \begin{itemize}
    \item [$(\Rightarrow)$] For any $\ruleset$-derivation from $\instance$ to $\instance_n$, we can define for every $\atom \in \instance$ an $\ruleset$-derivation from $\instance^*_\atom$ to $\instance'_\atom$ such that $\bigcup_{\atom \in \instance} \instance'_\atom \models \instance_n$;
    \item [$(\Leftarrow)$] Let for any $\atom \in \instance$ an $\ruleset$-derivation from $\instance^*_\atom$ to $\instance'_\atom$;  there exists an $\ruleset$-derivation from $\instance$ to $\instance'$ such that $\instance' \models \bigcup_{\atom \in \instance} \instance'_\atom$.
  \end{itemize}
\end{propositionrep}

\begin{proof}
  
  $(\Leftarrow)$ Consider an $\ruleset$-derivation $D$ from some $\instance^*_\atom$ to $\instance'_\atom$. There is an $\mathcal R$-derivation $D_1$ from $\{\atom\}$ to some $\instance^1_\atom$ such that $\instance^*_\atom \subseteq \instance^1_\atom$, {\it i.e.,} $\instance^1_\atom = \instance^\ast_\atom \cup X$. Let us now build a derivation $D_2$ from $\instance^1_\atom$, applying rules using the same homomorphisms (unless already applied) in the same order as in $D$. The result of this derivation is $\instance'_\atom \cup X$. The derivation obtained from the concatenation of $D_1$ and $D_2$ is a derivation from $\{\atom\}$ (and by consequence $\instance$) to $\instance' = \instance'_\atom \cup X$ (and by consequence $\instance' = \instance'_\atom \cup X \cup F$), and thus $\instance' \models \instance'_\atom$.

We conclude by pointing out that the concatenation of all the derivations as build previously from $\instance$, for all atoms $\atom \in \instance$ result in some $\instance'$ containing  $\bigcup_{\atom \in \instance} \instance'_\atom$, and thus $\instance' \models \bigcup_{\atom \in \instance} \instance'_\atom$. 
  $(\Rightarrow)$ We show the converse entailment by induction on the length of a derivation associated with an initial portion of $\chase{\instance}{\ruleset}$: 
  given any derivation $D^i$ of length $i\geq 0$ from $\instance$, we build derivations $D^i_\atom$ (of length at most $i$) from $\instance^*_\atom$ for each $\atom \in \instance$, 
  such that $\bigcup_{\atom \in \instance} \fun{atoms}(D_\atom^i) \models \fun{atoms}(D^i)$. 

  More precisely, we prove the following: For all $i \geq 0$, for all derivation $D^i = \instance_0 (= \instance), \ldots, \instance_i$, 
  there are $h_i,\psi_i$, and derivations $D_\atom^i$ for each $\atom \in \instance$ such that:
  \begin{enumerate}
   \item $D_\atom^i$ is a derivation from $\instance^*_\atom$ 
   (of length less or equal to $i$) 
   \item $h_i$ is a homomorphism from $\instance_i$ to $\bigcup_{\atom \in \instance} \fun{atoms}(D_\atom^i)$;
   \item $\psi_i$ is a function that assigns to each atom $\beta \in \instance_i$ an atom $\atom \in \instance$ such that for any atom $\beta' \in \instance_i$ with $\terms{\beta'} \subseteq \terms{\beta}, h_i(\beta') \in \fun{atoms}(D^i_\atom)$.
   \end{enumerate}
  Note that the homomorphism $h_i$ (Condition 2) shows that $\bigcup_{\atom \in F} \fun{atoms}(D_\atom^i) \models \fun{atoms}(D^i)$.
  
  \medskip
  For $i = 0$, we define each $D_\atom^0$ as the empty derivation from $\instance^*_{\atom}$, $h_0$ as the identity on terms of $\instance$ and $\psi_0$ as the identity on atoms of $\instance$. These choices fulfill our induction assumption, because:
  \begin{enumerate}
   \item Empty derivations are derivations;
   \item As $D_\atom^0$ are empty derivations, it holds that $\fun{atoms}(D_\atom^0) = \instance^*_{\atom}$, hence $h_0$ is indeed a homomorphism from $\instance$ to $\bigcup_{\atom \in \instance} \fun{atoms}(D_\atom^i)$;
   \item Let $\beta' \in \instance$ such that $\terms{\beta'} \subseteq \terms{\beta}$. Then $\beta' \in \instance_{\beta}^* = \instance_{\psi_0(\beta)}^*$.
     \end{enumerate}
  
  For the induction step, let us assume that $\{D_\atom^i\}_{\atom \in \instance}$, $h_i$, $\psi_i$ are defined and satisfy the three conditions, and let $\match_{i+1}$ and $\erule_{i+1}$ be such that $\instance_{i+1}$ is obtained from $\instance_i$ by applying rule $\erule_{i+1}$ using $\match_{i+1}$. Let $g_{i+1}$ be the guard of $\erule_{i+1}$. As $\match_{i+1}(g_{i+1})$ belongs to $\instance_i$,  $\psi_i(\match_{i+1}(g_{i+1}))$ is defined. Let us denote this atom by $\atom_{r_i}$. Let $\atom_{i+1}$ be the atom created by the application of $\erule_{i+1}$ under $\match_{i+1}$. We assume w.l.o.g that $\atom_{i+1}$ does not belong to $\instance_i$.
  
  We define $\{D_\atom^{i+1}\}_{\atom \in \instance}$, $h_{i+1}$, and $\psi_{i+1}$ as follows:
  \begin{itemize}
   \item $\psi_{i+1}$ is the extension of $\psi_{i}$ defined by $\psi_{i+1}(\atom_{i+1}) = \atom_{r_i}$;
   \item $D_\atom^{i+1} = D_\atom^i$ if $\atom \not = \atom_{r_i}$. 
  For   $\atom_{r_i}$, we consider  the application of $\erule_{i+1}$ to $\fun{atoms}(D_{\atom_{r_i}}^{i})$ under $h_i\circ\match_{i+1}$.
  Note that  for each atom $b \in \fun{body}(\erule_{i+1})$, it holds that $h_i(\match_{i+1}(b)) \in \fun{atoms}(D_{\atom_{r_i}})$,
   as $\fun{terms}(b) \subseteq \fun{terms}(g_{i+1})$, hence $\fun{terms}(\match_{i+1}(b)) \subseteq \fun{terms}(\match_{i+1}(g_{i+1}))$, 
   and, since $\psi_i(\match_{i+1}(g_{i+1})) = \atom_{r_i}$, by Condition 3, we get $h_i(\match_{i+1}(b)) \in \fun{atoms}(D_{\atom_{r_i}})$. 
   Therefore, $h_i\circ\match_{i+1}$ defines a homomorphism from  $\body{\erule_{i+1}}$ to $\fun{atoms}(D_{\atom_{r_i}}^{i})$. 
    We denote by $\beta_{i+1}$ the atom created by this rule application. If $\beta_{i+1} \in \fun{atoms}(D_{\atom_{r_i}}^i)$
   we set $D_{\atom_{r_i}}^{i+1} = D_{\atom_{r_i}}^{i}$, otherwise $D_{\atom_{r_i}}^{i+1}$ is the 
   extension of $D_{\atom_{r_i}}^{i}$ by the application of $\erule_{i+1}$ under $h_i\circ\match_{i+1}$.  
   \item 
  For every null $e$ occurring in $\atom_{i+1}$, there is a null $e'$ occurring at the same position of $\beta_{i+1}$, and we extend $h_i$ to $h_{i+1}$ by setting $h_{i+1}(e) = e'$. 
  \end{itemize}
  
  We now check that the built objects still fulfill the induction statement:
  \begin{enumerate}
   \item By the induction assumption, for all $\atom \in \instance$, $D_\atom^i$ is a derivation from $\instance_\atom^*$ of length less or equal to $i$. Hence, for all 
    $ \atom \neq \atom_{r_i}$, $D_\atom^{i+1} = D_\atom^i$ is a derivation from $\instance_\atom^*$ (of length less of equal to $i$), and, for $\atom_{r_i}$, $D_{\atom_{r_i}}^{i+1}$ is the extension of $D_{\atom_{r_i}}^{i}$ by a rule application, hence of length less or equal to $i+1$; 
   \item $\atom_{i+1}$ is the only new atom in $\instance_{i+1}$. By construction of $h_{i+1}$ and $D_{\atom_{r_i}}^{i+1}$, $h_{i+1}(\atom_{i+1})$ belongs to $\fun{atoms}(D_{\atom_{r_i}}^{i+1})$, hence $h_{i+1}$ is a homomorphism from $\instance_{i+1}$ to $\bigcup_{\atom \in \instance} \fun{atoms}(D_\atom^{i+1})$;
  
   \item 
   We check that, for all atoms $\beta_1$ and $\beta_2$ in $\instance_{i+1}$ with $\terms{\beta_2} \subseteq \terms{\beta_1}$, $h_{i+1}(\beta_2) \in \fun{atoms}(D^{i+1}_{\psi_{i+1}(\beta_1)})$. 
   This holds for $\beta_1 = \beta_2 = \atom_{i+1}$, since by construction, $h_{i+1}(\atom_{i+1}) \in \fun{atoms}(D_{\atom_{r_i}}^{i+1})$ with $\atom_{r_i} = \psi_{i+1}(\atom_{i+1})$. 
    Let us consider the atoms $\beta' \in \instance_{i+1}$ with $\beta' \neq \atom_{i+1}$ such that  (1) $\terms{\beta'} \subseteq \terms{\atom_{i+1}}$
   or  (2) $\terms{\atom_{i+1}} \subseteq  \terms{\beta'}$. 
  For case (1), since  $\beta'$ was already in $\instance_{i}$,  $\terms{\beta'} \subseteq \terms{\match_{i+1}(g_{i+1})}$, 
   and by induction assumption $h_{i+1}(\beta') \in \fun{atoms}(D_{\psi_{i+1}(\match_{i+1}(g_{i+1}))}) = \fun{atoms}(D_{\psi_{i+1}(\atom_{i+1})})$. 
   For case (2), 
   either $\terms{\atom_{i+1}} \subseteq \terms{\instance}$, 
   then $\atom_{i+1} \in \instance^*_{\psi_{i+1}(\beta')} \subseteq \fun{atoms}(D^{i+1}_{\psi_{i+1}(\beta')})$, 
   or $\atom_{i+1}$ contains at least one null. Since $\terms{\atom_{i+1}} \subseteq  \terms{\beta'}$, this null was already present
   in $\match_{i+1}(g_{i+1})$. By the induction assumption, $\psi_i(\match_{i+1}(g_{i+1})) = \psi_i(\beta')$, hence $\psi_{i+1}(\beta') = \atom_{r_i}$.
  By construction of $h_{i+1}$ and $D_{\atom_{r_i}}^{i+1}$, $h_{i+1}(\atom_{i+1}) \in \fun{atoms}(D_{\atom_{r_i}}^{i+1})$. \qedhere

  \end{enumerate}
  \end{proof}

However, a similar problem of incompleteness would occur if the (extended) instance resulting from a rule application was not closed as well. That is why the notion of locally complete derivation considers at each step the chase of the resulting instance restricted to the terms of this instance.

We are now ready to define the reduction itself.
Given a KB $(\instance,\ruleset)$ on a vocabulary $\mathcal V$ (the `original vocabulary'), where $\ruleset$ is a set of guarded rules, we build $(\instance',\ruleset')$ where $\ruleset'$ is a set of linear rules, such that for any Boolean CRPQ $\query$ on $\mathcal V$, it holds that $\instance,\ruleset \models \query$ if and only if $\instance',\ruleset' \models \query$. The new instance $\instance'$ is the encoding of $\chase{\instance}{\ruleset}_{\mid \terms{\instance}}$ using complex predicates, while $\ruleset'$ contains two kinds of rules: \emph{reconstruction rules}, which allow to generate back atoms on $\mathcal V$ from atoms on complex predicates, and \emph{complex rules}, which simulate one step of a locally complete derivation. Let us start with the definition of $\instance'$ and its illustration.

\begin{definition}[Guarded Translation of $\instance$ ($\instance'$)]
\label{def-instance-construction}
The \emph{guarded translation} of $\instance$ w.r.t. $\ruleset$ is 
$\instance' = \{ \fun{can}(I^*_\atom, \atom) ~| ~\atom \in \instance \}$,
where $(I^*_\atom, \atom)$ denotes the guarded set $(\chase{I}{\ruleset}_{\mid \terms{\atom}}, \atom)$.
 \end{definition}

 \begin{example}
Consider $\instance$ and $\ruleset$ from Example \ref{ex-naive-translation}, and let $\atom_1 = r(b,c)$ and $\atom_2 = r(a,b)$. In the following, we underline the guard in a guarded set. 
  $(\instance_{\atom_1}^*,\atom_1) =  \{\underline{r(b,c)},q(b),q(c)\}$, hence $\fun{can}(\instance_{\atom_1}^*,\atom_1)
 = p_{\{\underline{r(\normVar_1,\normVar_2)},q(\normVar_1),q(\normVar_2)\}}(b,c)$. Similarly, $(\instance_{\atom_2}^*,\atom_2) = \{\underline{r(a,b)},p(a),q(a),q(b)\}$, hence $\fun{can}(\instance_{\atom_2}^*,\atom_2) = p_{\{\underline{r(X_1,X_2)},p(X_1),q(X_1),q(X_2)\}}(a,b)$.
 
 \end{example}

We now present the set of linear rules $\ruleset' = \ruleset_r \cup \ruleset_c$, which simulates the locally complete chase, starting with reconstruction rules $\ruleset_r$. 

\begin{definition}[Reconstruction rules ($\ruleset_r$)]
\label{def-reconstruction-rule}
 Let $(G,\atom) \in \mathcal{G}$. The set of \emph{reconstruction rules} associated with $(G,\atom)$ contains for each $\beta \in G$ the rule of the form $\fun{can}(G,\atom) \rightarrow \beta.$
 The set of all reconstruction rules (associated with $\mathcal{G}$) is denoted by $\ruleset_r$.
\end{definition}

Note that reconstruction rules do not contain any existentially quantified variable.

\begin{example}
 Consider the guarded set $\{\underline{r(\normVar_1,\normVar_2)},p(\normVar_2)\}$. There are two associated reconstruction rules, namely $p_{\{\underline{r(\normVar_1,\normVar_2)},p(\normVar_2)\}}(\normVar_1,\normVar_2) \rightarrow r(\normVar_1,\normVar_2)$ and $p_{\{\underline{r(\normVar_1,\normVar_2)},p(\normVar_2)\}}(\normVar_1,\normVar_2) \rightarrow p(\normVar_2)$.
\end{example}

Their role is to reconstruct the atoms on the original vocabulary that are encoded by complex predicates.

\begin{definition}[Expansion]
\label{def-expansion}
\sloppypar{
The \emph{expansion} of a complex atom $\atom$, denoted by $\expansion(\atom)$, is defined as $\chase{\atom}{\ruleset_r} \setminus \{\atom\}$. The expansion of a set of complex atoms is the union of the expansions of its atoms.}
\end{definition}

The next lemmas formalize the role of reconstruction rules.

\begin{lemma}
\label{lemma-reconstruction-rules}
Let $(G,\atom)$ be a guarded set of atoms. It holds that
$$\expansion(\fun{can}(G,\atom)) = G.$$
\end{lemma}

\begin{proof}
We prove both inclusions. Let us first prove that $G \subseteq \expansion(\fun{can}(G,\atom))$. Let $\beta \in G$. It holds that $\Phi_\atom(\beta) \in \Phi_\atom(G)$. Moreover $\fun{can}(\Phi_\atom(G),\Phi_\atom(\atom)) \rightarrow \Phi_\atom(\beta)$ belongs to $\ruleset_r$ as $(\Phi_\atom(G),\Phi_\atom(\atom)) \in \mathcal{G}$. 
This rule is applicable to $\fun{can}(G,\atom)$ and generates $\beta$. 

We now prove that $\expansion(\fun{can}(G,\atom)) \subseteq G$. 
Any reconstruction rule applicable to $\fun{can}(G,\atom)$ is of the form $\fun{can}(\Phi_\atom(G),\Phi_\atom(\atom)) \rightarrow \Phi_\atom(\gamma)$ for some $\gamma \in G$. Applying such a rule to $\fun{can}(G,\atom)$ generates $\beta$ only if   $\beta \in G$.
\end{proof}

Using Lemma \ref{lemma-reconstruction-rules}, we show that the expansion of $\instance'$ is the first instance of a locally complete $\ruleset$-derivation of $\instance$. 

\begin{lemma}
\label{lemma-expansion-instance}
It holds that:
 $$\expansion(\instance') = \chase{\instance}{\ruleset}_{\mid \terms{\instance}}.$$
\end{lemma}

\begin{proof}
\sloppypar{
By definition, $\instance' = \cup_{\atom \in \instance} \fun{can}(\instance^*_\atom,\atom)$.
Hence, $\expansion(\instance') = \cup_{\atom \in \instance} \expansion(\fun{can}(\instance^*_\atom,\atom))$. 
By applying Lemma \ref{lemma-reconstruction-rules}, it holds that $\expansion(\fun{can}(\instance^*_\atom,\atom)) = \instance^*_\atom$, hence 
$\expansion(\instance') = \cup_{\atom \in \instance} \instance^*_\atom$. 
As $\cup_{\atom \in \instance} \instance^*_\atom = \chase{\instance}{\ruleset}_{\mid \terms{\instance}}$ by Proposition \ref{prop-splitting-guarded-chase}, this concludes the proof.}
\end{proof}

We now focus on complex rules. We recall that their role is to simulate one step of a locally complete derivation.

\begin{definition}[Complex rules ($\ruleset_c$)]
\label{def-complex-rule}
Let $(G,\atom) \in \mathcal{G}$ and $\erule \in \ruleset$ be applicable to $G$ by $\match$.  
Let $\atom' =  \match^{\mathrm{safe}}(\head{\erule})$ and  $\canonRenaming{\atom',\terms{\atom}}(\alpha')$ be the canonical renaming
of  $\atom'$ w.r.t.\ $\terms{\atom}$. The \emph{complex rule} associated with $G, \erule$ and $\match$ is:
$$\fun{can}(G,\atom) \rightarrow \fun{can}(G',\canonRenaming{\atom',\terms{\atom}}(\atom'))$$
where $G' = \canonRenaming{\atom',\terms{\atom}}(\chase{G \cup \{\atom'\}}{\ruleset}_{\mid \terms{\atom'}})$.
The set of  complex rules (associated with $\mathcal{G}$) is denoted by $\ruleset_c$.
\end{definition}

\begin{example}
 We continue Example \ref{ex-naive-translation}.  Consider the guarded set  $G = \{\underline{r(\normVar_1,\normVar_2)},p(\normVar_1),q(\normVar_1),q(\normVar_2)\}$. 
 The rule $p(x) \wedge q(x) \rightarrow \exists y ~s(x,y)$ is applicable to $G$ through $\match = \{x \mapsto  \normVar_1\}$. 
 Hence, we build the complex rule $p_{\underline{r(\normVar_1,\normVar_2)},p(\normVar_1),q(\normVar_1),q(\normVar_2)}(\normVar_1,\normVar_2) \rightarrow p_{G'}(\normVar_1,\normVar_3)$, where $G' = \{\underline{s(\normVar_1,\normVar_3)},p(\normVar_1),q(\normVar_1),h(\normVar_3)\}$. Note that $G'$ contains $h(\normVar_3)$, which is not justified by the application of $p(x) \wedge q(x) \rightarrow \exists y ~s(x,y)$ itself, but by the future application of $q(x)\wedge s(x,y) \rightarrow h(y)$.
 \end{example}

We now focus on showing the correspondence between a (classical) $\ruleset_c$-derivation and a locally complete $\ruleset$-derivation. Lemma \ref{lemma-complex-alternative} specifies the correspondence at the level of a single rule application.

\begin{lemma}
\label{lemma-complex-alternative} 
 Let $(G,\atom = r(\mathbf{t}))$ be a guarded set such that $\chase{G}{\ruleset}_{\mid \terms{G}} = G$. The following statements hold:
 \begin{itemize}
  \item Let $\erule \in \ruleset$ be applicable to $G$ by $\match$, and let $\hat{G}$ be obtained by the corresponding locally complete derivation step.
  There exists a complex rule $\erule'$ applicable to $p_{\Phi_\atom(G)}(\mathbf{t})$ by $\match'$ such that there exists a homomorphism $\varphi$ from $\hat{G}$ to $\expansion(p_{\Phi_\atom(G)}(\mathbf{t}) \cup \match'^{\mathrm{safe}}(\head{\erule'}))$ with $\varphi$ being the identity on $\terms{G}$.
   
  \item Let $\erule' \in \ruleset_c$ be applicable to $G' = p_{\Phi_\atom(G)}(\mathbf{t})$ by $\match'$, and let $F' = G' \cup \match'^{\mathrm{safe}}(\head{\erule'})$. There exists $\erule \in \ruleset$ applicable to $G$, generating $\hat{G}$ by one step of locally complete derivation, such that there exists a homomorphism $\varphi$ from $\expansion(F')$ to $\hat{G}$ with $\varphi$ being the identity on $\terms{G}$.
 \end{itemize}

\end{lemma}

\begin{proof} 
 (1) Let $\erule$ be applicable to $G$ by $\match$, generating $\hat{\alpha}$. Then $\erule$ is applicable to $\canonRenaming{\atom}(G)$ by $\match^*=\canonRenaming{\atom}\circ\match$. By the definition of complex rules, there is a rule $\erule' = \fun{can}(\canonRenaming{\atom}(G),\canonRenaming{\atom}(\atom)) \rightarrow \fun{can}(G^{*},\canonRenaming{\atom^{*},\terms{\canonRenaming{\atom}(\atom)}}(\atom^{*}))$, where we define $\atom^{*} = {\match^*}^{\mathrm{safe}}(\head{\erule})$ and $G^{*} = \canonRenaming{\atom^{*},\terms{\canonRenaming{\atom}(\atom)}}(\chase{\canonRenaming{\atom}(G) \cup\{\atom^{*}\}}{\ruleset}_{\mid \terms{\atom^{*}}})$. The rule $\rho'$ is applicable  to $\fun{can}(G,\atom) = p_{\canonRenaming{\atom}(G)}(\mathbf{t})$ by $\match' =\canonRenaming{\atom}^{-1}$ and yields the atom $\atom' = {\match'}^{\mathrm{safe}}(\head{\erule'})$.
Let $\hat{G}$ be the result of performing a single locally complete derivation step to $G$ using the rule $\erule$ and $\match$, i.e., $\hat{G} = \chase{G \cup \hat{\atom}}{\ruleset}_{\mid \terms{G \cup \{\hat{\atom}\}}}$.  
Define $\varphi$ as the injective mapping that is the identity on the terms of $G$ and that associates with each fresh null in $\hat{\alpha}$ 
the corresponding fresh null from $\alpha'$ that was introduced by applying $\erule'$ to $p_{\canonRenaming{\atom}(G)}(\mathbf{t})$. 

We claim that $\varphi$ is a homomorphism from $\hat{G}$ to $\expansion(\fun{can}(G,\atom) \cup \atom')$. Indeed, if $\beta \in \hat{G}$ and $\terms{\beta} \subseteq \terms{G}$, then by assumption, $\beta \in G$, and thus belongs to $\expansion(\fun{can}(G,\atom))$ by~Lemma \ref{lemma-reconstruction-rules}. Otherwise, we have $\terms{\beta} \subseteq \terms{\hat{\alpha}}$ and $\beta \in \chase{G\cup\{\hat{\atom}\}}{\ruleset}_{\mid \terms{\hat{\atom}}}$. We then remark that 
$\expansion(\atom') = \chase{G\cup\{\varphi^{-1}(\hat{\atom})\}}{\ruleset}_{\mid \terms{\varphi^{-1}(\hat{\atom})}}$, from which we obtain $\varphi(\beta) \in \expansion(\atom')$. 
 (2) 
Suppose $\erule' \in \ruleset_c$ is applicable to $G' = p_{\Phi_\atom(G)}(\mathbf{t})= \fun{can}(G,\atom)$ by $\match'$, and let $F' = G' \cup \{\atom'\}$, where $\atom' =\match'^{\mathrm{safe}}(\head{\erule'})$.  As $\erule'$ is a complex rule, it has been created because some $(G_c,\atom_c)$ belongs to $\mathcal{G}$, with $\canonRenaming{\atom_c}(G_c) = \canonRenaming{\atom}(G)$, and there is some (original) rule $\erule$ that is applicable to $G_c$ 
through some $\match^*$, such that $\head{\erule'} = \fun{can}(G^*,\canonRenaming{\atom^*,\terms{\atom_c}}(\atom^*))$, where $\atom^*= {\match^*}^{\mathrm{safe}}(\head{\erule})$ and $G^*=\canonRenaming{\atom^*,\terms{\atom_c}}(\chase{G_c\cup\{\atom^*\}}{\ruleset}_{\mid \terms{\atom^*}})$.
Observe that $G$ and $G_c$ are isomorphic, and let $\Psi$ be an isomorphism from $G_c$ to $G$.
Then $\erule$ is applicable to $G$ via $\pi=\Psi\circ\match^*$, and the locally complete derivation step
yields $\hat G= \chase{G\cup\{\hat{\atom}\}}{\ruleset}_{\mid \terms{G \cup\{\hat{\atom}\}}})$, with 
$\hat{\atom}= \pi^{\mathrm{safe}}(\head{\erule})$. 

Now let $\varphi$ be the (injective) extension of the identity on $\terms{G}$ that maps each fresh null in $\atom'$ introduced during the application of $\erule'$ to $G'$ to the corresponding null in $\hat{\atom}$ introduced by the application of $\erule$ to $G$. 
An atom $\beta$ in $\expansion(F')$ belongs either to $\expansion(\fun{can}(G,\atom))$ or to $\expansion(\atom')$. By Lemma ~\ref{lemma-reconstruction-rules}, we have  $\expansion(\fun{can}(G,\atom)) = 
G$. 
It follows that if $\beta \in \expansion(\fun{can}(G,\atom))$, then $\varphi(\beta) = \beta \in G \subseteq \hat{G}$. 
 Otherwise, $\beta \in \expansion(\atom')$, so $\beta \in  \chase{G\cup\{\varphi^{-1}(\hat{\atom})\}}{\ruleset}_{\mid \terms{\varphi^{-1}(\hat{\atom})}}$. 
From this we obtain $\varphi(\beta) \in  \chase{G\cup\{\hat{\atom}\}}{\ruleset}_{\mid \terms{\hat{\atom}}} \subseteq \hat{G}$.\qedhere
\end{proof}

Repeatedly using the first bullet point of Lemma \ref{lemma-complex-alternative}, Proposition \ref{prop-reduction-complete} proves that anything that is entailed by a locally complete $\ruleset$-derivation of $\instance$ is also entailed by the expansion of an $\ruleset_c$-derivation of $\instance'$---the guarded translation of  $\instance$.

\begin{proposition}
\label{prop-reduction-complete}
Let $\ruleset$ be a set of guarded rules and $\instance$ be an instance. Let $\hat{\instance_0} ,\ldots, \hat{\instance}_n$  
be a locally complete $\ruleset$-derivation from $\instance$. Let $\instance'$ and $\ruleset_c$ be defined as above. 
There is a classical $\ruleset_c$-derivation $\instance' = \instance'_0,\ldots,\instance'_n$ such that for any $i \in \{1,\ldots,n\}$, there exists a homomorphism $\varphi_i$ from $\instance_i$ to $\expansion(\instance'_i)$, such that for all guarded sets $(G,\atom)$ with $G \subseteq \instance_i$, there is an atom $\beta \in \instance'_i$ with $\varphi_i(G) \subseteq \expansion(\beta)$.
\end{proposition}

\begin{proof}
We prove the result by induction on the length of the locally complete derivation. 
\begin{itemize}
 \item If the locally complete derivation is of length $0$, then $I_0 = \chase{\instance}{\ruleset}_{\mid\terms{\instance}} = \expansion(\instance')$ by Lemma \ref{lemma-expansion-instance}, and one can take the identity for $\varphi_0$. Indeed, if $(G,\atom)$ is a guarded set with $G \subseteq \instance_0$, we let $\beta= \fun{can}(\instance^*_\atom) \in \instance'$ and observe that $G \subseteq \chase{\instance}{\ruleset}_{\mid \terms{\atom}} = \expansion(\beta)$. 
 \item Otherwise, let us assume that there is a homomorphism  $\varphi_i$ from $\instance_i$ to $\expansion(\instance'_i)$ such that for all guarded sets $G \subseteq \instance_i$, there exists $\beta \in \instance'_i$ such that $\varphi_i(G) \subseteq \expansion(\beta)$. Let $\erule_i$ be the rule applied to $\instance_i$ through $\match_i$ to generate $\instance_{i+1}$, creating an atom $\atom_{i+1}$. As $\match_i(\body{\erule_i})$ is a guarded set in $\instance_i$, by the induction assumption, there is $\beta_i \in \instance'_i$ such that $\varphi_i(\match_i(\body{\erule_i})) \subseteq \expansion(\beta_i)$. Let us notice that $\chase{\expansion(\beta_i)}{\ruleset}_{\mid \terms{\expansion(\beta_i)}} = \expansion(\beta_i)$, as it holds for any initial $\beta_i$ and for any atom added by definition of complex rules. 
 Also observe that {$\beta_i = \fun{can}(\expansion(\beta_i))$}.
We can thus apply Lemma~\ref{lemma-complex-alternative} to infer that 
 there exists $\erule_i'$ applicable to $\beta_i$, whose application creates $\atom'_{i+1}$, such that 
 there is a homomorphism from $\chase{\expansion(\beta_i) \cup \{\atom_{i+1}\}}{\ruleset}_{\mid \terms{\alpha_{i+1}}} $ to $\expansion(\beta_i \cup \alpha_{i+1}')$ that is the identity on $\terms{\expansion(\beta_i)}$
 We define $\varphi_{i+1}$ as the mapping extending $\varphi_i$ by mapping a fresh null introduced in $\atom_{i+1}$ to the corresponding fresh null introduced in $\atom'_{i+1}$. 
 For any guarded set $G$ of $I_{i+1}$, one of the two following case holds:
 \begin{itemize}
 \item $\terms{G} \subseteq \terms{I_i}$, and so by induction assumption, there exists $\beta_i \in I_i'$ for which $\varphi_{i+1}(G) = \varphi_i(G) \subseteq \expansion(\beta_i)$;
 \item $\terms{G} \subseteq \terms{\atom_{i+1}}$, in which case $\atom'_{i+1} \in \instance_{i+1}'$ is such that $\varphi_{i+1}(G) \subseteq \expansion(\atom'_{i+1})$.
 \end{itemize}
 In particular, this implies that $\varphi_{i+1}$ is a homomorphism from $\instance_{i+1}$ to $\expansion(\instance_{i+1}')$. \qedhere
\end{itemize}
\end{proof}
Conversely, Proposition \ref{prop-reduction-sound} uses the second bullet point of Lemma \ref{lemma-complex-alternative} to show that anything that is entailed by the expansion of the result of a classical $\ruleset_c$-derivation of $\instance'$ is also entailed by the result of a locally complete $\ruleset$-derivation of $\instance$.

\begin{proposition}
\label{prop-reduction-sound}
Let $\ruleset$ be a set of guarded rules and $\instance$ be an instance. Let $\instance'$ and $\ruleset_c$ be defined as above. For any classical $\ruleset_c$-derivation $I' = I'_0,\ldots,I'_n$, there exists a locally complete $\ruleset$-derivation $\hat{I}_0 = \chase{I}{\ruleset}_{\mid \terms{I}},\ldots,\hat{I}_{n}$ such that for any $i \in \{0,\ldots,n\}$, there is a homomorphism $\varphi_i$ from  $\expansion(I'_i)$ to  $\hat{I}_{i}$.
\end{proposition}

\begin{proof} 
We show the result by induction on the length of the $\ruleset_c$-derivation.
\begin{itemize}
\item If the $\ruleset_c$-derivation is of length $0$, then $I'_n = I'$, and by Lemma \ref{lemma-expansion-instance}, $\expansion(I') = \chase{I}{\ruleset}_{\mid\terms{I}}$. The derivation $\hat I_0 = \chase{I}{\ruleset}_{\mid \terms{I}}$ is a locally complete $\ruleset$-derivation of $I$ of length $0$, and one can take $\varphi_0$ as the identity;
\item Let us assume that there exists a homomorphism $\varphi_i$ from $\expansion(I'_i)$ to $\hat{I}_{i}$. Let $\erule_i \in \ruleset_c$ be the rule applicable to $I'_i$ by $\match_i$ on $\atom$ that yields $I'_{i+1}= I'_i \cup \{\atom'\}$. By Lemma~\ref{lemma-complex-alternative}, there exists $\erule \in \ruleset$ that is applicable to $\expansion(\atom)$ by $\match'$ 
such that there is a homomorphism $h$ from $\expansion(\atom \cup \atom')$ to $\chase{H}{\ruleset}_{\mid \terms{H}}$ with $H=\expansion(\atom) \cup \match'^{\mathrm{safe}}(\head{\erule})$ that is  the identity on $\terms{\expansion(\atom)}$. 
It follows in particular that $\erule$ is applicable to $\varphi_i(\expansion(\atom)) \subseteq \hat{I}_{i}$ using $\match''=\varphi_i \circ \pi'$. 
We can thus let
 $\hat{I}_{i+1}=  \chase{\hat{I}_i \cup \match''^{\mathrm{safe}}(\head{\erule})}{\ruleset}_{\mid \terms{\hat{I}_i \cup \match''^{\mathrm{safe}}(\head{\erule})}}$ be the result of applying the locally complete derivation step associated with $\erule$ and $\match''$ to $\hat{I}_{i}$. 
 We observe that there is a homomorphism $g$ from 
 {$\chase{H}{\ruleset}_{\mid \terms{H}}$} to $\hat{I}_{i+1}$ defined by setting 
$g(t)=\varphi_i(t)$ for all $t \in \terms{\expansion(\atom)}$ and for all $t \in 
\terms{\match'^{\mathrm{safe}}(\head{\erule})} \setminus \terms{\expansion(\atom)}$,
letting $g(t)$ be the corresponding fresh term in $\match''^{\mathrm{safe}}(\head{\erule})$. 
The desired homomorphism $\varphi_{i+1}$ from $\expansion(I'_{i+1})$ to $\hat{I}_{i+1}$ is obtained by 
setting $\varphi_{i+1}(t) = \varphi_i(t)$ for all $t \in \terms{\expansion(I'_i)}$ and  
$\varphi_{i+1}(t) = {g \circ h(t)}$ for all $t \in \terms{\expansion(I'_{i+1})} \setminus \terms{\expansion(I'_{i})}\subseteq \terms{\expansion(\atom')} \setminus \terms{\expansion(\atom)}$. Indeed, this is a consequence of the facts that (i) $\expansion(I'_{i+1}) = \expansion(I'_i) \cup \expansion(\atom')$, (ii) 
$\varphi_i$ is a homomorphism from $\expansion(I'_i)$ to $\hat{I}_{i} \subseteq \hat{I}_{i+1}$, (iii) $g \circ h$ is a homomorphism from $\expansion(\atom \cup \atom')$
 to $\hat{I}_{i+1}$,  (iv) $\varphi_i$ and $g \circ h$ agree on their common domain, and (v) $\terms{\expansion(\atom')} \cap \terms{\expansion(I'_i)} \subseteq \terms{\expansion(\atom)} $. \qedhere
\end{itemize}
\end{proof}

The preceding proposition together with Proposition\ \ref{corollary-classical-alternative} 
directly imply the following corollary.

\begin{corollary}
\label{corollary-alternative-complex}
 For any KB $(\instance,\ruleset)$ and homomorphism-closed Boolean query $\query$ on the original vocabulary, there is a locally complete $\ruleset$-derivation from $\instance$ resulting in $\hat{\instance}_m$ such that $\hat{\instance}_m \models \query$ iff there exists a classical $(\ruleset_c \cup \ruleset_r)$-derivation from $\instance'$ resulting in $\instance'_n$ with $\instance'_n \models \query$.
\end{corollary}

We finally define $\ruleset' = \ruleset_r \cup \ruleset_c$.
Theorem \ref{reduction-guarded-linear}, which states the correction of our reduction, is a direct consequence of Proposition\ \ref{corollary-classical-alternative} and Corollary \ref{corollary-alternative-complex}.

\begin{theorem}\label{reduction-guarded-linear}
Let $q$ be a Boolean CRPQ on the original vocabulary. Then $\chase{\instance}{\ruleset} \models q$ if and only if $\chase{\instance'}{\ruleset'} \models q$. 
\end{theorem}

It remains to study the computational resources that are required to compute this reduction.

\begin{proposition}\label{prop-guarded-complexity}
 Given $(\instance,\ruleset)$, $\instance'$ and $\ruleset'$ as defined above can be computed 
  in \textsc{2ExpTime} in combined complexity, \textsc{ExpTime}  in combined complexity with bounded-predicate arity, and in \textsc{PTime} in data complexity. The number of types (Definition \ref{def-type}) whose predicate appears in $\ruleset'$ is at most a double exponential in $\instance$ and $\ruleset$ (exponential when the arity is bounded, and constant \emph{w.r.t.} the data).
\end{proposition}
  
 \begin{proof}
 Let $\maxArity$ be the maximum arity of a predicate appearing in $\ruleset$ or $\instance$, $p$ be the number of such predicates, and $n$ be the number of atoms of $\instance$, and $r$ be the number of rules.

Let $\instance$ be an instance and $\ruleset$ be a set of guarded rules. We analyze the size of the instance $\instance'$ and ruleset $\ruleset'$:	
\begin{itemize}
 \item $\instance'$ contains $n$ atoms;
 \item every guarded set in $\mathcal{G}$ can be represented in single exponential size (polynomial when the arity is bounded and in data complexity): there are $p(2\maxArity)^\maxArity$ atoms that can be built using the $2\maxArity$ canonical variables, and a predicate is described by a subset of such atoms (plus the selected guard atom);
 the same bounds apply to space required to store a complex predicate, as each complex predicate corresponds to a guarded set;

 \item there is a double exponential number of complex predicates (exponential when the arity is bounded and constant in data complexity);
 \item for each complex predicate, there is one reconstruction rule per atom in the guarded set of the predicate, hence at most $p(2\maxArity)^\maxArity$ such rules per predicate;
 \item for each complex predicate, there is one complex rule per homomorphism of a rule body in $\ruleset$ into the guarded set of atoms defining the predicate; in fact, it only matters where the guard atom is mapped, so we have at most $rp(2\maxArity)^\maxArity$ complex rules per new predicate; 
 \item there is a double exponential number of types (exponential when the arity is bounded, constant in data complexity): for each complex predicate, the number of types with this predicate is equal to the number of partitions of the predicate's arguments. 
\end{itemize}

To build the instance $\instance'$, we first compute for each atom $\atom \in I$ the set $\chase{I}{\ruleset}_{\mid \terms{\atom}}$.  This can be done by making at most $np\maxArity^\maxArity$ calls to a \textsc{2-ExpTime} oracle for query answering under guarded rules. This is a polynomial number of calls to an \textsc{ExpTime} oracle when the arity is bounded, and a polynomial number of calls to a \textsc{PTime} oracle in data complexity). The new instance $I'$ can be constructed in polynomial time from the saturated atoms. 

Building the set of reconstruction rules can be done in polynomial time with respect to the number of guarded sets. To build the set of complex rules, for each guarded set and each original rule $\rho$, we proceed as follows: apply the reconstruction rules, apply the original rule, saturate the resulting atom
with respect to the original rules, and create the corresponding complex rule. The preceding operations can be performed by making an exponential number of calls to a query answering oracle for guarded rules, and thus, this step can be done in double exponential time. When the arity is bounded, this makes a polynomial number of calls to an \textsc{ExpTime} oracle, hence is doable in \textsc{ExpTime}. In data complexity, we have a constant number of calls to a \textsc{PTime} oracle, hence is doable in \textsc{PTime}.  
\end{proof}

The next theorem follows from the provided reduction and Algorithm \ref{algo-rpq-answering}. Let us remark that the reduction to linear rules not being polynomial, we cannot directly apply Theorem \ref{thm-linear-rpq-answering} to obtain the desired upper bounds. However, a careful analysis of Algorithm \ref{algo-rpq-answering} allows us to show that it actually runs polynomially in the number of types built from the predicates in $\mathcal{V}$. The proof details  
are provided in the appendix.

\begin{theoremrep}
\label{thm-guarded-complexity}
  CRPQ answering under guarded rules is \textsc{2ExpTime}-complete in combined complexity, \textsc{ExpTime}-complete in combined complexity with bounded-predicate arity, and \textsc{PTime}-complete in data complexity.
\end{theoremrep}

{We believe that the result for combined complexity in the bounded-arity case still holds for arbitrary guarded rules with multiple head atoms (translated into atomic-head rules), the key argument being again that the number of types to be considered for a fresh predicate remains polynomial. However, the proof of that claim would require updating  the definitions of $\ruleset_c$ and  $\ruleset_r$, intuitively to consider only `relevant' guarded sets, and to revise accordingly all the subsequent statements and their proof. This would significantly  complexify the arguments and obfuscate the main ideas of the reduction, hence we decided not to consider this extension. }

\begin{proof} 
 The lower bounds come from CQ answering under guarded rules. 
 For the upper bound, we use the reduction of CRPQ answering over guarded rules to CRPQ answering over linear rules, that is formalized in Theorem \ref{reduction-guarded-linear}. 
We observe that building the set of rules $\ruleset'$ takes double exponential time, so simply running the algorithm for linear rules over the constructed KB $(\instance', \ruleset')$ does not yield a procedure running in \textsc{2ExpTime} for guarded rules. Let us revisit the complexity of the algorithm for linear rules, starting with the RPQ answering procedure from Section~\ref{sec-rpq-linear-upper}: 
\begin{itemize}
 \item for RPQ answering under linear rules, the first step is the creation of a \texttt{Loop} table, which can be done in time polynomial with respect to the number of types. Combining this observation with the bounds on the number of types from Proposition \ref{prop-guarded-complexity}, we can see that this step will run in 
double exponential time w.r.t. the input (resp.\ exponential time when the arity is bounded, polynomial time in data complexity).
 \item once the Loop table has been computed, it remains to apply Algorithm \ref{algo-rpq-answering}:
   \begin{itemize}
    \item the number of iterations of the while loop is independent of the rules, and polynomial in the number of terms of the instance, which is unchanged by our transformation;
    \item at each iteration of the while loop, we guess a constant from the instance, a state from the NFA representing the RPQ, and either a transition in the NFA or a type. There are polynomially many constants, states, and transitions. %
 It follows that the overall number of different possible guesses at a given iteration is double exponential in general case (exponential for the bounded-arity case, polynomial for data complexity), and the guess can be stored in exponential space (polynomial space for bounded-arity or data complexity);
   \item once the preceding guesses have been made, the algorithm either checks whether the guessed type occurs in the Loop table, or performs an entailment check;
    \item checking that a given type belong to the correct cell of the Loop table can be done in polynomial time with respect to the number of types (double exponential in general, single exponential when the arity is bounded, polynomial in data complexity);
    \item entailment checking can be performed by a non-deterministic procedure whose space is polynomially bounded w.r.t.\ the maximum size of a type, 
    which is single exponential in the size of the original input (polynomial when the arity is bounded, constant in data complexity).
   \end{itemize}
\end{itemize}
Putting the above together, we obtain that RPQ answering under the set of linear rules that is output by our reduction can be done in double exponential time with respect to the initial parameters (and in \textsc{ExpTime} when arity is bounded, and in \textsc{PTime} in data complexity).

We next consider the procedure for CRPQ answering:
\begin{itemize}
 \item the size (\emph{i.e.}, the number of nodes) of the largest proof scheme to be considered is independent of both rules and data;
 \item the validity of a proof scheme can be checked as in the proof of Proposition \ref{validity-exp}:
 \begin{itemize}
  \item each root belonging to $\instance'$ 
  : as $\instance'$ has already been computed, this check can be performed in exponential time, polynomial in data complexity and with bounded arity, as the predicate names can be of exponential size with respect to the arity; 
    \item computing the set of transitions: doable in double exponential time, exponential time when the arity is bounded, polynomial in data complexity, as this is can be done through a quadratic number of calls to RPQ answering under the modified linear ruleset;
  \item selection of atoms, rules: doable in exponential time, polynomial in data complexity and when the arity is bounded, as predicate names can be of exponential size in the arity;
  \item rule application: exponential time, polynomial in data complexity and when the arity is bounded, same reason
  \item the while loop of \textbf{Step 2} of the query answering algorithm is iterated at most \maxsteps~times, which is a simple exponential (polynomial in data complexity) as per the proof of Proposition \ref{validity-exp}.\qedhere
 \end{itemize}
\end{itemize}
\end{proof}

\section{Related Work}
\label{sec-related-work}


\medskip
In the introduction, we gave an overview of related work on answering navigational queries with ontologies. We observed that  
such queries had been extensively studied within various description logic fragments, but hardly at all for existential rules. 
 We now take a closer look at related work more specifically linked to our contributions, whether regarding the techniques used or the complexity results obtained. We recall that we use the notation (C)RPQ for \emph{two-way} queries, whereas these are also denoted by 2(C)RPQ in other works that reserve the term (C)RPQ to the restricted form without inverse predicates (called one-way).

\smallskip
Concerning relevant description logics, we focus on the DL-Lite \citep{DBLP:journals/jar/CalvaneseGLLR07} and $\mathcal{EL}$ \citep{DBLP:conf/ijcai/BaaderBL05} families, which are well-known fragments related to linear and guarded existential rules, respectively. 
The most studied member of the DL-Lite family is the dialect DL-Lite$_\mathcal{R}$, underpinning OWL 2 QL, 
whose (positive) axioms can be seen as linear rules of the form $\atom_1 \rightarrow \atom_2$; as usual in DLs, each $\atom_i$ is a unary atom or a binary atom with distinct variables.  
Hence, linear rules strictly generalize DL-Lite$_\mathcal{R}$ by unrestricted predicate arity and co-occurrences of variables in an atom. 
In $\mathcal{EL}$, axioms can be seen as rules of the form $B[x] \rightarrow a(x)$ or $B[x] \rightarrow \exists z~r(x,z) \land a(z)$ where $B[x] $ is a conjunction of unary atoms on $x$ or a conjunction of the form $r'(x,y) \land a'(y)$. $\mathcal{ELH}$ adds role inclusions, i.e., rules of the form $r'(x,y) \rightarrow r(x,y)$ and $\mathcal{ELHI}$ moreover allows for inverse roles, i.e., atoms of the form $r(y,x)$. 
Note that $\mathcal{ELHI}$ subsumes (positive) DL-Lite$_\mathcal{R}$. Guarded rules strictly generalize $\mathcal {ELHI}$ axioms.   

We mentioned in Section \ref{sec-rpq-linear-upper} that our algorithm for RPQ answering under linear rules (Algorithm \ref{algo-rpq-answering}) was inspired by a related algorithm for DL-Lite ontologies. More specifically, we extended the algorithm presented in \citep{DBLP:journals/jair/BienvenuOS15}
for DL-Lite$_\mathcal{R}$ to take higher-arity predicates and co-occurrences of variables into account. Technically, the main differences are the following. 
In DL-Lite$_\mathcal{R}$, each detour of a path of terms to the anonymous part of the chase goes from a constant and comes back to the same constant, i.e., for any path  $\path = (d_0,\ldots,d_n)$ in the chase, where only $d_0$ and $d_n$ are constants, it holds that $d_0 = d_n$. This is not true anymore for linear rules. Instead, we observe that for any such path $\path$, there is an atom $\atom \in \instance$ in which $d_0$ and $d_n$ both occur, and such that the chase of $\atom$ contains a path isomorphic to $\path$ (see our Proposition \ref{ref-prop-atom-loop}). Hence, we keep track of ``loops'' within atoms instead of ``loops'' on constants. A less important extension relates to the notion of atom type, the idea being that atoms with the same type behave similarly regarding rule applications. In DL-Lite, the type of an atom is simply given by its predicate. For linear rules, we need to take co-occurrences of terms into account, hence a more general definition of  type (Definition \ref{def-type}). Modulo these extensions, our Algorithm \ref{algo-rpq-answering} is quite close to the algorithm for DL-Lite$_\mathcal{R}$ from \citep{DBLP:journals/jair/BienvenuOS15}. 

In Table \ref{table-dl-complexities} we indicate the data and combined complexities of (C)RPQ answering for the DLs mentioned above, based on \citep{DBLP:journals/jair/BienvenuOS15}. To facilitate comparison with our results, we recall the complexities obtained for linear and guarded existential rules; we consider here combined complexity with bounded predicate arity, as DL predicates are at most binary. It is interesting to observe that the generalizations from DL-Lite$_\mathcal{R}$ to linear and from  $\mathcal{ELHI}$ to guarded do not increase worst-case complexities, neither in data nor in bounded-arity combined complexity. Note however, that except for RPQ answering under linear rules (inspired by the algorithm for DL-Lite$_\mathcal{R}$), the techniques used to obtain complexity upper bounds in \citep{DBLP:journals/jair/BienvenuOS15} are quite different from ours, as they rely on a (non-deterministic) query rewriting algorithm that exploits the specificities of DLs.

\medskip

 \begin{table}[t]
   \caption{Data and arity-bounded combined complexities of (C)-RPQ-answering under lightweight DL ontologies}
 \scalebox{0.9}{
 \begin{tabular}{|c|c|c|c|c|}
  \hline
  &\multicolumn{2}{c|}{\textbf{RPQ answering}}&\multicolumn{2}{c|}{\textbf{CRPQ answering}}\\
  \hline
   \textbf{Fragment} &  \textbf{Data} &  \textbf{Combined} (b) &  \textbf{Data} &  \textbf{Combined} (b)\\
  \hline
  \textbf{DL-Lite$_\mathcal{R}$} &  NL-c &  PTime-c & NL-c & PSpace-c  \\
  \hline
  \rowcolor{lightgray} \textbf{Linear} &  NL-c &  PTime-c  & NL-c & PSpace-c \\
  \hline
  $\mathcal{ELH}$ & PTime-c &  PTime-c  & PTime-c & PSpace-c \\
 \hline
 $\mathcal{ELHI}$ &  PTime-c &   ExpTime-c & PTime-c &   ExpTime-c  \\
 \hline
  \rowcolor{lightgray} \textbf{ Guarded} &  PTime-c &   ExpTime-c & PTime-c &   ExpTime-c  \\
   \hline 
 \end{tabular}
 }
 \label{table-dl-complexities}
\end{table}

\smallskip
Finally, let us mention the work in \citep{DBLP:journals/jair/StefanoniMKR14}, which studies XPath queries under the DL  $\mathcal{ELRO^+}$, underpinning OWL 2 EL. This DL is incomparable with guarded existential rules, due to the presence of axioms of the form $r_1(x, x_1), \ldots, r_k(x_{k-1}, y) \rightarrow r(x, y)$ . Moreover, the considered XPath queries generalize (C)RPQs by allowing to express some constraints on the nodes traversed along a path. An interesting future work would be to check whether the atom types we use in our algorithm could be extended to process those XPath queries.

\smallskip
Concerning existential rules, let us first point out similarities between some of our technical tools and those of previous work on CQ answering in the linear and guarded fragments. 
Regarding linear rules, our valid proof schemes (see Definition \ref{def-validity}) are not far, in essence, from the proof generators
from Gottlob et al. \citep{ijcai-15-gmp} used in a combined approach for CQ entailment under linear rules; however, the presence of transitions and their interactions with linear rules make the validity
check much more intricate (see the proof of Proposition~\ref{validity-exp}). 
Also, the reduction employed in Section~\ref{sec-crpq-guarded-upper} to lift complexity results from linear to guarded rules is very similar to that introduced in \citep{DBLP:journals/ai/GottlobMP23}. However, as the technical lemmas provided there are not enough to be used as black boxes to get our results, we provided a novel presentation based on the notion of a locally complete derivation (Definition \ref {def-locally-complete}) and its properties of query entailment preservation (Proposition\ \ref{corollary-classical-alternative} and Corollary \ref{corollary-alternative-complex}).

\smallskip
We now review in more detail known results about (C)RPQ answering with existential rules. 
Beforehand, it is worth recalling that known decidability results on CQ answering with specific existential rule classes mostly rely on properties of the \emph{chase} or of query \emph{rewriting}: either the universal model of the KB computed by the chase is finite or ``well-shaped'', which allows one to evaluate the CQ against it, or the CQ can be rewritten using the rules into a finite query, which is then evaluated on the instance. 
When moving to CRPQs, it turns out that the landscape becomes quite different depending on the underlying technique.  
 Indeed, recent work has shown that CRPQ answering is decidable for all classes of rules that guarantee the existence of a universal KB model of finite clique-width (a notion that generalizes treewidth).  As pointed out in \citep{DBLP:conf/kr/Ostropolski-Nalewaja24}, this follows from a generic result from \citep{DBLP:conf/icdt/0001LOR23}. This result applies to major classes of rules with chase-based decidable CQ answering, like those defined by acyclicity notions ensuring finite chase \citep{DBLP:journals/jair/GrauHKKMMW13} or the guarded family, which includes generalizations of the guarded class considered in this paper  \citep{DBLP:conf/kr/ThomazoBMR12}. 
This contrasts with first results on rule classes with rewriting-based decidable CQ answering, which are also provided in \citep{DBLP:conf/kr/Ostropolski-Nalewaja24}.
 Indeed, RPQ answering is shown to be undecidable for first-order rewritable rule classes, also known as \emph{fus} \citep{DBLP:journals/ai/BagetLMS11}: such classes  guarantee that any CQ can be rewritten as a (finite) union of CQs. This negative result already holds for \emph{one-way} RPQs. On the positive side, RPQ answering is shown to be decidable
 {\footnote{For the proof to hold, RPQs must be of the form $\lang(x_1,x_2)$ with $x_1$ and $x_2$ distinct variables.}}
  for an important fus concrete subclass, namely sticky rules (introduced in \citep{DBLP:journals/pvldb/CaliGP10}).  However, this decidability result does not yield any upper bound on the complexity of RPQ answering under sticky rules, and it is left as an open question whether (one-way) CRPQ answering is decidable  under sticky rules.

\section{Conclusion}
\label{sec-conclusion}
In this paper, we provide the first complexity results for (C)RPQ answering under existential rules. {These results concern two prominent classes of existential rules, namely linear and guarded rules, and distinguish between data complexity and combined complexity, considering both bounded and unbounded predicate arity. All of our complexity results are tight, thereby yielding a complete picture of the worst-case complexity of (C)RPQ answering under linear and guarded rules. Interestingly, when comparing with existing results for description logics, we observe that moving from DL-Lite$_\mathcal{R}$ to linear rules and from $\mathcal{ELHI}$ to guarded rules does not lead to any increase in data complexity, nor in bounded-arity combined complexity. Moreover, for linear rules, the data complexity for (C)RPQs is the same as for plain graph databases (namely,  NL-complete). }

While we now have clear picture of the complexity for guarded rules,  
the decidability and complexity landscapes for (C)RPQ answering for other classes of existential rules remain largely unexplored. 
As detailed in the related work section (Section \ref{sec-related-work}), decidability has only recently been established for some major classes of rules with chase-based decidable CQ answering, {including the class of frontier-guarded rules which generalizes the guarded rules considered in the present paper}, 
but this result does not come with complexity bounds, nor a reasonably implementable algorithmic scheme. 
For classes with rewriting-based decidable CQ answering, there is a general undecidability result, and little is known about the decidability of concrete classes.  

{At present, our complexity results are purely of a theoretical nature. However, we have reason to believe that the algorithm we presented for RPQ answering under linear rules can lead to a practically efficient implementation. Indeed, the construction of the \texttt{Loop} table is data-independent, and empirical studies \citep{DBLP:journals/vldb/BonifatiMT20} have found that regular languages in real-world path queries are typically very simple (thus representable with only a handful of automata states). Moreover, to reduce computation at query time, one could perform an offline preprocessing of the data by adding all entailed binary atoms (in the spirit of combined approaches to OMQA \citep{DBLP:conf/ijcai/KontchakovLTWZ11}, which rely upon data enrichment to speed up query answering). This pre-computation of entailed facts may also enable us to devise RPQ answering procedures via rewriting to path queries over plain graph databases, thereby enabling the use of modern graph database systems (as has been explored in \citep{DBLP:conf/esws/LohnertAOO25} for description logics). 
By contrast, the algorithms we devised for CRPQ answering under linear and guarded rules are not readily implementable, so new insights, possibly coupled with restrictions on query structure, will be needed to address such queries in practice. }

\begin{acks}
This work was partially supported by the French ANR projects PAGODA (ANR-12-JS02-0007), CQFD (ANR-18-CE23-0003), and EXPAND (ANR-25-CE23-1215). We thank the reviewers for their comments that helped to enhance the quality of the paper.
\end{acks}

\bibliographystyle{ACM-Reference-Format}
\bibliography{bib}

\nosectionappendix
\begin{toappendix}
\vspace{1cm}
{Received 23 July 2025; Accepted 17 June 2026.}
\end{toappendix}

\end{document}